\documentclass[letterpaper]{article} 
\usepackage{aaai2026}  
\usepackage{tocloft}

\usepackage{times}  
\usepackage{helvet}  
\usepackage{courier}  
\usepackage[hyphens]{url}  
\usepackage{graphicx} 
\usepackage{natbib}  
\usepackage{caption} 
\usepackage{xspace}
\usepackage{hyperref}

\usepackage{algorithm}

\usepackage{algpseudocode}

\usepackage{booktabs}

\usepackage{multirow}
\usepackage{rotating}
\usepackage{hhline}
\usepackage{amsmath}
\usepackage{amssymb}
\usepackage{amsthm}
\usepackage{mathtools}
\usepackage{array, arydshln}
\usepackage[table]{xcolor}

\newtheorem{theorem}{Theorem}
\newtheorem{definition}[theorem]{Definition}

\newtheorem{remark}{Remark}

\usepackage{newfloat}
\usepackage{listings}
\usepackage{enumitem}

\DeclareCaptionStyle{ruled}{labelfont=normalfont,labelsep=colon,strut=off} 
\floatstyle{ruled}
\newfloat{listing}{tb}{lst}{}
\floatname{listing}{Listing}
\title{Tab-PET: Graph-Based Positional Encodings for Tabular Transformers}
\author {
    Yunze Leng\textsuperscript{\rm 1},
    Rohan Ghosh\textsuperscript{\rm 1},
    Mehul Motani\textsuperscript{\rm 1, \rm 2}
}
\affiliations {
    \textsuperscript{\rm 1}Department of Electrical and Computer Engineering, College of Design and Engineering, National University of Singapore\\
    \textsuperscript{\rm 2}N.1 Institute for Health, Institute for Digital Medicine (WisDM), Institute of Data Science, National University of Singapore\\
    yleng@u.nus.edu, rghosh92@gmail.com, motani@nus.edu.sg
}

\usepackage{bibentry}

\begin{document}

\maketitle

\begin{abstract}

Supervised learning with tabular data presents unique challenges, including low data sizes, the absence of structural cues, and heterogeneous features spanning both categorical and continuous domains. Unlike vision and language tasks, where models can exploit inductive biases in the data, tabular data lacks inherent positional structure, hindering the effectiveness of self-attention mechanisms. While recent transformer-based models like TabTransformer, SAINT, and FT-Transformer (which we refer to as 3T) have shown promise on tabular data, they typically operate without leveraging structural cues such as positional encodings (PEs), as no prior structural information is usually available. In this work, we find both theoretically and empirically that structural cues, specifically PEs can be a useful tool to improve generalization performance for tabular transformers. We find that PEs impart the ability to reduce the effective rank (a form of intrinsic dimensionality) of the features, effectively simplifying the task by reducing the dimensionality of the problem, yielding improved generalization. To that end, we propose Tab-PET (PEs for Tabular Transformers), a graph-based framework for estimating and inculcating PEs into embeddings. Inspired by approaches that derive PEs from graph topology, we explore two paradigms for graph estimation: association-based and causality-based. We empirically demonstrate that graph-derived PEs significantly improve performance across 50 classification and regression datasets for 3T. Notably, association-based graphs consistently yield more stable and pronounced gains compared to causality-driven ones. Our work highlights an unexpected role of PEs in tabular transformers, revealing how they can be harnessed to improve generalization.
\end{abstract}

\begin{links}
    \link{Code}{https://github.com/kentridgeai/Tab-PET}
    \link{Extended version}{https://arxiv.org/abs/2511.13338}
\end{links}

\begin{figure*}[t]
    \centering
    \includegraphics[width=1\linewidth]{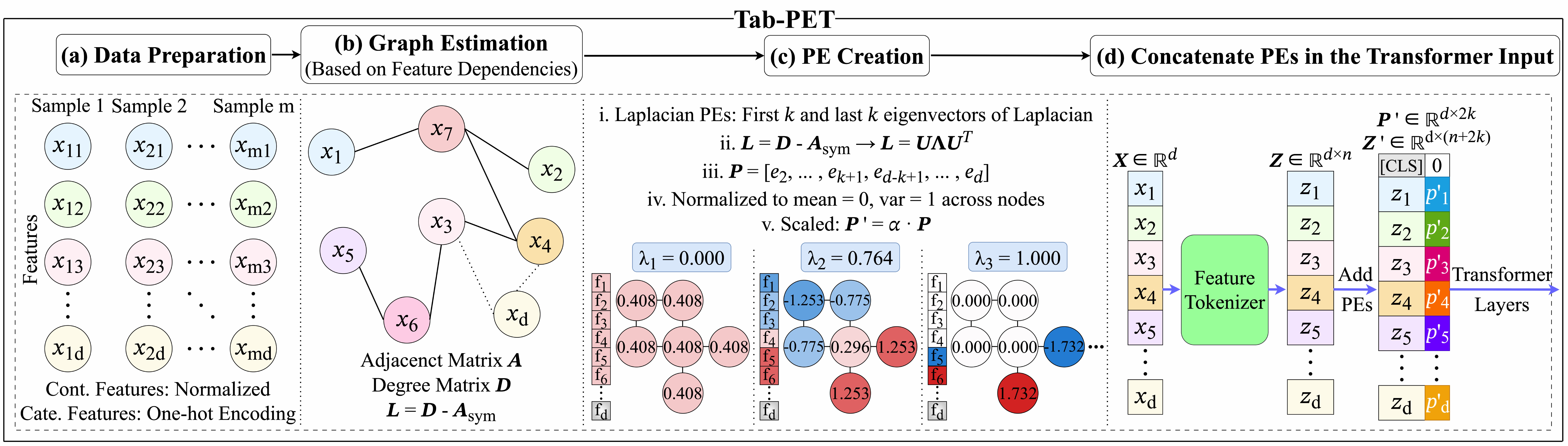}
    \caption{Tab-PET framework for integrating PEs in tabular transformers. 
(a) Categorical features are one-hot encoded and continuous features are normalized. 
(b) A feature-wise graph is estimated based on intra-feature dependencies, capturing relational structure among dimensions. 
(c) Graph Laplacian eigenvectors (examples shown) are extracted to form fixed PEs and scaled using the hyperparameter $\alpha$ to emphasize the degree of importance. 
(d) These encodings are concatenated with standard embeddings and fed into transformer layers.}
    \label{fig:flowchart}
\end{figure*}

\section{Introduction}

Tabular data remains one of the most prevalent formats in applied machine learning, spanning domains from healthcare to finance and recommender systems. Yet, learning from tabular data presents unique challenges that distinguish it from vision, language, and audio modalities. First, tabular datasets often suffer from scarcity of data: many real-world problems provide only hundreds to thousands of training examples, limiting model capacity and generalization \cite{shavitt2018regularization}. Second, high dimensionality exacerbates the problem: each sample may contain dozens or hundreds of heterogeneous features (especially after one-hot encoding the categorical variables), making interactions sparse and difficult to model. Compounding this is the heterogeneous nature of the features: categorical and continuous variables coexist, yet require distinct treatment in both preprocessing and architectural design \cite{huang2020tabtransformer}. Crucially, tabular data lacks the inductive biases that underpin recent deep learning breakthroughs: unlike images it has no spatial locality, unlike language it has no sequential ordering, and unlike audio it has no spectral coherence. This absence of structure directly impacts sample efficiency, as models must learn dependency structure from scratch, making the problems of data scarcity and high dimensionality even more pronounced. As a result, tabular learning remains an active area of research \cite{gorishniy2021revisiting, kadra2021well}.

Modalities like vision and audition benefit from inherent structural priors (spatial locality/temporal continuity). In vision, CNNs exploit translational symmetry by applying shared filters across spatial patches, embedding a strong inductive bias toward local structure. Even transformer-based models in vision and audition, such as ViT and AST, preserve this bias by tokenizing inputs into fixed-size patches or frames, thereby retaining partial spatial or temporal invariance \cite{dosovitskiy2020image, gong2021ast}. In contrast, language models introduce structure into the self-attention mechanism via positional encoding (PE), which enables the model to distinguish token order and learn position-sensitive dependencies \cite{vaswani2017attention}. This technique has been adapted to ViTs, where positional embeddings help recover spatial relationships lost during patch flattening \cite{chu2021we, xu2021vitae}. These architectural strategies demonstrate that injecting structure, whether through convolution, patching, or PEs, can be integral for sample-efficient learning in domains where raw data lacks explicit structural organization.

A foundational limitation of self-attention in transformers is that it operates over unordered inputs. Without extra information, the mechanism treats all tokens or features as equally exchangeable, encoding no preference for proximity or sequence. This lack of structural bias does not align with natural language, which is inherently organized: grammatical relations and syntactic dependencies rely on the relative positions of words and phrases \cite{vaswani2017attention}.

To introduce this bias directly into the model's computation, PEs modify each token by appending a position-specific vector that conveys its location in the input. These vectors, whether fixed or learned, alter the dot product between queries and keys, reshaping the attention pattern. Tokens close in sequence typically carry similar position vectors, which causes the self-attention mechanism to favor interactions among them. This preference allows the model to construct higher-order representations while discouraging irrelevant pairings unless they are statistically warranted. The outcome is a soft inductive bias that recovers local structure otherwise missing from the architecture.

Despite growing interest in PEs across sequence and graph-based domains, their application to tabular data remains largely unexplored. This stems from the lack of inherent structural priors in tabular inputs, i.e., feature order is arbitrary and structural relationships among features are rarely specified upfront. Yet, given the challenges facing tabular classification and regression tasks---limited data, high dimensionality, and feature heterogeneity---the question of whether we can impose meaningful inductive bias via PEs arises. The current consensus in literature is that PEs cannot benefit tabular transformers because tabular data is structure-less and has inputs of different types, in contrast to vision and language \cite{somepalli2021saint}. In this work, however, we show that PEs can serve to control learning complexity, reducing the dimensionality of features extracted by self-attention layers. We estimate graphical structures that capture inter-feature associations from the data, and then derive PEs from the eigenvectors of the graph Laplacians. These graph-derived encodings are used to augment the original embeddings, enabling a structured inductive bias where none existed natively. We refer to this as Positional Encodings for Tabular Transformers (Tab-PET).

\section{Contributions}

Our work has the following contributions:

\begin{enumerate}[nosep,leftmargin=*,wide = 0pt]

    \item \textbf{Tab-PET:} We propose a principled method for constructing PEs in tabular domains. The first step involves learning a feature graph that captures inter-feature associations. We explore two approaches: causality-based and association-based. The second step involves generating feature-wise encodings using the Laplacian eigenvectors of the learned feature graph. These embeddings are then used to augment input embeddings in the transformer.

    \item \textbf{Theoretical Motivation via Rank Analysis:} We find that the use of PEs imparts the ability to reduce the effective rank of the embeddings within transformer architectures, and more so when they are aligned with the structure of the data. Empirical tests confirm these findings. 

    \item \textbf{Empirical Evaluation Across Benchmarks:} We apply our proposed approach to leading tabular transformer models, including TabTransformer, SAINT, and FT-Transformer. Tab-PET demonstrates consistent improvements across 50 classification and regression datasets.

    \item \textbf{Ablation Studies and Performance Analysis:} We conduct ablations, statistical significance testing and comparisons with learnable PEs to demonstrate the effectiveness of our approach.

\end{enumerate}

\section{Background}

\subsection{Tabular Transformers}
The application of neural networks to tabular data has made strides in recent years. Architectures such as ResNet-like \cite{he2016identity}, NODE \cite{popov2019neural}, and SNN \cite{klambauer2017self} have demonstrated strong performance despite a fundamental challenge: tabular datasets tend to be small in size and lack the rich structural priors present in language or vision domains. 

Given the widespread success of transformer architectures in language and vision domains \cite{vaswani2017attention,dosovitskiy2020image}, there has been growing interest in adapting self-attention mechanisms to tabular data. Models such as TabTransformer \cite{huang2020tabtransformer} and FT-Transformer \cite{gorishniy2021revisiting} attempt this adaptation by embedding features into a higher-dimensional space via projection layers. FT-Transformer, for instance, utilizes a feature-tokenizer to create learnable embeddings for both categorical and continuous features. TabTransformer applies tokenization only to categorical features, treating continuous values collectively through concatenated representations. These embedding transformations enable application of multi-head self-attention followed by classification heads.

Architectures like SAINT \cite{somepalli2021saint} extend attention across intra-sample and inter-sample domains, yielding strong performance. Meanwhile, language-guided models designed for spreadsheet-style inputs (e.g., TAPAS \cite{herzig2020tapas}) incorporate semantic cues from column headers and contextual metadata. In our work, we focus exclusively on scenarios where only feature names and raw values are provided, discarding external linguistic context.

\subsection{Graph Estimation Approaches}

A variety of methods have been proposed to estimate graphical structures over tabular data, where each feature dimension is treated as a node in a graph. Broadly, these approaches fall into two categories: \textit{causality-based} and \textit{association-based}.

\noindent\textbf{Causality-Based Graphs:} Causality-based methods aim to infer directed edges between features that reflect underlying generative mechanisms: i.e., an edge from feature $i$ to feature $j$ implies that $i$ is a cause of $j$. These models often assume a linear structural equation model (SEM) of the form:
$\mathbf{x} = \mathbf{W} \mathbf{x} + \boldsymbol{\epsilon},$
where $\mathbf{W}$ is a weighted adjacency matrix encoding causal relationships, and $\boldsymbol{\epsilon}$ is a noise vector. Classical approaches such as LiNGAM \cite{shimizu2006linear} rely on non-Gaussianity assumptions to identify causal directions. More recent methods like NOTEARS and its variants \cite{zheng2018dags,lee2019scaling} reformulate structure learning as a continuous optimization problem, minimizing reconstruction error while enforcing acyclicity constraints via smooth penalties. These approaches have enabled scalable and differentiable learning of directed acyclic graphs (DAGs) from observational data.

\noindent\textbf{Association-Based Graphs:} Association-based methods construct graphs by quantifying statistical dependence between feature pairs. A common formulation sets the edge weight $w_{ij}$ between features $x_i$ and $x_j$ as:
\begin{equation} \label{eq:assoc_weights}
    w_{ij} = \rho(x_i, x_j),
\end{equation}
where $\rho$ is a measure of association, such as Pearson correlation, Spearman correlation, mutual information (MI), or distance. The Chow-Liu algorithm \cite{chow1968approximating} is a canonical example, which uses pairwise MI to construct a maximum-weight spanning tree that approximates the joint distribution. Notably, while Chow-Liu ensures the resulting graph is a tree-structured DAG, it does not model generative mechanisms explicitly. In general, association-based graphs prioritize statistical dependence over causal interpretability.

\subsection{Positional Encoding Integration with Graphs}

Our objective in this work is to infer PEs for each feature based on the underlying structure of the data, which we estimate via graph estimation (see part (b) of Figure~\ref{fig:flowchart}). Graph-based PEs have been widely studied in the literature, particularly in the context of graph neural networks and graph transformers. These methods can be broadly categorized into two families: \textit{fixed} and \textit{learnable} PEs.

\noindent\textbf{Fixed Positional Encodings:} Fixed PEs typically leverage the spectral properties of the graph Laplacian. The most common strategy involves computing its eigenvectors and using them to encode node positions:
(1) \textit{First $k$ eigenvectors}: Dwivedi and Bresson \cite{dwivedi2020generalization} propose using the lowest-frequency components of the Laplacian spectrum, which capture global graph structure; (2) \textit{All eigenvectors}: Ito et al. \cite{ito2025learning} argue that using the full spectrum preserves all structural information, avoiding frequency truncation; (3) \textit{First and last $k$ eigenvectors}: The same work \cite{ito2025learning} shows that only combining low- and high-frequency components yields robust encodings for both homophilous and heterophilous graphs.

\noindent\textbf{Learnable Positional Encodings:} Learnable PEs aim to optimize the encoding process by adapting to task-specific signals. Two notable approaches include: (1) \textit{Elastic PEs}: Liu et al. \cite{canturk2023graph} propose learning linear projections of Laplacian eigenvectors, enabling flexible adaptation to graph structure; (2) \textit{Eigenvector weighting}: Ito et al. \cite{ito2025learning} introduce a method that learns the importance of each eigenvector via backpropagation, allowing the model to emphasize relevant spectral components.



\noindent\textbf{Our Approach:} In this work, we adopt the fixed PE strategy using the first and last $k$ eigenvectors of the graph Laplacian. This choice avoids increasing parametric complexity and also ensures fair comparisons across architectures. Furthermore, empirically, we find that fixed PEs outperform learnable PEs (from scratch) on tabular datasets. 
Moreover, the use of both low- and high-frequency components has been shown to improve expressiveness across graph types, particularly in heterophilous settings \cite{ito2025learning}. 

\section{Motivation}

\subsection{Theoretical Results: PEs and Effective Rank}

In this section, we theoretically study the ability of PEs to control the intrinsic dimensionality of the learning problem. There have been many significant works that find that lower rank (i.e. low intrinsic dimensionality) of features often leads to better generalization performance \cite{ghosh2023local,huh2021low,arora2019implicit}.   A significant study that explores the link between intrinsic dimension and generalization finds that the intrinsic dimensionality of the data controls both the approximation error (training fit) and the generalization error \cite{nakada_ID_generalization}. Thus, the ability of an architecture to reduce the dimensionality of the learning task via low rank features is a positive sign from a generalization perspective. 

In what follows, we provide results that show the ability PEs to directly reduce the effective rank \cite{roy2007effective} of the CLS output embeddings within FT-Transformers. Note that the CLS output eventually is used for the final prediction via fully connected layers. Proofs are provided in technical appendix \ref{app:9} \cite{LengTabPET}. We reiterate the definition of effective rank below. 

\begin{definition}[Effective Rank]
For a matrix $\mathbf{X} \in \mathbb{R}^{n \times d}$ representing CLS embeddings from $n$ samples, the effective rank is defined as:
\begin{equation}
r_{\text{eff}}(\mathbf{X}) = \exp\left(-\sum_{i=1}^{r} \tilde{\sigma}_i \log \tilde{\sigma}_i\right),
\end{equation}
where $\tilde{\sigma}_i = \sigma_i / \sum_{j=1}^{r} \sigma_j$ are the normalized singular values obtained from SVD decomposition $\mathbf{X} = \mathbf{U}\boldsymbol{\Sigma}\mathbf{V}^T$, and $r$ is the rank of $\mathbf{X}$. This formulation captures the intrinsic dimensionality of the learned representations through the Shannon entropy of the singular value distribution.
\end{definition}
\setcounter{theorem}{0}

First, we provide a result in the general case where the input dimensions are i.i.d. 
\begin{theorem}\label{thm:1}[Effective Rank under Random Inputs]
    Let \( x \in \mathbb{R}^d \) be an input vector to a single-layer, single-head FT-Transformer with components \( x_i \sim \text{i.i.d.} \) and \( x_i \in (0,1) \). Let \( d_T \) denote the token dimension (inclusive of concatenated position encodings). Let \( q\in \mathbb{R}^{d_T} \) denote the learnable CLS token embedding, and \( p_i \in \mathbb{R}^{d_p} \) be the positional encodings for each input dimension. Assume the scaled positional encodings \( p'_i = \alpha p_i \) are used, where \( \alpha > 0 \).  Given the query, key and value matrices $Q,K,V$, where $K$ can be decomposed as $[K_x;K_p]$, where $K_x \in \mathbb{R}^{d \times d_T}, K_p \in \mathbb{R}^{d_p \times d_T}$, and similarly for $V$. Suppose the following conditions hold: $\max_i \langle Q^Tq,K_p^Tp_i \rangle - \max_{j \neq i} \langle  Q^Tq,K_p^Tp_j \rangle = \tau,$ and the norm of the query-key matrices $Q,K$ and CLS embedding $q$ are all bounded by $c_Q,c_K$ and $c_q$ respectively. Lastly, assume that the tokenizer weights $w_i$ have the same norm and the value matrix $V$ is norm preserving and satisfies $V_p = 0$.
Define
\begin{equation}
C_\alpha = \exp\left( \frac{\alpha \tau - 2 c_Kc_Q c_q }{\sqrt{d_T}} \right).
\end{equation}
Then the effective rank \( r_{\mathrm{eff}} \) of the CLS token output after self-attention satisfies
\begin{equation}
r_{\mathrm{eff}} \leq \left( C_\alpha + d \right) \cdot \exp\left( -\frac{C_\alpha}{C_\alpha + d} \cdot \log C_\alpha \right).
\end{equation}
In the regime where \( C_\alpha \gg d \), this simplifies to $r_{\mathrm{eff}} \approx 1 + \frac{d}{C_\alpha}.$
\end{theorem}
\begin{remark}
    Note that when $C_\alpha \gg d$, $r_{\mathrm{eff}} \approx 1 + C e^{-\alpha \tau/\sqrt{d_T}}$. Thus, the effective rank can be significantly reduced when the PEs are weighted higher using larger $\alpha$, but only when $\tau > 0$. When not using any PEs however, one obtains $\tau=0$ as $p_i=[0,0,..0]$. Thus without PEs, effective rank can be significantly larger.  
\end{remark}

We next discuss the case where the input data is not i.i.d dimension-wise but has some underlying structure. 

\begin{theorem}\label{thm:2}[Effective Rank under Structured Inputs]
Consider the same setting as in Theorem \ref{thm:1}, except that the input vector \( x \in \mathbb{R}^d \) is structured as follows: \( d \) is even, and
\begin{equation}
x_i = 
\begin{cases}
\theta & \text{for } i \leq d/2, \\
\theta'  & \text{for } i > d/2,
\end{cases}
\end{equation}
with shared latent variables \( \theta, \theta' \in (0,1) \), and coefficients \( \beta_i, \gamma_i \in \mathbb{R} \). Then the effective rank \( r_{\mathrm{eff}} \) of the CLS token output after self-attention satisfies:

\begin{itemize}
    \item \textbf{(a) Random positional encodings:}
\begin{equation}
    r_{\mathrm{eff}} \leq \left( 2C_\alpha + d \right) \cdot \exp\left( -\frac{2C_\alpha \log(2C_\alpha) + d \log d}{2C_\alpha + d} \right),
\end{equation}
    $\text{which simplifies to } r_{\mathrm{eff}} \approx 1 + \frac{d}{2C_\alpha} \text{ when } C_\alpha \gg d.$

    \item \textbf{(b) Shared positional encodings within groups:}
    If \( p_i \) is fixed for all \( i \leq d/2 \), and is a different fixed vector for \( i > d/2 \), then
\begin{equation}
    r_{\mathrm{eff}} \leq (C_\alpha + 1) \cdot \exp\left( -\frac{C_\alpha}{C_\alpha + 1} \cdot \log C_\alpha \right),
\end{equation}
    
 $\text{which simplifies to } r_{\mathrm{eff}} \approx 1 + \frac{1}{C_\alpha} \text{ for large } C_\alpha.$

\end{itemize}
\end{theorem}

\begin{remark}
The above result clearly indicates that the effective rank of the CLS token output of the FT-Transformer depends on whether the PEs have adapted to the structure of the underlying data. When some input dimensions are similar to each other (i.e. Theorem \ref{thm:2}), assigning the same PE to the similar dimensions can significantly reduce the effective rank of the CLS output. Thus, choosing appropriate PEs that follow data structure can significantly reduce the dimensionality of the learning problem, enhancing generalization performance. But this carries some limitations as well. Tasks which intrinsically require larger effective rank to appropriately address, may not benefit from the inclusion of PEs.  
\end{remark}





\section{Methodology}

In this section, we outline the four main steps involved in estimating and integrating PEs in transformer-based architectures for tabular data. We summarize the following steps in Figure \ref{fig:flowchart} as well. 



\subsection{Data Preparation}
We begin by applying one-hot encoding to all categorical variables. This helps eliminate any implicit structural bias induced by their native ordering based representation. A side effect of this transformation is that the feature dimensionality increases, which impacts the size of the graph estimated in subsequent sections. For continuous variables, we normalize each to have zero mean and unit variance. 

\subsection{Graph Estimation}

After preprocessing, we denote each input sample by the $d$-dimensional feature vector:
\begin{equation}
    \mathbf{x}^{(j)} = [x^{(j)}_1, x^{(j)}_2, \dots, x^{(j)}_d]^\top, \quad j = 1,\dots,m
\end{equation}
where $m$ is the number of samples, and $x^{(j)}_i$ denotes the individual feature dimensions of the processed input. Each feature $x_i$ corresponds to a node in graph, and edges represent statistical or causal dependencies between features.

We explore two primary paradigms for graph learning: \textit{causality-based methods} and \textit{association-based methods}. In the former, we assume a linear structural causal model given by $\mathbf{x} = \mathbf{A} \mathbf{x} + \boldsymbol{\epsilon}$,
where $\mathbf{A}$ is a weighted adjacency matrix representing causal relationships, and $\boldsymbol{\epsilon}$ is an independent noise vector. We apply algorithms such as LiNGAM and NOTEARS to learn this causal graph, resulting in DAGs.

As mentioned earlier, in association-based approaches, we can define the edge weight $w_{ij}$ between nodes $x_i$ and $x_j$ as a function of their statistical dependency: $w_{ij} = \rho(x_i, x_j)$.
Here, we choose $\rho$ from Pearson correlation, Spearman rank correlation, or MI. For MI-based graph estimation, we employ the Chow-Liu algorithm to ensure the resulting structure remains a DAG.

\subsection{Positional Encoding Creation}

Given the estimated graph, we first symmetrize the adjacency matrix to produce an undirected version: $\mathbf{A}_{\text{sym}} = \frac{1}{2} (\mathbf{A} + \mathbf{A}^\top)$.
We then compute the graph Laplacian:
$\mathbf{L} = \mathbf{D} - \mathbf{A}_{\text{sym}}$,
where $\mathbf{D}$ is the degree matrix. The top-$k$ and bottom-$k$ eigenvectors of $\mathbf{L}$ are selected (excluding the first eigenvector, which is constant valued), normalized to have zero mean and unit variance across nodes, and concatenated to form the PE matrix: $\mathbf{P} = [\mathbf{e}_2, \dots, \mathbf{e}_{k+1}, \mathbf{e}_{d-k+1}, \dots, \mathbf{e}_d].$
To modulate the influence of these encodings, we scale them using a hyperparameter $\alpha$:
\begin{equation} \label{eq:alpha_inclusion}
    \mathbf{P}' = \alpha \cdot \mathbf{P}.
\end{equation}
For categorical features with multiple one-hot encoded nodes, we average the individual encodings to yield a consolidated PE vector for the feature.

\subsection{Positional Encoding Integration}

In transformer-based architectures for tabular data, each feature undergoes tokenization to obtain an $n$-dimensional embedding vector. We integrate our estimated PEs $\mathbf{P}'$ by concatenating them with the corresponding feature embedding:
\begin{equation}
    \mathbf{z}_i' = [\mathbf{z}_i ; \mathbf{p}_i'] \in \mathbb{R}^{n + 2k},
\end{equation}
where $\mathbf{z}_i$ is the original embedding for $x_i$ and $\mathbf{p}_i'$ is the scaled PE for $x_i$.  Subsequently, these modified embeddings serve as inputs to the self-attention layers during training.

\begin{algorithm}[t]
\caption{Structure-Controlled Tabular Data Generation} \label{alg:synthetic_generator}
\begin{algorithmic}[1]
\State \textbf{Input:} Feature dimension $d$, number of partitions $k$, number of samples $n$
\State \textbf{Output:} Synthetic dataset $X \in \mathbb{R}^{n \times d}$, output variable $y \in \mathbb{R}^n$
\State \textbf{Group Assignment:} Partition features into $k$ disjoint groups $\{G_1, G_2, \dots, G_k\}$
\State Sample fixed weights $w_f \sim \mathcal{U}(-1, 1)$ for all features $f \in \{1, \dots, d\}$
\State \textbf{Generative Process:}
\For{each sample $t = 1$ to $n$}
    \For{each group $g = 1$ to $k$}
        \State Sample group latent variable $\theta_g^{(t)} \sim \mathcal{U}(-2, 2)$
        \For{each feature $f \in G_g$}
            \State $x_f^{(t)} = \theta_g^{(t)} \cdot w_f + \varepsilon_f^{(t)}, \quad \varepsilon_f^{(t)} \sim \mathcal{N}(0, 0.01)$
        \EndFor
    \EndFor
    \State \textbf{Target Outputs Generation:}
    \State Select fixed group index $g^* \in \{1, \dots, k\}$ (shared across all samples)
    \State Sample $w_t, b \sim \mathcal{U}(-1, 1)$
    \State Compute $y^{(t)} = w_t \cdot \theta_{g^*}^{(t)} + b$
\EndFor
\end{algorithmic}
\end{algorithm}

\section{Synthetic Experiments: Evaluating Structure-Sensitive Positional Encodings}

In this section, we design synthetic experiments to evaluate whether the benefits of 
PEs are inherently tied to the presence of structural relationships within tabular data. Transformers rely on PEs to guide self-attention, yet in tabular domains, structural cues are typically absent. We hypothesize that when feature correlations exist, PEs derived from such associations can improve learning. Conversely, when features are independent, PEs may have limited impact. To explore this, we introduce a synthetic framework where the structure of the dataset can be controlled parametrically.

\subsection{Definition of Structure}
We define structure as the degree of association between features. If all features are independent, no meaningful pairwise relationships exist, and all features are equally unrelated. To simulate a controllable structure, we partition the feature space into $k$ groups where intra-group features share latent correlations while inter-group features remain independent.

\subsection{Data Generation Algorithm}
Given input dimensionality $d$ and partition count $k$, the synthetic data generation is detailed in Algorithm \ref{alg:synthetic_generator}. We note that as $k$ increases, most features end up in their own group, leading to independence and minimal structure. When $k$ is lower, many features share common generative variables, increasing structure. We consider the scenario where the generated dataset embodies a regression problem, where the underlying function is a linear function of one of the partitions. 


\subsection{Experimental Setup and Results}

We evaluate FT-Transformer on synthetic datasets with input dimensionality fixed at $d = 30$, using Spearman-based graph-derived PEs. Each feature embedding is concatenated with PE and scaled by a hyperparameter $\alpha$ to modulate its influence. To analyze the impact of structural variation, we partition features into $k$ groups and categorize results into three regimes: (a) high structure ($k \leq 8$), (b) moderate structure ($10 \leq k \leq 22$), and (c) low structure ($k > 22$). Accuracy is assessed across varying $\alpha$ values and partition counts.

\begin{figure}[t]
    \centering
    \includegraphics[width=\linewidth]{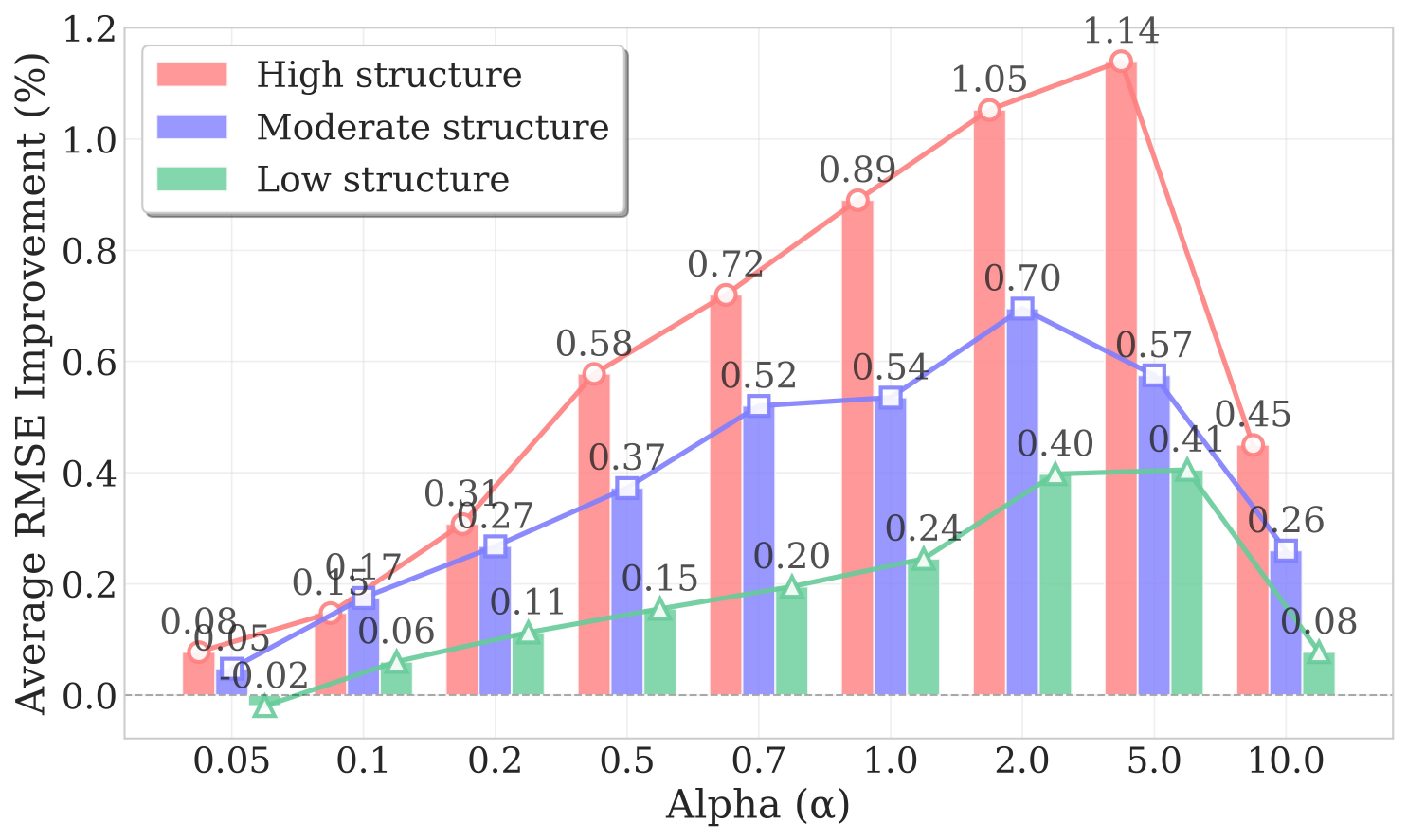}
    \caption{RMSE performance comparison with low, moderate, and high structure synthetic datasets across varying $\alpha$.}
    \label{fig:synthetic}
\end{figure}

\noindent\textbf{Results}
Figure~\ref{fig:synthetic} shows performance across structural regimes. As anticipated, datasets with stronger internal associations benefit more from graph-derived PEs, yielding larger improvements as the positional signal is amplified via higher $\alpha$ values. This empirically confirms that PEs are most useful when the data exhibits meaningful structure.

Interestingly, even in highly unstructured settings, we observe minor but consistent gains from PEs. This is because PEs can reduce the effective rank of the learning problem, even when the inputs are unstructured (Theorem \ref{thm:1}), and the generative structure is a simple linear function that doesn't require high-rank features.

Lastly, Figure \ref{fig:synthetic} shows that excessively amplifying the contribution of PEs, by increasing $\alpha$ to 10, can degrade model performance. This is intuitive, as a large $\alpha$ disproportionately weights the positional signal, potentially overshadowing the original input content encoded within the value vectors of the query-key-value decomposition.

\section{Experiment on Real Datasets}
In this section, we discuss the results on real datasets. We present the details regarding our experimental setup, the main results, and ablation as well as analytical studies on Tab-PET.
All experiments were conducted on 3 NVIDIA A100 SXM4 80GB GPUs and 3 NVIDIA RTX 6000 Ada Generation GPUs.

\subsection{Experimental Setup}

\textbf{Datasets:} We run experiments on a comprehensive collection of 50 tabular datasets sourced from OpenML (https://www.openml.org), comprising 25 classification and 25 regression tasks. These datasets span diverse domains with varying characteristics in terms of sample sizes, feature dimensionalities, and categorical ratios. The complete list of datasets with detailed properties is provided in technical appendix \ref{app:2.1}. To preserve the statistical properties of each dataset, we employ stratified sampling for classification tasks that maintains the exact class distribution of the original dataset. During training, we split each dataset into a 60:20:20 train-validation-test split. We outline the detailed data preprocessing steps in technical appendix \ref{app:2.2}. 

\noindent\textbf{Graph Estimation:} We found that some graph estimation approaches are computationally expensive. NO-TEARS was particularly expensive for larger datasets. Association-based graphs use weights from pairwise measures in Eq. \eqref{eq:assoc_weights}. Chow-Liu requires an additional step to ensure the graph is a DAG. The NO-TEARS and LinGAM approaches for graph estimation have tunable hyperparameters, which are outlined in technical appendix \ref{app:3}. 
 
\noindent\textbf{PE Creation:} When generating PEs, we design an automatic $k$ selection algorithm that adaptively determines the optimal number of low-frequency and high-frequency eigenvectors based on spectral gap analysis. The $k$ selection algorithm is based on thresholding the normalized eigenvalues, which are a proxy of the effective frequency of the eigenvector, from both sides. We outline this in technical appendix \ref{app:4}. The hyperparameter $\alpha$  (Fig. \ref{fig:flowchart} (c) and Eq. \eqref{eq:alpha_inclusion}), which is chosen from a set of 9 elements from $0.05$ to $10$, is optimized via the validation set using a greedy approach. 

\noindent\textbf{Training and Evaluation:}
Following numerous previous studies \cite{ye2024ptarl,gorishniy2021revisiting} that use RMSE and accuracy, we report RMSE for regression tasks and balanced accuracy for classification tasks. Balanced accuracy presents an unbiased estimate of performance that is independent of class imbalance, and following best practices, we train our models using balanced cross-entropy loss.
Each result is reported after averaging across five random seeds for all approaches tested in our work. We use early-stopping for all approaches during the training process. 
Details are provided in technical appendix \ref{app:5}.

\subsection{Baselines for Comparison}
We compare Tab-PET with two categories of baselines:
\begin{itemize}[nosep,leftmargin=*,wide = 0pt]

    \item \textbf{Tree-based:} We compare Tab-PET against two state-of-the-art (SOTA) gradient boosting methods that show superior performance on tabular data: XGBoost~\cite{chen2016xgboost} and CatBoost~\cite{prokhorenkova2018catboost}. For these methods, we use Optuna-driven hyperparameter optimization~\cite{akiba2019optuna} with 100 trials per dataset. Detailed hyperparameter ranges can be found in technical appendix \ref{app:6.1}.

    \item \textbf{Transformer-based:} We compare Tab-PET against three SOTA transformer-based methods designed for tabular data: TabTransformer~\cite{huang2020tabtransformer}, SAINT~\cite{somepalli2021saint}, and FT-Transformer~\cite{gorishniy2021revisiting}. Lastly, for fair comparison, we keep batch size, epochs, learning rate, dimensionality of the feature tokenizer output, and other such hyperparameters fixed when comparing PE and non-PE approaches for each approach. Complete hyperparameter configurations and fair comparison details are provided in technical appendix \ref{app:6.2} and \ref{app:6.3}.
\end{itemize}

\begin{table}[t]
\centering
\scriptsize
\setlength{\tabcolsep}{3pt}
\renewcommand{\arraystretch}{1.3}
\begin{tabular}{l!{\vrule}ccccc!{\vrule}cc!{\vrule}c}
\toprule
Method & CA & AS & DA & NL & TR & C Improv. $\uparrow$ & R Improv. $\uparrow$ & Time (min) \\
\midrule
NOTEARS & $\checkmark$ &  & $\checkmark$ &  &  & 1.36\%  & 3.64\%  & 76.83 \\
\arrayrulecolor{lightgray}\hline\arrayrulecolor{black}
LiNGAM & $\checkmark$ &  & $\checkmark$ &  &  & 1.41\%  & 3.97\%  & 10.96 \\
\arrayrulecolor{lightgray}\hline\arrayrulecolor{black}
Pearson &  & $\checkmark$ &  &  &  & \underline{1.61\% } & 4.16\%  & \underline{0.78} \\
\arrayrulecolor{lightgray}\hline\arrayrulecolor{black}
Spearman &  & $\checkmark$ &  & $\checkmark$ &  & \textbf{1.72\% } & \textbf{4.34\% } & 0.79 \\
\arrayrulecolor{lightgray}\hline\arrayrulecolor{black}
Chow-Liu &  & $\checkmark$ & $\checkmark$ & $\checkmark$ & $\checkmark$ & 1.17\%  & \underline{4.29\% } & \textbf{0.38} \\
\bottomrule
\end{tabular}
\caption{Average performance improvement of graph estimation approaches with Tab-PET. Results show improvement percentage over baseline. Time (min) represents the average extra computational time introduced by Tab-PET for graph estimation and PE creation. CA = Causal, AS = Association, DA = Directed Acyclic, NL = Nonlinear, TR = Tree.}
\label{tab:graph_estimation_results}
\end{table}

\begin{table*}[t]
\centering
\setlength{\tabcolsep}{0.5pt}
\resizebox{\textwidth}{!}{
\begin{tabular}{l!{\vrule}c!{\vrule}@{\hspace{0.3pt}}c@{\hspace{0.3pt}}@{\hspace{0.3pt}}c@{\hspace{0.3pt}}@{\hspace{0.3pt}}c@{\hspace{0.3pt}}@{\hspace{0.3pt}}c@{\hspace{0.3pt}}@{\hspace{0.3pt}}c@{\hspace{0.3pt}}@{\hspace{0.3pt}}c@{\hspace{0.3pt}}@{\hspace{0.3pt}}c@{\hspace{0.3pt}}@{\hspace{0.3pt}}c@{\hspace{0.3pt}}@{\hspace{0.3pt}}c@{\hspace{0.3pt}}@{\hspace{0.3pt}}c@{\hspace{0.3pt}}@{\hspace{0.3pt}}c@{\hspace{0.3pt}}@{\hspace{0.3pt}}c@{\hspace{0.3pt}}@{\hspace{0.3pt}}c@{\hspace{0.3pt}}@{\hspace{0.3pt}}c@{\hspace{0.3pt}}@{\hspace{0.3pt}}c@{\hspace{0.3pt}}@{\hspace{0.3pt}}c@{\hspace{0.3pt}}@{\hspace{0.3pt}}c@{\hspace{0.3pt}}@{\hspace{0.3pt}}c@{\hspace{0.3pt}}@{\hspace{0.3pt}}c@{\hspace{0.3pt}}@{\hspace{0.3pt}}c@{\hspace{0.3pt}}@{\hspace{0.3pt}}c@{\hspace{0.3pt}}@{\hspace{0.3pt}}c@{\hspace{0.3pt}}@{\hspace{0.3pt}}c@{\hspace{0.3pt}}@{\hspace{0.3pt}}c@{\hspace{0.3pt}}@{\hspace{0.3pt}}c@{\hspace{0.3pt}}}
\toprule
Model & \textbf{Rank} & AU$\uparrow$ & GE$\uparrow$ & SA$\uparrow$ & BL$\uparrow$ & CHU$\uparrow$ & CM$\uparrow$ & CR$\uparrow$ & DI$\uparrow$ & DN$\uparrow$ & EY$\uparrow$ & FI$\uparrow$ & HE$\uparrow$ & JA$\uparrow$ & KC$\uparrow$ & KR$\uparrow$ & MA$\uparrow$ & PH$\uparrow$ & QSA$\uparrow$ & ST$\uparrow$ & SY$\uparrow$ & TI$\uparrow$ & VE$\uparrow$ & WI$\uparrow$ & WIN$\uparrow$ & YE$\uparrow$ \\
\midrule
XGB & 3.40 & 0.837$^\ddag$ & 0.614 & 0.746 & 0.725 & 0.893 & 0.571$^\dag$ & 0.722$^\dag$ & 0.721 & 0.966$^\dag$ & 0.696 & 0.524$^\dag$ & 0.711 & 0.811$^\ddag$ & 0.704 & 0.997$^\dag$ & 0.861$^\dag$ & 0.845 & 0.858$^\dag$ & 0.806 & 0.934 & 0.978 & 0.759 & 0.914 & 0.396$^\dag$ & 0.582$^\dag$ \\
CB & 3.76 & 0.830 & 0.635 & 0.776 & 0.705 & 0.894 & 0.562 & 0.681 & 0.710 & 0.964$^\ddag$ & 0.726 & 0.502$^\ddag$ & 0.722 & 0.819$^\dag$ & 0.719$^\ddag$ & 0.992 & 0.856$^\ddag$ & 0.860 & 0.842 & 0.810$^\ddag$ & 0.953$^\dag$ & 0.981$^\dag$ & 0.747 & 0.930 & 0.273 & 0.567$^\ddag$ \\
\arrayrulecolor{lightgray}\hline\arrayrulecolor{black}
TT & 7.33 & 0.822 & — & — & — & 0.831 & 0.485 & 0.661 & — & 0.956 & 0.595 & — & — & 0.785 & — & 0.993 & — & — & — & — & — & 0.973 & — & — & — & — \\
\textbf{+PET} & \textbf{5.33} & \textbf{0.836} & — & — & — & \textbf{0.840} & \textbf{0.488} & \textbf{0.683} & — & \textbf{0.960} & \textbf{0.598} & — & — & \textbf{0.791} & — & \textbf{0.994} & — & — & — & — & — & \textbf{0.980$^\ddag$} & — & — & — & — \\
\arrayrulecolor{lightgray}\hline\arrayrulecolor{black}
ST & 4.52 & \textbf{0.826} & \textbf{0.663$^\dag$} & 0.814$^\ddag$ & 0.733 & 0.855 & 0.557 & \textbf{0.709$^\ddag$} & 0.730 & 0.961 & 0.687 & 0.471 & 0.725 & 0.797 & 0.715 & 0.992 & \textbf{0.809} & 0.852 & 0.841 & 0.798 & 0.940 & 0.969 & 0.762 & 0.977$^\ddag$ & 0.350 & 0.545 \\
\textbf{+PET} & \textbf{3.28$^\ddag$} & 0.826 & 0.660$^\ddag$ & \textbf{0.862$^\dag$} & \textbf{0.756$^\ddag$} & \textbf{0.878} & \textbf{0.563} & 0.709 & \textbf{0.740} & \textbf{0.964} & \textbf{0.730} & \textbf{0.474} & \textbf{0.728$^\ddag$} & \textbf{0.807} & \textbf{0.718} & \textbf{0.995} & 0.805 & \textbf{0.854} & \textbf{0.854$^\ddag$} & \textbf{0.805} & \textbf{0.943} & \textbf{0.973} & \textbf{0.785} & \textbf{0.982$^\dag$} & \textbf{0.351} & \textbf{0.551} \\
\arrayrulecolor{lightgray}\hline\arrayrulecolor{black}
FT & 4.44 & 0.828 & 0.641 & 0.779 & 0.748 & 0.908$^\ddag$ & 0.545 & 0.644 & 0.769$^\ddag$ & 0.953 & 0.748$^\ddag$ & 0.460 & 0.724 & 0.805 & 0.714 & \textbf{0.996$^\ddag$} & 0.718 & 0.861$^\ddag$ & 0.831 & 0.795 & 0.947 & 0.972 & 0.786$^\ddag$ & 0.957 & 0.363 & 0.545 \\
\textbf{+PET} & \textbf{2.44$^\dag$} & \textbf{0.838$^\dag$} & \textbf{0.650} & \textbf{0.792} & \textbf{0.758$^\dag$} & \textbf{0.917$^\dag$} & \textbf{0.571$^\ddag$} & \textbf{0.682} & \textbf{0.771$^\dag$} & \textbf{0.962} & \textbf{0.750$^\dag$} & \textbf{0.470} & \textbf{0.730$^\dag$} & \textbf{0.806} & \textbf{0.731$^\dag$} & 0.995 & \textbf{0.740} & \textbf{0.866$^\dag$} & \textbf{0.844} & \textbf{0.811$^\dag$} & \textbf{0.948$^\ddag$} & \textbf{0.977} & \textbf{0.808$^\dag$} & \textbf{0.970} & \textbf{0.376$^\ddag$} & \textbf{0.562} \\
\bottomrule
\end{tabular}
}
\vspace{1em}

\resizebox{\textwidth}{!}{
\begin{tabular}{l!{\vrule}c!{\vrule}@{\hspace{0.3pt}}c@{\hspace{0.3pt}}@{\hspace{0.3pt}}c@{\hspace{0.3pt}}@{\hspace{0.3pt}}c@{\hspace{0.3pt}}@{\hspace{0.3pt}}c@{\hspace{0.3pt}}@{\hspace{0.3pt}}c@{\hspace{0.3pt}}@{\hspace{0.3pt}}c@{\hspace{0.3pt}}@{\hspace{0.3pt}}c@{\hspace{0.3pt}}@{\hspace{0.3pt}}c@{\hspace{0.3pt}}@{\hspace{0.3pt}}c@{\hspace{0.3pt}}@{\hspace{0.3pt}}c@{\hspace{0.3pt}}@{\hspace{0.3pt}}c@{\hspace{0.3pt}}@{\hspace{0.3pt}}c@{\hspace{0.3pt}}@{\hspace{0.3pt}}c@{\hspace{0.3pt}}@{\hspace{0.3pt}}c@{\hspace{0.3pt}}@{\hspace{0.3pt}}c@{\hspace{0.3pt}}@{\hspace{0.3pt}}c@{\hspace{0.3pt}}@{\hspace{0.3pt}}c@{\hspace{0.3pt}}@{\hspace{0.3pt}}c@{\hspace{0.3pt}}@{\hspace{0.3pt}}c@{\hspace{0.3pt}}@{\hspace{0.3pt}}c@{\hspace{0.3pt}}@{\hspace{0.3pt}}c@{\hspace{0.3pt}}@{\hspace{0.3pt}}c@{\hspace{0.3pt}}@{\hspace{0.3pt}}c@{\hspace{0.3pt}}@{\hspace{0.3pt}}c@{\hspace{0.3pt}}@{\hspace{0.3pt}}c@{\hspace{0.3pt}}}
\toprule
Model & \textbf{Rank} & QS$\downarrow$ & AB$\downarrow$ & AI$\downarrow$ & BO$\downarrow$ & BOS$\downarrow$ & CA$\downarrow$ & CH$\downarrow$ & CL$\downarrow$ & CO$\downarrow$ & CP$\downarrow$ & CPU$\downarrow$ & DIA$\downarrow$ & EN$\downarrow$ & FR$\downarrow$ & GR$\downarrow$ & KI$\downarrow$ & LI$\downarrow$ & MU$\downarrow$ & PL$\downarrow$ & SE$\downarrow$ & SO$\downarrow$ & SP$\downarrow$ & STO$\downarrow$ & TE$\downarrow$ & WIS$\downarrow$ \\
\midrule
XGB & 5.20 & 1.047 & 2.333 & 2.247 & 3.424 & 3.545 & 0.169 & 0.593 & 0.300$^\ddag$ & 5.050 & 6.367 & 3.112 & 1157 & 0.476 & 0.980 & 0.015 & 0.153 & 2.858$^\dag$ & 30.15 & 238.0 & 0.646$^\dag$ & 24.62 & 0.169 & 1.811 & 6.530 & 39.75 \\
CB & 2.96 & 1.000 & 2.243 & 1.445$^\ddag$ & 0.587 & 3.159 & 0.133$^\dag$ & 0.540$^\ddag$ & 0.272$^\dag$ & 4.305$^\dag$ & 2.288$^\dag$ & 2.713$^\dag$ & 529.4$^\dag$ & 0.433$^\dag$ & 0.987 & 0.007 & 0.093 & 3.011 & 10.87$^\ddag$ & 227.5 & 0.676 & 24.22 & 0.108 & 0.689$^\ddag$ & 1.544$^\dag$ & 36.69 \\
\arrayrulecolor{lightgray}\hline\arrayrulecolor{black}
TT & 7.14 & — & — & — & — & 5.203 & — & — & 1.369 & — & — & — & 1039 & — & — & — & — & — & 153.5 & 227.9 & 0.801 & 28.83 & — & — & — & — \\
\textbf{+PET} & \textbf{5.71} & — & — & — & — & \textbf{5.173} & — & — & \textbf{1.327} & — & — & — & \textbf{1025} & — & — & — & — & — & \textbf{153.0} & \textbf{226.9} & \textbf{0.713} & \textbf{28.08} & — & — & — & — \\
\arrayrulecolor{lightgray}\hline\arrayrulecolor{black}
ST & 3.64 & 0.939$^\ddag$ & 2.187 & \textbf{1.957} & 1.060 & 3.194 & \textbf{0.145} & 0.575 & 0.491 & 5.169 & 2.339 & \textbf{2.791$^\ddag$} & 537.5 & \textbf{0.547} & \textbf{0.947} & 0.006 & 0.067 & \textbf{2.940} & 12.40 & 224.6$^\ddag$ & 0.717 & 17.53$^\ddag$ & \textbf{0.100$^\dag$} & 0.994 & 2.485 & 35.38$^\ddag$ \\
\textbf{+PET} & \textbf{2.84$^\dag$} & \textbf{0.932$^\dag$} & \textbf{2.160$^\dag$} & 2.005 & \textbf{1.039} & \textbf{3.090$^\ddag$} & 0.146 & \textbf{0.547} & \textbf{0.376} & \textbf{4.984$^\ddag$} & \textbf{2.332$^\ddag$} & 2.807 & \textbf{535.5$^\ddag$} & 0.550 & 0.950 & \textbf{0.006$^\ddag$} & \textbf{0.064$^\dag$} & 2.940 & \textbf{9.717$^\dag$} & \textbf{224.1$^\dag$} & \textbf{0.715} & \textbf{17.51$^\dag$} & 0.100$^\dag$ & \textbf{0.989} & \textbf{1.818$^\ddag$} & \textbf{35.20$^\dag$} \\
\arrayrulecolor{lightgray}\hline\arrayrulecolor{black}
FT & 4.08 & 0.971 & 2.200 & 1.466 & 0.327$^\ddag$ & 3.232 & \textbf{0.141$^\ddag$} & \textbf{0.540$^\dag$} & 0.484 & 5.079 & 2.358 & 2.856 & 2468 & 0.484 & 0.739$^\ddag$ & 0.006 & 0.070 & 2.936 & 25.85 & 592.0 & 0.691 & 17.85 & 0.106 & 0.695 & \textbf{4.113} & 38.01 \\
\textbf{+PET} & \textbf{2.88$^\ddag$} & \textbf{0.961} & \textbf{2.168$^\ddag$} & \textbf{1.310$^\dag$} & \textbf{0.277$^\dag$} & \textbf{2.993$^\dag$} & 0.141 & 0.540 & \textbf{0.315} & \textbf{5.045} & \textbf{2.339} & \textbf{2.827} & \textbf{2447} & \textbf{0.472$^\ddag$} & \textbf{0.716$^\dag$} & \textbf{0.005$^\dag$} & \textbf{0.067$^\ddag$} & \textbf{2.880$^\ddag$} & \textbf{22.98} & \textbf{591.4} & \textbf{0.666$^\ddag$} & \textbf{17.83} & \textbf{0.104} & \textbf{0.676$^\dag$} & 4.123 & \textbf{37.92} \\
\bottomrule
\end{tabular}
}
\caption{Performance comparison on classification and regression datasets. Each result is averaged over 5 random seeds. $\uparrow$ indicates balanced accuracy, $\downarrow$ indicates RMSE. Rank = the mean rank across all datasets (lower is better). \textbf{Bold} = better between baseline transformer and its Tab-PET variant. $^\dag$ = best across all methods, $^\ddag$ = second-best across all methods. Models: XGB = XGBoost, CB = CatBoost, TT = TabTransformer, ST = SAINT, FT = FT-Transformer, \textbf{+PET} = corresponding Tab-PET variant. The dataset names corresponding to the abbreviations are provided in technical appendix \ref{app:2.1}.}
\label{tab:simplified_results}
\end{table*}

\subsection{Graph Estimation Approaches Analysis}

\textbf{Performance:} 
To evaluate how different graph estimation approaches influence downstream performance, we compare five representative approaches on 50 datasets using FT-Transformer as the backbone.
As shown in Table~\ref{tab:graph_estimation_results}, association-based approaches consistently outperform causality-based ones across both tasks. Spearman correlation achieves the highest average improvement, followed closely by Pearson. Additionally, Spearman exhibits the most consistent positive gains, rarely showing performance degradation. In contrast, causal discovery approaches NOTEARS and LiNGAM exhibit relatively weaker performance improvement. Chow-Liu, which constructs tree-structured dependency graphs, doesn't perform well on classification. Detailed results for all approaches are provided in technical appendix \ref{app:7.1}.

\noindent\textbf{Computational Cost:} 
Addressing computational efficiency concerns, we report the average computational time (in minutes) for graph estimation and PE creation across all 50 datasets in Table~\ref{tab:graph_estimation_results}. The introduced computational overhead is minimal; for instance, Spearman graphs add only 0.79 minutes on average. Detailed timing results are provided in technical appendix \ref{app:7.2}.


\begin{figure}[t]
    \centering
    \includegraphics[width=\linewidth]{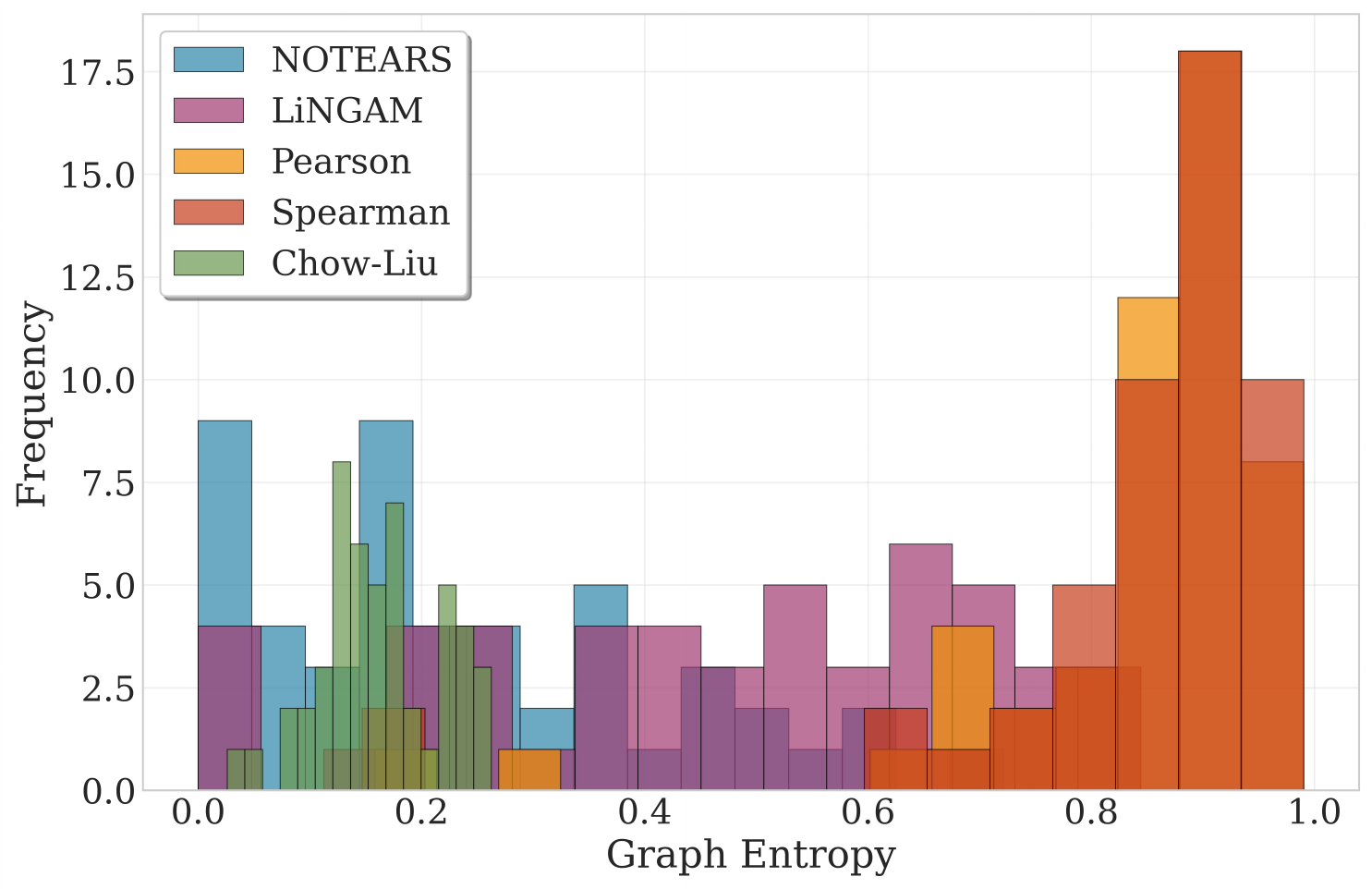}
    \caption{Graph entropy distributions across five graph estimation approaches. Varying bar widths reflect the different value ranges spanned by each approach.}
    \label{fig:graph_entropy}
\end{figure}

\noindent\textbf{Graph Entropy:}
To better understand why association-based approaches outperform causality-based ones, we analyze structural properties via graph entropy, a measure quantifying the uniformity of edge weight distributions.
For a graph with $n$ feature nodes, we compute the normalized entropy for each node $i$. Given the edge weights $w_{ij}$ to other nodes, the normalized entropy $H_i$ is defined as:
$H_i = -\frac{\sum_{j \neq i} p_{ij} \ln p_{ij}}{\ln(n-1)}$, where $p_{ij} = \frac{w_{ij}}{\sum_{j \neq i} w_{ij}}$. The overall graph entropy is then obtained by averaging $H_i$ across all nodes. Higher entropy indicates more uniform, densely connected graphs, while lower entropy suggests sparse, highly structured graphs with concentrated edge weights.


Figure~\ref{fig:graph_entropy} reveals a clear pattern between graph entropy and downstream performance. For all 50 datasets tested here, causal methods (NOTEARS and LiNGAM) concentrate in the low graph entropy region, producing sparse, highly constrained graphs. Spearman and Pearson correlations generate graphs with high entropy. These denser structures align with the strongest performance gains, suggesting that PEs yield more useful information when derived from dense feature dependencies rather than sparse causal structures. Additional causal–association comparison analyses are provided in technical appendix \ref{app:8.1}.

\subsection{Classification and Regression: Main Results}

Based on Table \ref{tab:graph_estimation_results}, we select the best-performing approach (Spearman) for integrating PEs into transformer architectures. 
Table~\ref{tab:simplified_results} compares performance across tree-based and transformer-based models on 50 datasets, with detailed results with standard deviations in technical appendix \ref{app:7.3}.
Metrics include balanced accuracy for classification and RMSE for regression, based on five-fold cross-validation. 
For TabTransformer, we use only datasets with multiple categorical variables, since TabTransformer's architecture only applies embeddings (and thus PEs) to categorical features, while continuous features bypass the embedding layer entirely.
Overall, Tab-PET consistently improves performance across multiple transformer architectures and tasks.

To understand performance better, following common practice \cite{gorishniy2021revisiting, gorishniy2022embeddings, gorishniy2023tabr, mcelfresh2023neural, ye2025disentangling}, we report mean ranks in Table~\ref{tab:simplified_results} and find that Tab-PET methods achieve the best overall performance, surpassing GBDTs and baseline transformers in both classification and regression tasks, with Tab-PET variants of FT-Transformer and SAINT ranking as the top two methods overall.
Notably, Tab-PET shows statistically significant improvements against no-PE baselines across all transformer architectures ($p < 0.05$, Wilcoxon signed-rank tests), with detailed p-values provided in technical appendix \ref{app:8.6}.

\subsection{Comparing with Learnable Positional Encodings}

A key consideration in our study is whether fixed PEs, derived from the intrinsic structure of tabular data, outperform learnable PEs that adapt based on data input. This question echoes ongoing debates in NLP where it's widely accepted that learnable PEs can perform well when ample data is available, but may falter under low-data regimes \cite{radford2018improving}.

To explore this in the tabular domain, we compare performance improvements between learnable PEs and those generated via Tab-PET across all classification and regression datasets. These results, detailed in Table~\ref{tab:main_learnable_comparison}, reveal that Tab-PET consistently achieves higher average gains than learnable alternatives. This suggests that for tabular datasets, which are typically smaller in size, incorporating graph-structured positional information as Tab-PET does provide a substantial advantage. Detailed results are provided in technical appendix \ref{app:7.4}, while further analysis and ablation studies on Tab-PET are shown in technical appendix \ref{app:8}.

\begin{table}[t]
\centering
\scriptsize
\setlength{\tabcolsep}{4pt}
\renewcommand{\arraystretch}{1}
\begin{tabular}{l!{\vrule}cc!{\vrule}cc!{\vrule}cc!{\vrule}cc}
\toprule
\multirow{2}{*}{Method} & \multicolumn{2}{c}{Avg Improv.\%} & \multicolumn{2}{c}{Median Improv.\%} & \multicolumn{2}{c}{Min Improv.\%} & \multicolumn{2}{c}{Win Rate\%} \\
\cmidrule{2-9}
& C & R & C & R & C & R & C & R \\
\midrule
Learnable & 0.04 & 0.62 & -0.08 & 0.05 & -3.1 & -8.6 & 12 & 8 \\
Tab-PET & \textbf{1.72} & \textbf{4.34} & \textbf{1.39} & \textbf{1.92} & \textbf{-0.1} & \textbf{-0.2} & \textbf{88} & \textbf{92} \\
\bottomrule
\end{tabular}
\caption{Comparison between Tab-PET and Learnable PE. C = Classification, R = Regression.}
\label{tab:main_learnable_comparison}
\end{table}

\begin{figure}[t]
    \centering
    \includegraphics[width=\linewidth]{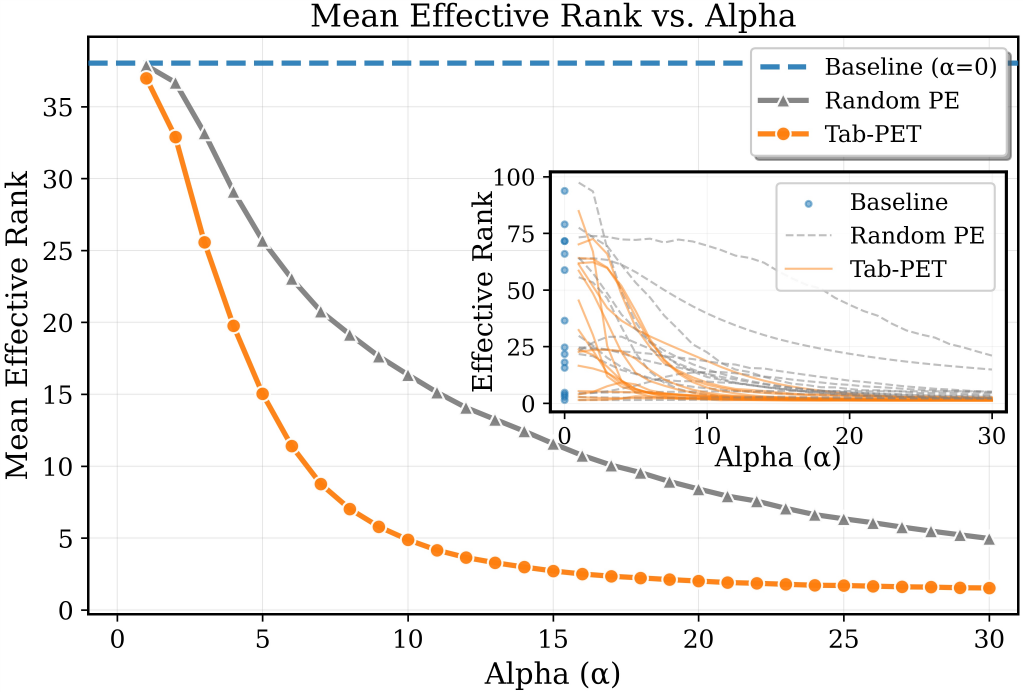}
    \caption{Empirical validation of effective rank reduction. Main plot shows the mean effective rank vs. $\alpha$ for Tab-PET, Random PE, and baseline. Inset shows per-dataset trends.}
    \label{fig:effective_rank}
\end{figure}

\subsection{PEs and Effective Rank on Real Datasets}

To empirically validate our theoretical results, we conduct experiments across 15 real-world tabular datasets where we measure the effective rank of features, comparing three conditions: (1) \textit{Baseline} without PEs ($\alpha = 0$); (2) \textit{Tab-PET} with graph-derived PEs; (3) \textit{Random PE} with the same dimensionality and statistical properties.

\noindent\textbf{Experimental Setup:}
We use an FT-Transformer architecture with a single layer and single attention head to isolate the effect of PEs on effective rank. For each PE type, we vary the scaling parameter $\alpha \in \{1, 2, 3, \ldots, 30\}$ and compute the effective rank of CLS token embeddings before the final prediction via fully connected layers. Complete experimental details are provided in technical appendix \ref{app:1.1}.

\noindent\textbf{Observations:}
Figure \ref{fig:effective_rank} presents our main findings, which strongly align with Theorems \ref{thm:1} and \ref{thm:2}. Specifically, as $\alpha$ increases we see: (1) Both Tab-PET and random PE reduce effective rank compared to baseline across all $\alpha$ values, confirming that PEs enable the architecture to lower representation complexity when needed; (2) Tab-PET's effective rank is significantly lower than random PE, with the gap widening as $\alpha$ increases initially, before converging again; (3) The exponential-like decay in effective rank aligns with our theoretical takeaways. Further analysis is provided in technical appendix \ref{app:1.2}.


\section{Conclusion}

PEs, traditionally overlooked in tabular learning, can greatly enhance transformer performance by incorporating graph-derived structural priors. Evaluation on 50 datasets across multiple baselines establishes that Tab-PET (particularly Spearman graph estimation) extensions show best performance overall when compared with gradient boosting approaches and the non-PE baselines. Our theoretical analysis confirms that PEs can reduce effective rank of embeddings, and even more so when meaningful structure exists. Our synthetic studies highlight the significant performance improvements when the data has underlying structure. Overall, our work highlights the importance of actively incorporating structural inductive biases in tabular data learning.

\section*{Acknowledgments}

This research is supported by A*STAR, CISCO Systems (USA) Pte. Ltd and the National University of Singapore under its Cisco-NUS Accelerated Digital Economy Corporate Laboratory (Award I21001E0002).

\bibliography{main}

\begin{thebibliography}{39}
\providecommand{\natexlab}[1]{#1}

\bibitem[{Akiba et~al.(2019)Akiba, Sano, Yanase, Ohta, and Koyama}]{akiba2019optuna}
Akiba, T.; Sano, S.; Yanase, T.; Ohta, T.; and Koyama, M. 2019.
\newblock Optuna: A next-generation hyperparameter optimization framework.
\newblock In \emph{Proceedings of the 25th ACM SIGKDD international conference on knowledge discovery \& data mining}, 2623--2631.

\bibitem[{Arora et~al.(2019)Arora, Cohen, Hu, and Luo}]{arora2019implicit}
Arora, S.; Cohen, N.; Hu, W.; and Luo, Y. 2019.
\newblock Implicit regularization in deep matrix factorization.
\newblock \emph{Advances in neural information processing systems}, 32.

\bibitem[{Bahri et~al.(2021)Bahri, Jiang, Tay, and Metzler}]{bahri2021scarf}
Bahri, D.; Jiang, H.; Tay, Y.; and Metzler, D. 2021.
\newblock Scarf: Self-supervised contrastive learning using random feature corruption.
\newblock \emph{arXiv preprint arXiv:2106.15147}.

\bibitem[{Cant{\"u}rk et~al.(2023)Cant{\"u}rk, Liu, Lapointe-Gagn{\'e}, L{\'e}tourneau, Wolf, Beaini, and Ramp{\'a}{\v{s}}ek}]{canturk2023graph}
Cant{\"u}rk, S.; Liu, R.; Lapointe-Gagn{\'e}, O.; L{\'e}tourneau, V.; Wolf, G.; Beaini, D.; and Ramp{\'a}{\v{s}}ek, L. 2023.
\newblock Graph positional and structural encoder.
\newblock \emph{arXiv preprint arXiv:2307.07107}.

\bibitem[{Chen and Guestrin(2016)}]{chen2016xgboost}
Chen, T.; and Guestrin, C. 2016.
\newblock Xgboost: A scalable tree boosting system.
\newblock In \emph{Proceedings of the 22nd acm sigkdd international conference on knowledge discovery and data mining}, 785--794.

\bibitem[{Chow and Liu(1968)}]{chow1968approximating}
Chow, C.; and Liu, C. 1968.
\newblock Approximating discrete probability distributions with dependence trees.
\newblock \emph{IEEE transactions on Information Theory}, 14(3): 462--467.

\bibitem[{Chu et~al.(2021)Chu, Zhang, Tian, Wei, and Xia}]{chu2021we}
Chu, X.; Zhang, B.; Tian, Z.; Wei, X.; and Xia, H. 2021.
\newblock Do we really need explicit position encodings for vision transformers.
\newblock \emph{arXiv preprint arXiv:2102.10882}, 3(8).

\bibitem[{Dosovitskiy et~al.(2020)Dosovitskiy, Beyer, Kolesnikov, Weissenborn, Zhai, Unterthiner, Dehghani, Minderer, Heigold, Gelly et~al.}]{dosovitskiy2020image}
Dosovitskiy, A.; Beyer, L.; Kolesnikov, A.; Weissenborn, D.; Zhai, X.; Unterthiner, T.; Dehghani, M.; Minderer, M.; Heigold, G.; Gelly, S.; et~al. 2020.
\newblock An image is worth 16x16 words: Transformers for image recognition at scale.
\newblock \emph{arXiv preprint arXiv:2010.11929}.

\bibitem[{Dwivedi and Bresson(2020)}]{dwivedi2020generalization}
Dwivedi, V.~P.; and Bresson, X. 2020.
\newblock A generalization of transformer networks to graphs.
\newblock \emph{arXiv preprint arXiv:2012.09699}.

\bibitem[{Ghosh and Motani(2023)}]{ghosh2023local}
Ghosh, R.; and Motani, M. 2023.
\newblock Local intrinsic dimensional entropy.
\newblock In \emph{Proceedings of the AAAI Conference on Artificial Intelligence}, volume~37, 7714--7721.

\bibitem[{Gong, Chung, and Glass(2021)}]{gong2021ast}
Gong, Y.; Chung, Y.-A.; and Glass, J. 2021.
\newblock Ast: Audio spectrogram transformer.
\newblock \emph{arXiv preprint arXiv:2104.01778}.

\bibitem[{Gorishniy, Rubachev, and Babenko(2022)}]{gorishniy2022embeddings}
Gorishniy, Y.; Rubachev, I.; and Babenko, A. 2022.
\newblock On embeddings for numerical features in tabular deep learning.
\newblock \emph{Advances in Neural Information Processing Systems}, 35: 24991--25004.

\bibitem[{Gorishniy et~al.(2023)Gorishniy, Rubachev, Kartashev, Shlenskii, Kotelnikov, and Babenko}]{gorishniy2023tabr}
Gorishniy, Y.; Rubachev, I.; Kartashev, N.; Shlenskii, D.; Kotelnikov, A.; and Babenko, A. 2023.
\newblock Tabr: Tabular deep learning meets nearest neighbors in 2023.
\newblock \emph{arXiv preprint arXiv:2307.14338}.

\bibitem[{Gorishniy et~al.(2021)Gorishniy, Rubachev, Khrulkov, and Babenko}]{gorishniy2021revisiting}
Gorishniy, Y.; Rubachev, I.; Khrulkov, V.; and Babenko, A. 2021.
\newblock Revisiting deep learning models for tabular data.
\newblock \emph{Advances in neural information processing systems}, 34: 18932--18943.

\bibitem[{He et~al.(2016)He, Zhang, Ren, and Sun}]{he2016identity}
He, K.; Zhang, X.; Ren, S.; and Sun, J. 2016.
\newblock Identity mappings in deep residual networks.
\newblock In \emph{European conference on computer vision}, 630--645. Springer.

\bibitem[{Herzig et~al.(2020)Herzig, Nowak, M{\"u}ller, Piccinno, and Eisenschlos}]{herzig2020tapas}
Herzig, J.; Nowak, P.~K.; M{\"u}ller, T.; Piccinno, F.; and Eisenschlos, J.~M. 2020.
\newblock TaPas: Weakly supervised table parsing via pre-training.
\newblock \emph{arXiv preprint arXiv:2004.02349}.

\bibitem[{Huang et~al.(2020)Huang, Khetan, Cvitkovic, and Karnin}]{huang2020tabtransformer}
Huang, X.; Khetan, A.; Cvitkovic, M.; and Karnin, Z. 2020.
\newblock Tabtransformer: Tabular data modeling using contextual embeddings.
\newblock \emph{arXiv preprint arXiv:2012.06678}.

\bibitem[{Huh et~al.(2021)Huh, Mobahi, Zhang, Cheung, Agrawal, and Isola}]{huh2021low}
Huh, M.; Mobahi, H.; Zhang, R.; Cheung, B.; Agrawal, P.; and Isola, P. 2021.
\newblock The low-rank simplicity bias in deep networks.
\newblock \emph{arXiv preprint arXiv:2103.10427}.

\bibitem[{Ito et~al.(2025)Ito, Zhu, Chen, Koutra, and Wiens}]{ito2025learning}
Ito, M.; Zhu, J.; Chen, D.; Koutra, D.; and Wiens, J. 2025.
\newblock Learning laplacian positional encodings for heterophilous graphs.
\newblock \emph{arXiv preprint arXiv:2504.20430}.

\bibitem[{Kadra et~al.(2021)Kadra, Lindauer, Hutter, and Grabocka}]{kadra2021well}
Kadra, A.; Lindauer, M.; Hutter, F.; and Grabocka, J. 2021.
\newblock Well-tuned simple nets excel on tabular datasets.
\newblock \emph{Advances in neural information processing systems}, 34: 23928--23941.

\bibitem[{Klambauer et~al.(2017)Klambauer, Unterthiner, Mayr, and Hochreiter}]{klambauer2017self}
Klambauer, G.; Unterthiner, T.; Mayr, A.; and Hochreiter, S. 2017.
\newblock Self-normalizing neural networks.
\newblock \emph{Advances in neural information processing systems}, 30.

\bibitem[{Lee et~al.(2019)Lee, Danieletto, Miotto, Cherng, and Dudley}]{lee2019scaling}
Lee, H.-C.; Danieletto, M.; Miotto, R.; Cherng, S.~T.; and Dudley, J.~T. 2019.
\newblock Scaling structural learning with NO-BEARS to infer causal transcriptome networks.
\newblock In \emph{Pacific Symposium on Biocomputing 2020}, 391--402. World Scientific.

\bibitem[{Leng, Ghosh, and Motani(2025)}]{LengTabPET}
Leng, Y.; Ghosh, R.; and Motani, M. 2025.
\newblock Tab-PET: Graph-Based Positional Encoding for Tabular Transformers.
\newblock \emph{https://github.com/kentridgeai/Tab-PET/blob/main/Tab-PET-Arxiv.pdf?raw=true}.

\bibitem[{McElfresh et~al.(2023)McElfresh, Khandagale, Valverde, Prasad~C, Ramakrishnan, Goldblum, and White}]{mcelfresh2023neural}
McElfresh, D.; Khandagale, S.; Valverde, J.; Prasad~C, V.; Ramakrishnan, G.; Goldblum, M.; and White, C. 2023.
\newblock When do neural nets outperform boosted trees on tabular data?
\newblock \emph{Advances in Neural Information Processing Systems}, 36: 76336--76369.

\bibitem[{Nakada and Imaizumi(2020)}]{nakada_ID_generalization}
Nakada, R.; and Imaizumi, M. 2020.
\newblock Adaptive approximation and generalization of deep neural network with intrinsic dimensionality.
\newblock \emph{J. Mach. Learn. Res.}, 21(1).

\bibitem[{Popov, Morozov, and Babenko(2019)}]{popov2019neural}
Popov, S.; Morozov, S.; and Babenko, A. 2019.
\newblock Neural oblivious decision ensembles for deep learning on tabular data.
\newblock \emph{arXiv preprint arXiv:1909.06312}.

\bibitem[{Prokhorenkova et~al.(2018)Prokhorenkova, Gusev, Vorobev, Dorogush, and Gulin}]{prokhorenkova2018catboost}
Prokhorenkova, L.; Gusev, G.; Vorobev, A.; Dorogush, A.~V.; and Gulin, A. 2018.
\newblock CatBoost: unbiased boosting with categorical features.
\newblock \emph{Advances in neural information processing systems}, 31.

\bibitem[{Radford et~al.(2018)Radford, Narasimhan, Salimans, and Sutskever}]{radford2018improving}
Radford, A.; Narasimhan, K.; Salimans, T.; and Sutskever, I. 2018.
\newblock Improving Language Understanding by Generative Pre-Training.
\newblock \emph{OpenAI Technical Report}.

\bibitem[{Roy and Vetterli(2007)}]{roy2007effective}
Roy, O.; and Vetterli, M. 2007.
\newblock The effective rank of a matrix: A measure of effective dimensionality.
\newblock \emph{European Signal Processing Conference}, 606--610.

\bibitem[{Shavitt and Segal(2018)}]{shavitt2018regularization}
Shavitt, I.; and Segal, E. 2018.
\newblock Regularization learning networks: deep learning for tabular datasets.
\newblock \emph{Advances in neural information processing systems}, 31.

\bibitem[{Shimizu et~al.(2006)Shimizu, Hoyer, Hyv{\"a}rinen, Kerminen, and Jordan}]{shimizu2006linear}
Shimizu, S.; Hoyer, P.~O.; Hyv{\"a}rinen, A.; Kerminen, A.; and Jordan, M. 2006.
\newblock A linear non-Gaussian acyclic model for causal discovery.
\newblock \emph{Journal of Machine Learning Research}, 7(10).

\bibitem[{Somepalli et~al.(2021)Somepalli, Goldblum, Schwarzschild, Bruss, and Goldstein}]{somepalli2021saint}
Somepalli, G.; Goldblum, M.; Schwarzschild, A.; Bruss, C.~B.; and Goldstein, T. 2021.
\newblock Saint: Improved neural networks for tabular data via row attention and contrastive pre-training.
\newblock \emph{arXiv preprint arXiv:2106.01342}.

\bibitem[{Vaswani et~al.(2017)Vaswani, Shazeer, Parmar, Uszkoreit, Jones, Gomez, Kaiser, and Polosukhin}]{vaswani2017attention}
Vaswani, A.; Shazeer, N.; Parmar, N.; Uszkoreit, J.; Jones, L.; Gomez, A.~N.; Kaiser, {\L}.; and Polosukhin, I. 2017.
\newblock Attention is all you need.
\newblock \emph{Advances in neural information processing systems}, 30.

\bibitem[{Wang and Sun(2022)}]{wang2022transtab}
Wang, Z.; and Sun, J. 2022.
\newblock Transtab: Learning transferable tabular transformers across tables.
\newblock \emph{Advances in Neural Information Processing Systems}, 35: 2902--2915.

\bibitem[{Wikipedia(2025)}]{wiki_fiedler}
Wikipedia. 2025.
\newblock Algebraic connectivity.
\newblock \url{https://en.wikipedia.org/wiki/Algebraic_connectivity}.
\newblock Accessed: 2025-11-17.

\bibitem[{Xu et~al.(2021)Xu, Zhang, Zhang, and Tao}]{xu2021vitae}
Xu, Y.; Zhang, Q.; Zhang, J.; and Tao, D. 2021.
\newblock Vitae: Vision transformer advanced by exploring intrinsic inductive bias.
\newblock \emph{Advances in neural information processing systems}, 34: 28522--28535.

\bibitem[{Ye et~al.(2024)Ye, Fan, Song, Zheng, Zhao, Guo, and Chang}]{ye2024ptarl}
Ye, H.; Fan, W.; Song, X.; Zheng, S.; Zhao, H.; Guo, D.; and Chang, Y. 2024.
\newblock Ptarl: Prototype-based tabular representation learning via space calibration.
\newblock \emph{arXiv preprint arXiv:2407.05364}.

\bibitem[{Ye et~al.(2025)Ye, Tan, Hu, Yang, Cheng, and Huang}]{ye2025disentangling}
Ye, J.; Tan, Z.; Hu, Y.; Yang, X.; Cheng, G.; and Huang, K. 2025.
\newblock Disentangling Tabular Data Towards Better One-Class Anomaly Detection.
\newblock In \emph{Proceedings of the AAAI Conference on Artificial Intelligence}, volume~39, 13061--13068.

\bibitem[{Zheng et~al.(2018)Zheng, Aragam, Ravikumar, and Xing}]{zheng2018dags}
Zheng, X.; Aragam, B.; Ravikumar, P.~K.; and Xing, E.~P. 2018.
\newblock Dags with no tears: Continuous optimization for structure learning.
\newblock \emph{Advances in neural information processing systems}, 31.

\end{thebibliography}

\appendix
\onecolumn

\begin{center}
\bfseries\LARGE
{\LARGE Technical Appendix}
\end{center}

\tableofcontents










\newpage

\section{Empirical Validation of Effective Rank Reduction on Real Datasets} \label{app:1}

To empirically validate our theoretical results on the effective rank reduction capabilities of PEs, we conduct experiments across multiple real-world tabular datasets. Our experiments demonstrate that graph-derived PEs consistently reduce the effective rank of learned representations compared to both baseline models and random PEs, which provides strong empirical support for our theorems. Furthermore, we also find that even random PEs can yield significant reductions in effective rank, which shows the importance of incorporating PEs in general in transformer-based architectures for tabular data. 

\subsection{Experimental Setup} \label{app:1.1}

\subsubsection{Dataset Selection} We evaluate our approach on 15 datasets selected from Table \ref{tab:dataset_properties}, specifically choosing datasets with sample sizes between 4,500 and 50,000 to ensure computational feasibility while maintaining statistical significance. The selected datasets span both classification and regression tasks, including GE, SA, CHU, EY, FI, HE, PH, SY, WI, WIN, CA, CP, CPU, GR, KI. Please refer to Table \ref{tab:dataset_properties} for the full dataset name and version.

\subsubsection{Model Architecture} To isolate the effect of PEs on effective rank, we employ a simplified FT-Transformer architecture with a single layer and single attention head ($n\_layers=1$, $n\_heads=1$). This minimal configuration allows us to directly observe the impact of PEs without the confounding effects of deep architectural components. The model maintains an effective total token dimension of 192. The graph-derived PE method we use is Spearman.

\subsubsection{Positional Encoding Generation} We compare three types of positional encodings:
\begin{itemize}
    \item \textbf{Baseline}: Without PEs by setting $\alpha = 0$, which can effectively nullifying PE influence
    \item \textbf{Tab-PET}: Real graph-derived PEs with varying $\alpha \in \{1, 2, 3, \ldots, 30\}$
    \item \textbf{Random PE}: Randomly generated PEs with the same dimensionality, same statistical properties, same $\alpha \in \{1, 2, 3, \ldots, 30\}$, and same normalization as the real graph-derived PEs, providing a fair comparison.
\end{itemize}

\subsubsection{Effective Rank Computation} We compute the effective rank of the CLS token embeddings immediately before the final fully connected layer (classification/regression head). For a matrix $\mathbf{X} \in \mathbb{R}^{n \times d}$ representing CLS embeddings from $n$ samples, we calculate the effective rank as:
\begin{equation}
r_{\text{eff}}(\mathbf{X}) = \exp\left(-\sum_{i=1}^{r} \tilde{\sigma}_i \log \tilde{\sigma}_i\right),
\end{equation}
where $\tilde{\sigma}_i = \sigma_i / \sum_{j=1}^{r} \sigma_j$ are the normalized singular values obtained from SVD decomposition $\mathbf{X} = \mathbf{U}\boldsymbol{\Sigma}\mathbf{V}^T$, and $r$ is the rank of $\mathbf{X}$. This formulation captures the intrinsic dimensionality of the learned representations through the Shannon entropy of the singular value distribution.

\subsubsection{Training Protocol} Each model is trained for up to 50 epochs with early stopping (patience=20, minimum epochs=10) using AdamW optimizer with learning rate $1 \times 10^{-4}$. We use a 60-20-20 train-validation-test split and ensure reproducibility through fixed random seeds. For each PE type, we compute effective ranks across different $\alpha$ values.

\subsection{Results} \label{app:1.2}

\subsubsection{Comparison between Tab-PET and Random PE}

\begin{figure}[H]
    \centering
    \includegraphics[width=1\linewidth]{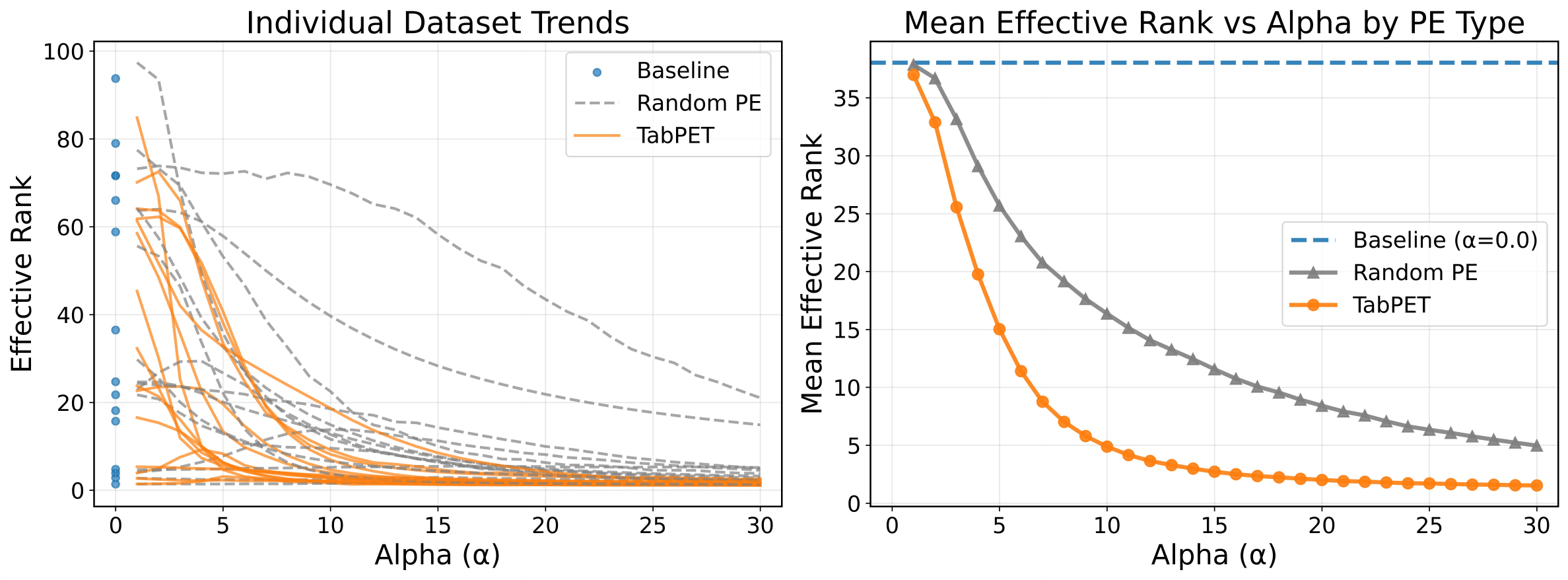}
    \caption{Empirical Validation of Effective Rank Reduction Theory. Left panel displays individual dataset trends across 15 tabular datasets for both Tab-PET (orange solid lines) and Random PE (gray dashed lines) compared to baseline models without PEs (blue scattered points at $\alpha=0$). Right panel shows the mean effective rank across all datasets.}
    \label{fig:effective_rank_results}
\end{figure}

Figure \ref{fig:effective_rank_results} presents our main empirical findings, showing both individual dataset trends (left panel) and mean effective ranks across all datasets (right panel). The results provide strong empirical support for our theorems.

The left panel of Figure \ref{fig:effective_rank_results} demonstrates that the rank reduction pattern holds consistently across diverse datasets, despite their varying characteristics. This universality suggests that our theoretical framework captures fundamental properties of PE integration in transformer architectures.

The right panel of Figure \ref{fig:effective_rank_results} shows that the baseline models (with $\alpha = 0$) achieve a mean effective rank of $38.04$. In contrast, Tab-PET with $\alpha = 1$ reduces this to $36.98$, while random PE achieves $37.86$. This immediate reduction upon introducing PEs validates our theoretical framework.

As $\alpha$ increases, we observe different behaviors between Tab-PET and random PE. Tab-PET exhibits rapid and consistent effective rank reduction, reaching $15.03$ at $\alpha = 5$ and further decreasing to $1.53$ at $\alpha = 30$. This steep decline demonstrates the power of structured PEs in reducing representation complexity and the generalization gap. On the other hand, random PE shows more gradual reduction, decreasing from $37.86$ at $\alpha = 1$ to $25.69$ at $\alpha = 5$ and $4.97$ at $\alpha = 30$. While still reducing effective rank, the improvement is substantially less pronounced than Tab-PET.

\subsubsection{Theoretical Alignment} Our empirical results strongly align with Theorems 1 and 2. Specifically:

\begin{enumerate}
    \item \textbf{Consistent Rank Reduction}: Both Tab-PET and random PE reduce effective rank compared to baseline, confirming that PEs enable the architecture in lowering representation complexity if needed.
    
    \item \textbf{Structured vs. Random PEs}: Tab-PET consistently outperforms random PE across all $\alpha$ values, with the gap widening as $\alpha$ increases.
    
    \item \textbf{Alpha Scaling}: The exponential-like decay in effective rank as $\alpha$ increases aligns with our theoretical prediction that larger $\alpha$ values amplify the rank reduction effect, with Tab-PET showing steeper decay due to its structured nature (Theorem 2 part (b)).
\end{enumerate}

\subsubsection{Practical Implications} The effective rank reduction achieved by Tab-PET (up to 25× reduction from baseline to $\alpha = 30$) suggests significant potential for improved generalization. According to established connections between low-rank representations and generalization \cite{arora2019implicit,huh2021low, nakada_ID_generalization}, these results indicate that graph-derived PEs may enhance model performance by constraining the learning problem to lower-dimensional manifolds.

In summary, this empirical evaluation across 15 diverse tabular datasets provides compelling evidence for our theoretical claims, showing that graph-derived PEs consistently and substantially reduce the effective rank of learned representations compared to random PEs and no PEs. 

\newpage
\section{Dataset Information and Preprocessing}\label{app:2}

\subsection{Dataset Properties and Characteristics} \label{app:2.1}

\begin{table}[H]
\centering
\resizebox{0.68\textwidth}{!}{
\small
\setlength{\tabcolsep}{4pt}
\renewcommand{\arraystretch}{1.2}
\definecolor{lightgray}{gray}{0.7}
\begin{tabular}{l!{\vrule}c!{\vrule}c!{\vrule}c!{\vrule}c!{\vrule}c}
\toprule
Dataset & Version & Abbr. & (Features, Samples) & Cat. Features & Type \\
\midrule
Australian & 4 & AU & (14, 690) & 8 & C \\
GesturePhaseSegmentationProcessed & 1 & GE & (32, 9873) & 0 & C \\
Satellite & 1 & SA & (36, 5100) & 0 & C \\
blood-transfusion-service-center & 1 & BL & (4, 748) & 0 & C \\
churn & 1 & CHU & (20, 5000) & 4 & C \\
cmc & 1 & CM & (9, 1473) & 7 & C \\
credit-g & 1 & CR & (20, 1000) & 13 & C \\
diabetes & 1 & DI & (8, 768) & 0 & C \\
dna & 1 & DN & (180, 3186) & 180 & C \\
eye\_movements & 1 & EY & (27, 10936) & 3 & C \\
first-order-theorem-proving & 1 & FI & (51, 6118) & 0 & C \\
heloc & 3 & HE & (23, 10459) & 0 & C \\
jasmine & 1 & JA & (144, 2984) & 136 & C \\
kc1 & 1 & KC & (21, 2109) & 0 & C \\
kr-vs-kp & 1 & KR & (36, 3196) & 36 & C \\
madeline & 1 & MA & (259, 3140) & 0 & C \\
phoneme & 1 & PH & (5, 5404) & 0 & C \\
qsar-biodeg & 1 & QSA & (41, 1055) & 0 & C \\
steel-plates-fault & 3 & ST & (27, 1941) & 0 & C \\
sylvine & 1 & SY & (20, 5124) & 0 & C \\
tic-tac-toe & 1 & TI & (9, 958) & 9 & C \\
vehicle & 1 & VE & (18, 846) & 0 & C \\
wilt & 2 & WI & (5, 4839) & 0 & C \\
wine-quality-white & 1 & WIN & (11, 4898) & 0 & C \\
yeast & 1 & YE & (8, 1484) & 0 & C \\
\arrayrulecolor{lightgray}\hline\arrayrulecolor{black}
QSAR\_fish\_toxicity & 7 & QS & (6, 908) & 0 & R \\
abalone & 5 & AB & (8, 4177) & 1 & R \\
airfoil\_self\_noise & 8 & AI & (5, 1503) & 0 & R \\
bodyfat & 1 & BO & (14, 252) & 0 & R \\
boston & 1 & BOS & (13, 506) & 2 & R \\
california & 4 & CA & (8, 20640) & 0 & R \\
chscase\_census2 & 1 & CH & (7, 400) & 0 & R \\
cloud & 1 & CL & (5, 108) & 2 & R \\
concrete\_compressive\_strength & 7 & CO & (8, 1030) & 0 & R \\
cpu\_activity & 6 & CP & (21, 8192) & 0 & R \\
cpu\_small & 2 & CPU & (12, 8192) & 0 & R \\
diamonds & 8 & DIA & (9, 53940) & 3 & R \\
energy\_efficiency & 9 & EN & (8, 768) & 0 & R \\
fri\_c3\_100\_25 & 1 & FR & (25, 100) & 0 & R \\
grid\_stability & 7 & GR & (12, 10000) & 0 & R \\
kin8nm & 1 & KI & (8, 8192) & 0 & R \\
liver-disorders & 1 & LI & (5, 345) & 0 & R \\
munich-rent-index-1999 & 1 & MU & (8, 3082) & 4 & R \\
plasma\_retinol & 1 & PL & (13, 315) & 3 & R \\
sensory & 1 & SE & (11, 576) & 11 & R \\
socmob & 1 & SO & (5, 1156) & 4 & R \\
space\_ga & 1 & SP & (6, 3107) & 0 & R \\
stock & 1 & STO & (9, 950) & 0 & R \\
tecator & 1 & TE & (124, 240) & 0 & R \\
wisconsin & 1 & WIS & (32, 194) & 0 & R \\
\bottomrule
\end{tabular}
}
\caption{Properties of OpenML datasets used in experiments. The first section shows classification datasets, the second shows regression datasets. Version indicates the OpenML dataset version used. Type: C = Classification, R = Regression. Features and samples indicate the dimensionality of each dataset. Cat. Features show the number of categorical features in the dataset.}
\label{tab:dataset_properties}
\end{table}

The detailed properties of the datasets used in this paper are shown in Table \ref{tab:dataset_properties}. OpenML (https://www.openml.org) provides standardized benchmarks commonly used in tabular learning literature, ensuring fair comparison with existing methods. 

\subsection{Data Preprocessing} \label{app:2.2}

Simple pre-processing used in our work applies the following operations: (1) separation of continuous and categorical features based on data types, (2) listwise deletion of samples with missing numerical values, (3) one-hot encoding of categorical features with missing value imputation using a designated `Missing' category, and (4) cardinality reduction for high-dimensional categorical features, retaining the top 9 most frequent categories plus an `Other' category when unique values exceed 10. 

To deal with datasets with missing values, robust pre-processing extends the above approach by first removing features with $>70\%$ missing values, then choosing between two imputation strategies based on data completeness: if complete cases comprise $\geq 30\%$ of the dataset, we employ listwise deletion; otherwise, we apply iterative imputation using scikit-learn's \texttt{IterativeImputer}. All numerical features are subsequently standardized using \texttt{StandardScaler} with zero mean and unit variance normalization.

\newpage
\section{NO-TEARS and LiNGAM Hyperparameters}\label{app:3}

This section provides detailed hyperparameter configurations for the causality-based graph estimation methods used in our experiments: NO-TEARS and LiNGAM. Both methods require hyperparameter tuning to achieve optimal performance across different datasets.

For NO-TEARS, we employ a grid search approach over the L1 regularization parameter $\lambda_1$, which controls the sparsity of the learned graph structure. The data fitting loss function uses L2 norm (squared loss), while L1 regularization is applied to encourage sparse adjacency matrices. We use standard convergence criteria and set appropriate computational limits for efficiency.

For LiNGAM, we focus on the threshold parameter that determines the significance level for causal relationships. We use the direct LiNGAM method, which assumes linear relationships and non-Gaussian noise distributions.

\subsection{NO-TEARS Hyperparameters}

\begin{table}[H]
\centering
\setlength{\tabcolsep}{8pt}
\renewcommand{\arraystretch}{1.2}
\begin{tabular}{l!{\vrule}l}
\toprule
\textbf{Parameter} & \textbf{Value/Search Space} \\
\midrule
$\lambda_1$ & $\{0.01, 0.05, 0.1, 0.2\}$ \\
loss\_type & l2 \\
max\_iter & 100 \\
h\_tol & $1 \times 10^{-8}$ \\
rho\_max & $1 \times 10^{16}$ \\
w\_threshold & 0.3 \\
\bottomrule
\end{tabular}
\caption{NO-TEARS hyperparameter configuration.}
\label{tab:notears_hyperparams}
\end{table}

For certain large-scale datasets with high computational complexity, we use a reduced search space for the regularization parameter. Specifically, for datasets DN, JA, and MA (refer to Table \ref{tab:dataset_properties} for full dataset name and properties), we restrict the search space to $\lambda_1 \in \{0.1, 0.2\}$.

\subsection{LiNGAM Hyperparameters}

\begin{table}[H]
\centering
\setlength{\tabcolsep}{8pt}
\renewcommand{\arraystretch}{1.2}
\begin{tabular}{l!{\vrule}l}
\toprule
\textbf{Parameter} & \textbf{Value/Search Space} \\
\midrule
threshold & $\{0.01, 0.05, 0.1, 0.2\}$ \\
method & direct \\
\bottomrule
\end{tabular}
\caption{LiNGAM hyperparameter configuration.}
\label{tab:lingam_hyperparams}
\end{table}

\newpage
\section{Automatic $k$ Selection Algorithm}\label{app:4}

Our automatic $k$ selection method adaptively determines the optimal number of eigenvectors to use for PE generation based on the spectral properties of each dataset's feature graph. The algorithm avoids mid-frequency eigenvectors around eigenvalue 1.0, which typically correspond to less informative spectral components, and focuses on low-frequency (structural) and high-frequency (local) patterns. See Algorithm \ref{alg:auto_k_selection} for details.

\begin{algorithm}[H]
\caption{Automatic $k$ Selection for PEs} \label{alg:auto_k_selection}
\begin{algorithmic}[1]
\State \textbf{Input:} Sorted eigenvalues $\boldsymbol{\lambda} = [\lambda_1, \lambda_2, ..., \lambda_n]$ of normalized Laplacian
\State \textbf{Output:} $k_{\text{first}}$, $k_{\text{last}}$

\State \textbf{Define frequency windows:}
\State Set $\tau_{\text{low}} = 0.75$, $\tau_{\text{high}} = 1.25$

\State \textbf{Count available eigenvectors:}
\State $\text{low\_count} = |\{\lambda_i : \lambda_i \leq \tau_{\text{low}}, i \geq 2\}|$
\State $\text{high\_count} = |\{\lambda_i : \lambda_i \geq \tau_{\text{high}}\}|$

\State \textbf{Apply spectral gap analysis:}
\If{sufficient eigenvalues available}
    \State Identify significant gaps in eigenvalue sequences
    \State Refine $\text{low\_count}$ and $\text{high\_count}$ by cutting before gaps
\EndIf

\State \textbf{Apply constraints:}
\State $k_{\text{first}} = \max(2, \min(\text{low\_count}, 10))$
\State $k_{\text{last}} = k_{\text{first}}$

\State \textbf{Return:} $k_{\text{first}}$, $k_{\text{last}}$
\end{algorithmic}
\end{algorithm}

\subsection{Algorithm Details}

\noindent\textbf{Frequency Window:} We establish a spectral exclusion window around eigenvalue 1.0, specifically avoiding eigenvalues in the range $[0.75, 1.25]$. This strategy is motivated by the observation that eigenvalues near 1.0 often correspond to less informative mid-frequency components that neither capture global structural patterns (low frequencies) nor local connectivity patterns (high frequencies). By focusing on eigenvalues outside this window, we ensure that our PEs capture the most structurally relevant information.

\noindent\textbf{Spectral Gap:} For both low-frequency ($\lambda_i \leq 0.75$) and high-frequency ($\lambda_i \geq 1.25$) eigenvalues, we perform gap analysis to identify natural clusters in the eigenvalue spectrum. The algorithm computes consecutive differences between sorted eigenvalues and identifies significant gaps using a threshold based on the median gap size. When a significant gap is detected, we truncate the selection before the gap to ensure we capture coherent spectral clusters while avoiding potentially noisy or less informative eigenvalues. Gap analysis is only applied when sufficient eigenvalues are available, falling back to simple count-based selection otherwise.

\noindent\textbf{Constraints:} To balance expressiveness with computational efficiency, we apply the following constraints:
\begin{itemize}
    \item \textbf{Minimum constraint ($k \geq 2$):} Ensures that basic structural information is always captured, even for graphs with limited spectral diversity.
    \item \textbf{Maximum constraint ($k \leq 10$):} Prevents excessive dimensionality that could lead to overfitting.
    \item \textbf{Symmetric selection ($k_{\text{last}} = k_{\text{first}}$):} Ensures balanced representation of both global structural patterns and local connectivity patterns.
\end{itemize}

\newpage
\section{Training and Evaluation}\label{app:5}

For each dataset, we run 5 experiments using random seeds \{1, 2, 3, 4, 5\}, respectively.
To handle class imbalance prevalent in real-world tabular datasets, we employ balanced versions of both the training loss function and evaluation metric for classification tasks.

\subsection{Balanced Cross-Entropy Loss}

Standard cross-entropy loss treats all classes equally, which can lead to biased learning toward majority classes in imbalanced datasets. Our balanced cross-entropy loss addresses this by assigning class-specific weights inversely proportional to class frequencies.

\noindent\textbf{Weight Computation:} For a classification task with $C$ classes, let $n_c$ denote the number of training samples in class $c$, and $N$ be the total number of training samples. The class weight for class $c$ is computed as:
\begin{align}
w_c = \frac{1}{C} \cdot \frac{N}{n_c} = \frac{1}{C \cdot p_c},
\end{align}
where $p_c = \frac{n_c}{N}$ is the proportion of samples in class $c$.

\noindent\textbf{Balanced Cross-Entropy Loss:} The balanced cross-entropy loss for a batch of samples is then defined as:
\begin{align}
\mathcal{L}_{\text{balanced}} = -\frac{1}{B} \sum_{i=1}^{B} w_{y_i} \log(\hat{p}_{i,y_i}),
\end{align}
where $B$ is the batch size, $y_i$ is the true class label for sample $i$, $\hat{p}_{i,y_i}$ is the predicted probability for the true class, and $w_{y_i}$ is the weight for class $y_i$.

\subsection{Balanced Accuracy}
Standard accuracy can be misleading for imbalanced datasets, as high accuracy can be achieved by simply predicting the majority class. Balanced accuracy provides a fairer evaluation by giving equal weight to the performance on each class.

\noindent\textbf{Traditional Formulation:} For a test set with $C$ classes, balanced accuracy is defined as the macro-average of per-class recalls:
\begin{align}
\text{Balanced Accuracy} = \frac{1}{C} \sum_{c=1}^{C} \frac{\text{TP}_c}{\text{TP}_c + \text{FN}_c} = \frac{1}{C} \sum_{c=1}^{C} \text{Recall}_c.
\end{align}

\noindent\textbf{Equivalent Weighted Implementation:} In our implementation, we compute this using sample-wise weighting. For each class $c$ in the test set with $n_c^{\text{test}}$ samples, we assign weight:
\begin{align}
w_c^{\text{test}} = \frac{1}{n_c^{\text{test}} \cdot C}.
\end{align}
The balanced accuracy is then computed as:
\begin{align}
\text{Balanced Accuracy} = \sum_{i=1}^{N_{\text{test}}} \mathbf{1}[\hat{y}_i = y_i] \cdot w_{y_i}^{\text{test}}.
\end{align}
This weighted formulation is mathematically equivalent to the traditional definition but more efficient to compute in practice.

\newpage
\section{Baseline Method Hyperparameters}\label{app:6}

\subsection{Tree-Based Baseline Configurations} \label{app:6.1}

\subsubsection{XGBoost Configuration}

For XGBoost, we fix several parameters based on established best practices, while optimizing other hyperparameters (in Table \ref{tab:xgboost_hyperparams}) using Optuna \cite{akiba2019optuna} (with 100 optimization trials). For each dataset, we run 5 experiments using random seeds \{1, 2, 3, 4, 5\}, respectively.

\noindent\textbf{Fixed Parameters:}
\begin{itemize}
    \item \texttt{booster}: `gbtree'
    \item \texttt{n\_estimators}: 2000
    \item \texttt{early\_stopping\_rounds}: 50
\end{itemize}

\noindent\textbf{Hyperparameter Search Space:}

\begin{table}[H]
\centering
\setlength{\tabcolsep}{8pt}
\renewcommand{\arraystretch}{1.2}
\begin{tabular}{l!{\vrule}l}
\toprule
\textbf{Parameter} & \textbf{Search Space} \\
\midrule
max\_depth & $\mathcal{U}\{3, 4, \ldots, 10\}$ \\
min\_child\_weight & $\mathcal{LU}[1\text{e-8}, 1\text{e5}]$ \\
subsample & $\mathcal{U}[0.5, 1.0]$ \\
learning\_rate & $\mathcal{LU}[1\text{e-5}, 1.0]$ \\
colsample\_bylevel & $\mathcal{U}[0.5, 1.0]$ \\
colsample\_bytree & $\mathcal{U}[0.5, 1.0]$ \\
gamma & $\{0, \mathcal{LU}[1\text{e-8}, 1\text{e2}]\}$ \\
reg\_lambda & $\{0, \mathcal{LU}[1\text{e-8}, 1\text{e2}]\}$ \\
reg\_alpha & $\{0, \mathcal{LU}[1\text{e-8}, 1\text{e2}]\}$ \\
\bottomrule
\end{tabular}
\caption{XGBoost hyperparameter search space. $\mathcal{U}\{a, b, \ldots, c\}$ denotes uniform integer sampling, $\mathcal{U}[a,b]$ denotes uniform continuous sampling, $\mathcal{LU}[a,b]$ denotes log-uniform sampling, and $\{0, \text{dist}\}$ denotes categorical choice between 0 and sampling from the specified distribution.}
\label{tab:xgboost_hyperparams}
\end{table}

\subsubsection{CatBoost Configuration}

For CatBoost, similarly, we fix several parameters based on established best practices, while optimizing other hyperparameters (in Table \ref{tab:catboost_hyperparams}) using Optuna \cite{akiba2019optuna} (with 100 optimization trials). For each dataset, we run 5 experiments using random seeds \{1, 2, 3, 4, 5\}, respectively.

\noindent\textbf{Fixed Parameters:}
\begin{itemize}
    \item \texttt{iterations}: 2000
    \item \texttt{early\_stopping\_rounds}: 50
    \item \texttt{od\_pval}: 0.001 (overfitting detection p-value)
\end{itemize}

\noindent\textbf{Hyperparameter Search Space:}

\begin{table}[H]
\centering
\setlength{\tabcolsep}{8pt}
\renewcommand{\arraystretch}{1.2}
\begin{tabular}{l!{\vrule}l}
\toprule
\textbf{Parameter} & \textbf{Search Space} \\
\midrule
max\_depth & $\mathcal{U}\{3, 4, \ldots, 10\}$ \\
learning\_rate & $\mathcal{LU}[1\text{e-5}, 1.0]$ \\
bagging\_temperature & $\mathcal{U}[0.0, 1.0]$ \\
l2\_leaf\_reg & $\mathcal{LU}[1.0, 10.0]$ \\
leaf\_estimation\_iterations & $\mathcal{U}\{1, 2, \ldots, 10\}$ \\
\bottomrule
\end{tabular}
\caption{CatBoost hyperparameter search space. $\mathcal{U}\{a, b, \ldots, c\}$ denotes uniform integer sampling, $\mathcal{U}[a,b]$ denotes uniform continuous sampling, and $\mathcal{LU}[a,b]$ denotes log-uniform sampling.}
\label{tab:catboost_hyperparams}
\end{table}

\newpage
\subsection{Transformer-Based Baseline Configurations} \label{app:6.2}

To ensure fair comparison with Tab-PET, we implement three transformer-based baselines using their original architectures and recommended/default hyperparameter settings. All methods employ identical training protocols with 5 experiments using random seeds \{1, 2, 3, 4, 5\}, respectively.

\subsubsection{FT-Transformer Configuration}

The configuration of the FT-Transformer is shown in Table \ref{tab:ft_transformer_config}.

\begin{table}[H]
\centering
\setlength{\tabcolsep}{8pt}
\renewcommand{\arraystretch}{1.2}
\begin{tabular}{l!{\vrule}l}
\toprule
\textbf{Parameter} & \textbf{Value} \\
\midrule
\multicolumn{2}{c}{\textit{Architecture Parameters}} \\
\midrule
n\_layers & 3 \\
n\_heads & 8 \\
d\_ffn\_factor & 4/3 \\
target\_total\_dim & 192 \\
attention\_dropout & 0.2 \\
ffn\_dropout & 0.1 \\
residual\_dropout & 0.0 \\
activation & reglu \\
prenormalization & True \\
initialization & kaiming \\
\midrule
\multicolumn{2}{c}{\textit{Training Parameters}} \\
\midrule
optimizer & AdamW \\
learning\_rate & 1e-4 \\
weight\_decay & 1e-5 \\
batch\_size & 32$^*$ \\
num\_epochs & 500 \\
early\_stopping & True \\
patience & 50 \\
min\_delta & 1e-6 \\
min\_epochs & 50 \\
\bottomrule
\end{tabular}
\caption{FT-Transformer hyperparameter configuration. $^*$Special batch sizes are used for specific datasets, detailed in Table \ref{tab:special_batch_sizes}.}
\label{tab:ft_transformer_config}
\end{table}

\newpage
\subsubsection{SAINT Configuration}

The configuration of the SAINT is shown in Table \ref{tab:saint_config}.

\begin{table}[H]
\centering
\setlength{\tabcolsep}{8pt}
\renewcommand{\arraystretch}{1.2}
\begin{tabular}{l!{\vrule}l}
\toprule
\textbf{Parameter} & \textbf{Value} \\
\midrule
\multicolumn{2}{c}{\textit{Architecture Parameters}} \\
\midrule
target\_embedding\_dim & 32 \\
transformer\_depth & 1 \\
attention\_heads & 4 \\
attention\_dropout & 0.8 \\
ffn\_dropout & 0.8 \\
cont\_embeddings & MLP \\
attentiontype & colrow \\
final\_mlp\_style & sep \\
\midrule
\multicolumn{2}{c}{\textit{Training Parameters}} \\
\midrule
optimizer & AdamW \\
learning\_rate & 1e-4 \\
weight\_decay & 1e-5 \\
batch\_size & 256$^{**}$ \\
num\_epochs & 500 \\
early\_stopping & True \\
patience & 50 \\
min\_delta & 1e-6 \\
min\_epochs & 50 \\
\bottomrule
\end{tabular}
\caption{SAINT hyperparameter configuration. $^{**}$Batch size is adjusted to min(64, 256) when feature count $> 100$, following official SAINT settings.}
\label{tab:saint_config}
\end{table}

\newpage
\subsubsection{TabTransformer Configuration}

The configuration of the TabTransformer is shown in Table \ref{tab:tabtransformer_config}.

\begin{table}[H]
\centering
\setlength{\tabcolsep}{8pt}
\renewcommand{\arraystretch}{1.2}
\begin{tabular}{l!{\vrule}l}
\toprule
\textbf{Parameter} & \textbf{Value} \\
\midrule
\multicolumn{2}{c}{\textit{Architecture Parameters}} \\
\midrule
dim & 32 \\
depth & 6 \\
heads & 8 \\
dim\_head & 16 \\
attn\_dropout & 0.1 \\
ff\_dropout & 0.1 \\
mlp\_hidden\_mults & (4, 2) \\
mlp\_act & ReLU \\
num\_special\_tokens & 2 \\
use\_shared\_categ\_embed & True \\
shared\_categ\_dim\_divisor & 8.0 \\
num\_residual\_streams & 4 \\
\midrule
\multicolumn{2}{c}{\textit{Training Parameters}} \\
\midrule
optimizer & AdamW \\
learning\_rate & 1e-4 \\
weight\_decay & 1e-5 \\
batch\_size & 32$^*$ \\
num\_epochs & 500 \\
early\_stopping & True \\
patience & 50 \\
min\_delta & 1e-6 \\
min\_epochs & 50 \\
\bottomrule
\end{tabular}
\caption{TabTransformer hyperparameter configuration. $^*$Special batch sizes are used for specific datasets as detailed in Table \ref{tab:special_batch_sizes}.}
\label{tab:tabtransformer_config}
\end{table}

\subsubsection{Special Batch Size Settings}

For FT-Transformer and TabTransformer, we apply dataset-specific batch size adjustments for large datasets, shown in Table \ref{tab:special_batch_sizes}.

\begin{table}[H]
\centering
\setlength{\tabcolsep}{12pt}
\renewcommand{\arraystretch}{1.2}
\begin{tabular}{l!{\vrule}c}
\toprule
\textbf{Dataset} & \textbf{Batch Size} \\
\midrule
GesturePhaseSegmentationProcessed & 64 \\
eye\_movements & 64 \\
heloc & 64 \\
california & 128 \\
cpu\_activity & 64 \\
cpu\_small & 64 \\
diamonds & 256 \\
grid\_stability & 64 \\
kin8nm & 64 \\
\hline
All other datasets & 32 (default) \\
\bottomrule
\end{tabular}
\caption{Special batch size settings for FT-Transformer and TabTransformer.}
\label{tab:special_batch_sizes}
\end{table}

For SAINT, the batch size is dynamically adjusted based on feature count: when the number of features exceeds 100, the batch size is reduced to $min(64, default\_batch\_size)$, following the official SAINT implementation guidelines (https://github.com/somepago/saint/blob/main/train.py).

\subsection{Fairness of Comparison} \label{app:6.3}

To ensure rigorous and unbiased evaluation, we establish a fair comparison framework between baseline transformer models and their corresponding Tab-PET variants. This section details our experimental design choices that guarantee the validity of our performance improvements.

\subsubsection{Identical Model Configurations}

For all three transformer architectures (TabTransformer, SAINT, and FT-Transformer), we maintain \textbf{identical configurations} between baseline models and their Tab-PET variants, as detailed in technical appendix \ref{app:6.2}, which includes architecture parameters, training parameters, and data preprocessing (e.g., normalization, categorical encoding, train/validation/test splits, multiple seeds, etc.)

The \textbf{only difference} between baseline and Tab-PET variants lies in the PE component, which we carefully control to ensure fair comparison.

\subsubsection{Token Dimension Consistency}

To maintain architectural equivalence, we enforce consistent total token dimensions across all comparisons. We use FT-Transformer as an example, and SAINT and TabTransformer follow exact same setups:

\textbf{Baseline Models:}
\begin{align}
d_{\text{token}} &= 192 - d_{\text{PE}}, \\
d_{\text{total}} &= d_{\text{token}} + d(\mathbf{0} \cdot \mathbf{PE}_{\text{fixed}}) = 192,
\end{align}
where $\mathbf{PE}_{\text{fixed}} \in \mathbb{R}^{d_{\text{PE}}}$ is our fixed graph-derived PE matrix, and $\mathbf{0} \cdot \mathbf{PE}_{\text{fixed}}$ represents zero-padding of dimension $d_{\mathbf{PE}_{\text{fixed}}}$.

\textbf{Tab-PET Models:}
\begin{align}
d_{\text{token}} &= 192 - d_{\text{PE}}, \\
d_{\text{total}} &= d_{\text{token}} + d(\alpha \cdot \mathbf{PE}_{\text{fixed}}) = 192,
\end{align}
where $\alpha$ is the scaling hyperparameter.

\subsubsection{Parameter Count Equivalence}

This design ensures that both baseline and Tab-PET models have exactly the same number of trainable parameters (it should be noted that our PEs introduced to the transformer are not learnable), and the forward and backward pass times are equivalent. Therefore, any performance differences can be attributed solely to the structural inductive bias introduced by PEs. To further confirm the performance improvement, we also use Wilcoxon signed-rank tests (p $<$ 0.05) for validation.

\subsubsection{(additional) Fair Comparison with Learnable PEs}

To establish a rigorous comparison between Tab-PET and learnable PEs, we maintain identical experimental conditions with the \textbf{only difference} being the source of positional information.

\textbf{Learnable PE Models:}
\begin{align}
d_{\text{token}} &= 192 - d_{\text{PE}}, \\
d_{\text{total}} &= d_{\text{token}} + d(\text{learnable PE}) = 192,
\end{align}
where learnable PE parameters are randomly initialized and optimized during training through backpropagation.

This setup ensures that both methods maintain identical total token dimensions, but differs significantly in parameter efficiency. Learnable PE models contain additional trainable parameters across all PE dimensions, whereas Tab-PET introduces no additional learnable parameters since the PE components are fixed and derived from the graph structure. Therefore, learnable PE approaches have substantially more trainable parameters than Tab-PET models, making Tab-PET's superior performance even more remarkable from a parameter efficiency perspective. The key distinction is that learnable PEs must discover structural relationships from scratch during training using additional parameters, while Tab-PET leverages pre-computed graph-derived structural priors without increasing model complexity.

\newpage
\section{Main Results} \label{app:7}

\subsection{Comparing 5 Graph Estimation Approaches on FT-Transformer} \label{app:7.1}

\begin{table}[H]
\centering
\tiny
\setlength{\tabcolsep}{5pt}
\renewcommand{\arraystretch}{1.2}
\definecolor{lightgray}{gray}{0.7}

\resizebox{0.95\textwidth}{!}{
\begin{tabular}{l!{\vrule}ccccc!{\vrule}ccccc}
\toprule
Method & CA & AS & DA & NL & TR & AU & GE & SA & BL & CHU \\
\midrule
NT & \multirow{2}{*}{$\checkmark$} & \multirow{2}{*}{} & \multirow{2}{*}{$\checkmark$} & \multirow{2}{*}{} & \multirow{2}{*}{} & 0.8242 ± 0.0203 & 0.6412 ± 0.0081 & \textbf{0.8060 ± 0.0383} & 0.7481 ± 0.0078 & 0.9131 ± 0.0016 \\
+PET &  &  &  &  &  & 0.8333 ± 0.0111 (+1.11\%) & \textbf{0.6574 ± 0.0081} (+2.53\%) & 0.7930 ± 0.0490 (-1.60\%) & \textbf{0.7633 ± 0.0066} (+2.03\%) & 0.9136 ± 0.0034 (+0.05\%) \\
\arrayrulecolor{lightgray}\hline\arrayrulecolor{black}
LG & \multirow{2}{*}{$\checkmark$} & \multirow{2}{*}{} & \multirow{2}{*}{$\checkmark$} & \multirow{2}{*}{} & \multirow{2}{*}{} & 0.8257 ± 0.0168 & \underline{0.6534 ± 0.0134} & 0.7724 ± 0.0247 & 0.7481 ± 0.0078 & 0.9131 ± 0.0016 \\
+PET &  &  &  &  &  & 0.8316 ± 0.0106 (+0.72\%) & 0.6492 ± 0.0110 (-0.64\%) & 0.7984 ± 0.0203 (+3.36\%) & 0.7522 ± 0.0099 (+0.55\%) & 0.9131 ± 0.0042 (-0.00\%) \\
\arrayrulecolor{lightgray}\hline\arrayrulecolor{black}
PS & \multirow{2}{*}{} & \multirow{2}{*}{$\checkmark$} & \multirow{2}{*}{} & \multirow{2}{*}{} & \multirow{2}{*}{} & 0.8283 ± 0.0093 & 0.6412 ± 0.0081 & 0.7791 ± 0.0266 & 0.7481 ± 0.0078 & 0.9076 ± 0.0043 \\
+PET &  &  &  &  &  & \textbf{0.8389 ± 0.0166} (+1.29\%) & 0.6471 ± 0.0080 (+0.93\%) & 0.7600 ± 0.0249 (-2.45\%) & 0.7577 ± 0.0141 (+1.29\%) & \underline{0.9152 ± 0.0051} (+0.83\%) \\
\arrayrulecolor{lightgray}\hline\arrayrulecolor{black}
SM & \multirow{2}{*}{} & \multirow{2}{*}{$\checkmark$} & \multirow{2}{*}{} & \multirow{2}{*}{$\checkmark$} & \multirow{2}{*}{} & 0.8283 ± 0.0093 & 0.6412 ± 0.0081 & 0.7791 ± 0.0266 & 0.7481 ± 0.0078 & 0.9076 ± 0.0043 \\
+PET &  &  &  &  &  & \underline{0.8379 ± 0.0162} (+1.16\%) & 0.6502 ± 0.0080 (+1.41\%) & 0.7924 ± 0.0252 (+1.71\%) & 0.7582 ± 0.0055 (+1.35\%) & \textbf{0.9172 ± 0.0046} (+1.05\%) \\
\arrayrulecolor{lightgray}\hline\arrayrulecolor{black}
CL & \multirow{2}{*}{} & \multirow{2}{*}{$\checkmark$} & \multirow{2}{*}{$\checkmark$} & \multirow{2}{*}{$\checkmark$} & \multirow{2}{*}{$\checkmark$} & 0.8248 ± 0.0124 & \underline{0.6534 ± 0.0134} & \textbf{0.8060 ± 0.0383} & 0.7481 ± 0.0078 & 0.9078 ± 0.0040 \\
+PET &  &  &  &  &  & 0.8354 ± 0.0074 (+1.29\%) & 0.6520 ± 0.0076 (-0.21\%) & 0.7914 ± 0.0551 (-1.80\%) & \underline{0.7598 ± 0.0075} (+1.56\%) & \underline{0.9152 ± 0.0071} (+0.82\%) \\
\midrule
\end{tabular}
}
\\
\vspace{0.5em}

\resizebox{0.95\textwidth}{!}{
\begin{tabular}{l!{\vrule}ccccc!{\vrule}ccccc}
\midrule
Method & CA & AS & DA & NL & TR & CM & CR & DI & DN & EY \\
\midrule
NT & \multirow{2}{*}{$\checkmark$} & \multirow{2}{*}{} & \multirow{2}{*}{$\checkmark$} & \multirow{2}{*}{} & \multirow{2}{*}{} & 0.5661 ± 0.0158 & 0.6395 ± 0.0371 & 0.7691 ± 0.0148 & 0.9550 ± 0.0088 & \textbf{0.7517 ± 0.0071} \\
+PET &  &  &  &  &  & 0.5656 ± 0.0077 (-0.08\%) & \underline{0.6919 ± 0.0047} (+8.19\%) & 0.7613 ± 0.0048 (-1.01\%) & 0.9594 ± 0.0039 (+0.46\%) & 0.7474 ± 0.0030 (-0.57\%) \\
\arrayrulecolor{lightgray}\hline\arrayrulecolor{black}
LG & \multirow{2}{*}{$\checkmark$} & \multirow{2}{*}{} & \multirow{2}{*}{$\checkmark$} & \multirow{2}{*}{} & \multirow{2}{*}{} & 0.5661 ± 0.0158 & 0.6395 ± 0.0371 & 0.7691 ± 0.0148 & 0.9550 ± 0.0088 & \textbf{0.7517 ± 0.0071} \\
+PET &  &  &  &  &  & \underline{0.5715 ± 0.0150} (+0.95\%) & \textbf{0.6971 ± 0.0125} (+9.01\%) & 0.7590 ± 0.0092 (-1.31\%) & \textbf{0.9653 ± 0.0025} (+1.07\%) & 0.7497 ± 0.0087 (-0.26\%) \\
\arrayrulecolor{lightgray}\hline\arrayrulecolor{black}
PS & \multirow{2}{*}{} & \multirow{2}{*}{$\checkmark$} & \multirow{2}{*}{} & \multirow{2}{*}{} & \multirow{2}{*}{} & 0.5449 ± 0.0222 & 0.6443 ± 0.0244 & 0.7691 ± 0.0148 & 0.9528 ± 0.0071 & 0.7469 ± 0.0050 \\
+PET &  &  &  &  &  & \textbf{0.5718 ± 0.0054} (+4.93\%) & \underline{0.6919 ± 0.0295} (+7.39\%) & \underline{0.7728 ± 0.0150} (+0.49\%) & 0.9615 ± 0.0045 (+0.91\%) & 0.7470 ± 0.0071 (+0.01\%) \\
\arrayrulecolor{lightgray}\hline\arrayrulecolor{black}
SM & \multirow{2}{*}{} & \multirow{2}{*}{$\checkmark$} & \multirow{2}{*}{} & \multirow{2}{*}{$\checkmark$} & \multirow{2}{*}{} & 0.5449 ± 0.0222 & 0.6443 ± 0.0244 & 0.7691 ± 0.0148 & 0.9528 ± 0.0071 & 0.7480 ± 0.0078 \\
+PET &  &  &  &  &  & 0.5705 ± 0.0060 (+4.70\%) & 0.6821 ± 0.0201 (+5.88\%) & 0.7712 ± 0.0179 (+0.28\%) & 0.9617 ± 0.0042 (+0.93\%) & 0.7499 ± 0.0060 (+0.24\%) \\
\arrayrulecolor{lightgray}\hline\arrayrulecolor{black}
CL & \multirow{2}{*}{} & \multirow{2}{*}{$\checkmark$} & \multirow{2}{*}{$\checkmark$} & \multirow{2}{*}{$\checkmark$} & \multirow{2}{*}{$\checkmark$} & 0.5661 ± 0.0158 & 0.6395 ± 0.0371 & 0.7691 ± 0.0148 & 0.9550 ± 0.0088 & \textbf{0.7517 ± 0.0071} \\
+PET &  &  &  &  &  & 0.5699 ± 0.0085 (+0.67\%) & 0.6879 ± 0.0133 (+7.56\%) & \textbf{0.7731 ± 0.0203} (+0.52\%) & \underline{0.9624 ± 0.0033} (+0.77\%) & 0.7458 ± 0.0100 (-0.78\%) \\
\midrule
\end{tabular}
}
\\
\vspace{0.5em}

\resizebox{0.95\textwidth}{!}{
\begin{tabular}{l!{\vrule}ccccc!{\vrule}ccccc}
\midrule
Method & CA & AS & DA & NL & TR & FI & HE & JA & KC & KR \\
\midrule
NT & \multirow{2}{*}{$\checkmark$} & \multirow{2}{*}{} & \multirow{2}{*}{$\checkmark$} & \multirow{2}{*}{} & \multirow{2}{*}{} & 0.4668 ± 0.0097 & 0.7202 ± 0.0044 & 0.8049 ± 0.0059 & 0.7195 ± 0.0156 & 0.9951 ± 0.0015 \\
+PET &  &  &  &  &  & \textbf{0.4740 ± 0.0031} (+1.52\%) & \underline{0.7312 ± 0.0006} (+1.53\%) & 0.8052 ± 0.0086 (+0.04\%) & \textbf{0.7366 ± 0.0104} (+2.37\%) & \underline{0.9963 ± 0.0007} (+0.12\%) \\
\arrayrulecolor{lightgray}\hline\arrayrulecolor{black}
LG & \multirow{2}{*}{$\checkmark$} & \multirow{2}{*}{} & \multirow{2}{*}{$\checkmark$} & \multirow{2}{*}{} & \multirow{2}{*}{} & 0.4719 ± 0.0024 & 0.7201 ± 0.0057 & 0.8049 ± 0.0059 & 0.7205 ± 0.0059 & 0.9951 ± 0.0015 \\
+PET &  &  &  &  &  & \underline{0.4727 ± 0.0068} (+0.16\%) & 0.7266 ± 0.0028 (+0.91\%) & 0.8059 ± 0.0045 (+0.12\%) & 0.7220 ± 0.0051 (+0.21\%) & 0.9960 ± 0.0007 (+0.09\%) \\
\arrayrulecolor{lightgray}\hline\arrayrulecolor{black}
PS & \multirow{2}{*}{} & \multirow{2}{*}{$\checkmark$} & \multirow{2}{*}{} & \multirow{2}{*}{} & \multirow{2}{*}{} & 0.4599 ± 0.0075 & 0.7235 ± 0.0059 & 0.8049 ± 0.0059 & 0.7138 ± 0.0095 & 0.9960 ± 0.0007 \\
+PET &  &  &  &  &  & 0.4694 ± 0.0071 (+2.06\%) & \textbf{0.7315 ± 0.0024} (+1.12\%) & \textbf{0.8076 ± 0.0028} (+0.33\%) & 0.7230 ± 0.0093 (+1.29\%) & \underline{0.9963 ± 0.0012} (+0.03\%) \\
\arrayrulecolor{lightgray}\hline\arrayrulecolor{black}
SM & \multirow{2}{*}{} & \multirow{2}{*}{$\checkmark$} & \multirow{2}{*}{} & \multirow{2}{*}{$\checkmark$} & \multirow{2}{*}{} & 0.4599 ± 0.0075 & 0.7235 ± 0.0059 & 0.8049 ± 0.0059 & 0.7138 ± 0.0095 & 0.9960 ± 0.0007 \\
+PET &  &  &  &  &  & 0.4700 ± 0.0121 (+2.20\%) & 0.7299 ± 0.0018 (+0.89\%) & \underline{0.8062 ± 0.0095} (+0.16\%) & \underline{0.7308 ± 0.0062} (+2.39\%) & 0.9954 ± 0.0000 (-0.06\%) \\
\arrayrulecolor{lightgray}\hline\arrayrulecolor{black}
CL & \multirow{2}{*}{} & \multirow{2}{*}{$\checkmark$} & \multirow{2}{*}{$\checkmark$} & \multirow{2}{*}{$\checkmark$} & \multirow{2}{*}{$\checkmark$} & 0.4668 ± 0.0097 & 0.7202 ± 0.0044 & 0.8049 ± 0.0059 & 0.7195 ± 0.0156 & 0.9951 ± 0.0015 \\
+PET &  &  &  &  &  & 0.4704 ± 0.0078 (+0.76\%) & 0.7285 ± 0.0013 (+1.14\%) & 0.8046 ± 0.0056 (-0.04\%) & 0.7294 ± 0.0095 (+1.38\%) & \textbf{0.9966 ± 0.0006} (+0.15\%) \\
\midrule
\end{tabular}
}
\\
\vspace{0.5em}

\resizebox{0.95\textwidth}{!}{
\begin{tabular}{l!{\vrule}ccccc!{\vrule}ccccc}
\midrule
Method & CA & AS & DA & NL & TR & MA & PH & QSA & ST & SY \\
\midrule
NT & \multirow{2}{*}{$\checkmark$} & \multirow{2}{*}{} & \multirow{2}{*}{$\checkmark$} & \multirow{2}{*}{} & \multirow{2}{*}{} & 0.7083 ± 0.0265 & 0.8607 ± 0.0071 & 0.8329 ± 0.0199 & 0.7911 ± 0.0188 & 0.9477 ± 0.0017 \\
+PET &  &  &  &  &  & 0.7329 ± 0.0539 (+3.46\%) & \underline{0.8629 ± 0.0042} (+0.26\%) & \textbf{0.8461 ± 0.0135} (+1.58\%) & \textbf{0.8177 ± 0.0067} (+3.35\%) & 0.9493 ± 0.0035 (+0.17\%) \\
\arrayrulecolor{lightgray}\hline\arrayrulecolor{black}
LG & \multirow{2}{*}{$\checkmark$} & \multirow{2}{*}{} & \multirow{2}{*}{$\checkmark$} & \multirow{2}{*}{} & \multirow{2}{*}{} & 0.6952 ± 0.0398 & 0.8607 ± 0.0071 & 0.8203 ± 0.0159 & 0.7949 ± 0.0104 & 0.9477 ± 0.0017 \\
+PET &  &  &  &  &  & 0.7326 ± 0.0217 (+5.37\%) & \underline{0.8629 ± 0.0080} (+0.26\%) & 0.8417 ± 0.0193 (+2.61\%) & \underline{0.8148 ± 0.0073} (+2.50\%) & 0.9485 ± 0.0048 (+0.08\%) \\
\arrayrulecolor{lightgray}\hline\arrayrulecolor{black}
PS & \multirow{2}{*}{} & \multirow{2}{*}{$\checkmark$} & \multirow{2}{*}{} & \multirow{2}{*}{} & \multirow{2}{*}{} & 0.7184 ± 0.0271 & 0.8607 ± 0.0071 & 0.8313 ± 0.0128 & 0.7949 ± 0.0104 & 0.9469 ± 0.0042 \\
+PET &  &  &  &  &  & \textbf{0.7555 ± 0.0469} (+5.16\%) & 0.8624 ± 0.0057 (+0.20\%) & 0.8391 ± 0.0115 (+0.94\%) & 0.8137 ± 0.0122 (+2.37\%) & \underline{0.9516 ± 0.0063} (+0.49\%) \\
\arrayrulecolor{lightgray}\hline\arrayrulecolor{black}
SM & \multirow{2}{*}{} & \multirow{2}{*}{$\checkmark$} & \multirow{2}{*}{} & \multirow{2}{*}{$\checkmark$} & \multirow{2}{*}{} & 0.7184 ± 0.0271 & 0.8607 ± 0.0071 & 0.8313 ± 0.0128 & 0.7949 ± 0.0104 & 0.9469 ± 0.0042 \\
+PET &  &  &  &  &  & \underline{0.7398 ± 0.0172} (+2.98\%) & \textbf{0.8663 ± 0.0036} (+0.65\%) & \underline{0.8441 ± 0.0145} (+1.54\%) & 0.8112 ± 0.0104 (+2.05\%) & 0.9479 ± 0.0044 (+0.10\%) \\
\arrayrulecolor{lightgray}\hline\arrayrulecolor{black}
CL & \multirow{2}{*}{} & \multirow{2}{*}{$\checkmark$} & \multirow{2}{*}{$\checkmark$} & \multirow{2}{*}{$\checkmark$} & \multirow{2}{*}{$\checkmark$} & 0.7011 ± 0.0251 & 0.8607 ± 0.0071 & 0.8329 ± 0.0199 & 0.8009 ± 0.0123 & \textbf{0.9526 ± 0.0025} \\
+PET &  &  &  &  &  & 0.7340 ± 0.0183 (+4.69\%) & 0.8612 ± 0.0091 (+0.06\%) & 0.8432 ± 0.0114 (+1.23\%) & 0.8086 ± 0.0117 (+0.96\%) & 0.9508 ± 0.0048 (-0.18\%) \\
\midrule
\end{tabular}
}
\\
\vspace{0.5em}

\resizebox{0.95\textwidth}{!}{
\begin{tabular}{l!{\vrule}ccccc!{\vrule}ccccc}
\midrule
Method & CA & AS & DA & NL & TR & TI & VE & WI & WIN & YE \\
\midrule
NT & \multirow{2}{*}{$\checkmark$} & \multirow{2}{*}{} & \multirow{2}{*}{$\checkmark$} & \multirow{2}{*}{} & \multirow{2}{*}{} & 0.9719 ± 0.0061 & 0.7979 ± 0.0134 & 0.9570 ± 0.0079 & 0.3629 ± 0.0121 & 0.5449 ± 0.0062 \\
+PET &  &  &  &  &  & 0.9797 ± 0.0017 (+0.80\%) & 0.7967 ± 0.0083 (-0.15\%) & 0.9668 ± 0.0078 (+1.02\%) & \underline{0.3775 ± 0.0241} (+4.01\%) & 0.5599 ± 0.0190 (+2.76\%) \\
\arrayrulecolor{lightgray}\hline\arrayrulecolor{black}
LG & \multirow{2}{*}{$\checkmark$} & \multirow{2}{*}{} & \multirow{2}{*}{$\checkmark$} & \multirow{2}{*}{} & \multirow{2}{*}{} & 0.9719 ± 0.0061 & 0.7856 ± 0.0180 & 0.9570 ± 0.0079 & 0.3629 ± 0.0121 & 0.5449 ± 0.0062 \\
+PET &  &  &  &  &  & 0.9787 ± 0.0051 (+0.70\%) & 0.7969 ± 0.0122 (+1.43\%) & \textbf{0.9706 ± 0.0063} (+1.42\%) & 0.3718 ± 0.0151 (+2.43\%) & \underline{0.5634 ± 0.0223} (+3.40\%) \\
\arrayrulecolor{lightgray}\hline\arrayrulecolor{black}
PS & \multirow{2}{*}{} & \multirow{2}{*}{$\checkmark$} & \multirow{2}{*}{} & \multirow{2}{*}{} & \multirow{2}{*}{} & 0.9719 ± 0.0061 & 0.7856 ± 0.0180 & 0.9570 ± 0.0079 & 0.3629 ± 0.0121 & 0.5449 ± 0.0062 \\
+PET &  &  &  &  &  & \underline{0.9826 ± 0.0045} (+1.10\%) & \underline{0.8069 ± 0.0160} (+2.71\%) & 0.9633 ± 0.0117 (+0.65\%) & \textbf{0.3793 ± 0.0112} (+4.50\%) & 0.5544 ± 0.0138 (+1.76\%) \\
\arrayrulecolor{lightgray}\hline\arrayrulecolor{black}
SM & \multirow{2}{*}{} & \multirow{2}{*}{$\checkmark$} & \multirow{2}{*}{} & \multirow{2}{*}{$\checkmark$} & \multirow{2}{*}{} & 0.9719 ± 0.0061 & 0.7856 ± 0.0180 & 0.9570 ± 0.0079 & 0.3629 ± 0.0121 & 0.5449 ± 0.0062 \\
+PET &  &  &  &  &  & 0.9772 ± 0.0047 (+0.54\%) & \textbf{0.8080 ± 0.0134} (+2.84\%) & \underline{0.9704 ± 0.0080} (+1.39\%) & 0.3755 ± 0.0158 (+3.47\%) & 0.5618 ± 0.0214 (+3.10\%) \\
\arrayrulecolor{lightgray}\hline\arrayrulecolor{black}
CL & \multirow{2}{*}{} & \multirow{2}{*}{$\checkmark$} & \multirow{2}{*}{$\checkmark$} & \multirow{2}{*}{$\checkmark$} & \multirow{2}{*}{$\checkmark$} & 0.9719 ± 0.0061 & 0.7916 ± 0.0194 & 0.9570 ± 0.0079 & 0.3640 ± 0.0044 & 0.5449 ± 0.0062 \\
+PET &  &  &  &  &  & \textbf{0.9853 ± 0.0057} (+1.38\%) & 0.8024 ± 0.0209 (+1.36\%) & 0.9684 ± 0.0100 (+1.18\%) & 0.3693 ± 0.0064 (+1.44\%) & \textbf{0.5635 ± 0.0126} (+3.41\%) \\
\bottomrule
\end{tabular}
}
\caption{Detailed performance on different graph estimation approaches for FT-Transformer on classification datasets. Each result is averaged over 5 random seeds. Use balanced accuracy (higher is better). CA = Causal, AS = Association, DA = Directed Acyclic, NL = Nonlinear, TR = Tree. Results show baseline performance and Tab-PET performance with improvement percentages in parentheses compared to the corresponding baselines. \textbf{Bold} indicates the best result, and \underline{underline} indicates the second-best result.}
\label{tab:ablation_classification_results}
\end{table}

\begin{table}[H]
\centering
\tiny
\setlength{\tabcolsep}{5pt}
\renewcommand{\arraystretch}{1.2}
\definecolor{lightgray}{gray}{0.7}

\resizebox{0.95\textwidth}{!}{
\begin{tabular}{l!{\vrule}ccccc!{\vrule}ccccc}
\toprule
Method & CA & AS & DA & NL & TR & QS & AB & AI & BO & BOS \\
\midrule
NT & \multirow{2}{*}{$\checkmark$} & \multirow{2}{*}{} & \multirow{2}{*}{$\checkmark$} & \multirow{2}{*}{} & \multirow{2}{*}{} & 0.9706 ± 0.0163 & 2.1627 ± 0.0153 & 1.4659 ± 0.0911 & 0.3401 ± 0.0368 & 2.9440 ± 0.3526 \\
+PET &  &  &  &  &  & 0.9689 ± 0.0193 (+0.17\%) & \underline{2.1461 ± 0.0175} (+0.77\%) & 1.3139 ± 0.0351 (+10.37\%) & 0.3199 ± 0.0243 (+5.92\%) & 3.0288 ± 0.2175 (-2.88\%) \\
\arrayrulecolor{lightgray}\hline\arrayrulecolor{black}
LG & \multirow{2}{*}{$\checkmark$} & \multirow{2}{*}{} & \multirow{2}{*}{$\checkmark$} & \multirow{2}{*}{} & \multirow{2}{*}{} & 0.9706 ± 0.0163 & 2.2002 ± 0.0326 & 1.4659 ± 0.0911 & 0.2968 ± 0.0345 & 3.1747 ± 0.3973 \\
+PET &  &  &  &  &  & \underline{0.9655 ± 0.0145} (+0.52\%) & \textbf{2.1367 ± 0.0122} (+2.89\%) & \textbf{1.2981 ± 0.0422} (+11.45\%) & 0.3004 ± 0.0305 (-1.22\%) & \underline{2.8108 ± 0.1238} (+11.46\%) \\
\arrayrulecolor{lightgray}\hline\arrayrulecolor{black}
PS & \multirow{2}{*}{} & \multirow{2}{*}{$\checkmark$} & \multirow{2}{*}{} & \multirow{2}{*}{} & \multirow{2}{*}{} & 0.9706 ± 0.0163 & 2.2002 ± 0.0326 & 1.4659 ± 0.0911 & 0.3268 ± 0.0891 & 3.2319 ± 0.1545 \\
+PET &  &  &  &  &  & 0.9684 ± 0.0142 (+0.22\%) & 2.1699 ± 0.0121 (+1.38\%) & 1.3446 ± 0.0510 (+8.28\%) & \underline{0.2913 ± 0.0433} (+10.84\%) & \textbf{2.7359 ± 0.1319} (+15.35\%) \\
\arrayrulecolor{lightgray}\hline\arrayrulecolor{black}
SM & \multirow{2}{*}{} & \multirow{2}{*}{$\checkmark$} & \multirow{2}{*}{} & \multirow{2}{*}{$\checkmark$} & \multirow{2}{*}{} & 0.9706 ± 0.0163 & 2.2002 ± 0.0326 & 1.4659 ± 0.0911 & 0.3268 ± 0.0891 & 3.2319 ± 0.1545 \\
+PET &  &  &  &  &  & \textbf{0.9606 ± 0.0199} (+1.03\%) & 2.1681 ± 0.0124 (+1.46\%) & \underline{1.3096 ± 0.0510} (+10.67\%) & \textbf{0.2773 ± 0.0480} (+15.13\%) & 2.9934 ± 0.2309 (+7.38\%) \\
\arrayrulecolor{lightgray}\hline\arrayrulecolor{black}
CL & \multirow{2}{*}{} & \multirow{2}{*}{$\checkmark$} & \multirow{2}{*}{$\checkmark$} & \multirow{2}{*}{$\checkmark$} & \multirow{2}{*}{$\checkmark$} & 0.9706 ± 0.0163 & 2.2002 ± 0.0326 & 1.4659 ± 0.0911 & 0.3390 ± 0.0653 & 3.1747 ± 0.3973 \\
+PET &  &  &  &  &  & 0.9721 ± 0.0153 (-0.15\%) & 2.1632 ± 0.0095 (+1.68\%) & 1.3202 ± 0.0345 (+9.94\%) & 0.3371 ± 0.0830 (+0.56\%) & 2.9734 ± 0.5154 (+6.34\%) \\
\midrule
\end{tabular}
}
\\
\vspace{0.5em}

\resizebox{0.95\textwidth}{!}{
\begin{tabular}{l!{\vrule}ccccc!{\vrule}ccccc}
\midrule
Method & CA & AS & DA & NL & TR & CA & CH & CL & CO & CP \\
\midrule
NT & \multirow{2}{*}{$\checkmark$} & \multirow{2}{*}{} & \multirow{2}{*}{$\checkmark$} & \multirow{2}{*}{} & \multirow{2}{*}{} & 0.1406 ± 0.0015 & \textbf{0.5399 ± 0.0007} & 0.4844 ± 0.0719 & 5.0793 ± 0.1121 & 2.3816 ± 0.0269 \\
+PET &  &  &  &  &  & \textbf{0.1392 ± 0.0014} (+1.00\%) & \textbf{0.5399 ± 0.0004} (-0.01\%) & 0.3109 ± 0.0178 (+35.82\%) & 5.0347 ± 0.1203 (+0.88\%) & \underline{2.3433 ± 0.0269} (+1.61\%) \\
\arrayrulecolor{lightgray}\hline\arrayrulecolor{black}
LG & \multirow{2}{*}{$\checkmark$} & \multirow{2}{*}{} & \multirow{2}{*}{$\checkmark$} & \multirow{2}{*}{} & \multirow{2}{*}{} & 0.1406 ± 0.0015 & \textbf{0.5399 ± 0.0007} & 0.4844 ± 0.0719 & 5.0793 ± 0.1121 & 2.3626 ± 0.0744 \\
+PET &  &  &  &  &  & \textbf{0.1392 ± 0.0010} (+1.00\%) & \textbf{0.5399 ± 0.0004} (-0.01\%) & \underline{0.3094 ± 0.0163} (+36.13\%) & 5.0253 ± 0.1506 (+1.06\%) & 2.3531 ± 0.0257 (+0.40\%) \\
\arrayrulecolor{lightgray}\hline\arrayrulecolor{black}
PS & \multirow{2}{*}{} & \multirow{2}{*}{$\checkmark$} & \multirow{2}{*}{} & \multirow{2}{*}{} & \multirow{2}{*}{} & 0.1406 ± 0.0015 & \textbf{0.5399 ± 0.0007} & 0.4844 ± 0.0719 & 5.0793 ± 0.1121 & 2.3576 ± 0.0224 \\
+PET &  &  &  &  &  & 0.1409 ± 0.0016 (-0.22\%) & 0.5401 ± 0.0007 (-0.04\%) & \textbf{0.3077 ± 0.0069} (+36.47\%) & \underline{4.8889 ± 0.1001} (+3.75\%) & 2.3516 ± 0.0499 (+0.25\%) \\
\arrayrulecolor{lightgray}\hline\arrayrulecolor{black}
SM & \multirow{2}{*}{} & \multirow{2}{*}{$\checkmark$} & \multirow{2}{*}{} & \multirow{2}{*}{$\checkmark$} & \multirow{2}{*}{} & 0.1406 ± 0.0015 & \textbf{0.5399 ± 0.0007} & 0.4844 ± 0.0719 & 5.0793 ± 0.1121 & 2.3576 ± 0.0224 \\
+PET &  &  &  &  &  & 0.1408 ± 0.0013 (-0.18\%) & 0.5402 ± 0.0007 (-0.05\%) & 0.3154 ± 0.0132 (+34.88\%) & 5.0447 ± 0.2000 (+0.68\%) & \textbf{2.3392 ± 0.0273} (+0.78\%) \\
\arrayrulecolor{lightgray}\hline\arrayrulecolor{black}
CL & \multirow{2}{*}{} & \multirow{2}{*}{$\checkmark$} & \multirow{2}{*}{$\checkmark$} & \multirow{2}{*}{$\checkmark$} & \multirow{2}{*}{$\checkmark$} & 0.1406 ± 0.0015 & \textbf{0.5399 ± 0.0007} & 0.4844 ± 0.0719 & 5.0793 ± 0.1121 & 2.3672 ± 0.0476 \\
+PET &  &  &  &  &  & 0.1408 ± 0.0011 (-0.15\%) & 0.5401 ± 0.0006 (-0.04\%) & 0.3112 ± 0.0052 (+35.76\%) & \textbf{4.8684 ± 0.2928} (+4.15\%) & 2.3541 ± 0.0149 (+0.55\%) \\
\midrule
\end{tabular}
}
\\
\vspace{0.5em}

\resizebox{0.95\textwidth}{!}{
\begin{tabular}{l!{\vrule}ccccc!{\vrule}ccccc}
\midrule
Method & CA & AS & DA & NL & TR & CPU & DIA & EN & FR & GR \\
\midrule
NT & \multirow{2}{*}{$\checkmark$} & \multirow{2}{*}{} & \multirow{2}{*}{$\checkmark$} & \multirow{2}{*}{} & \multirow{2}{*}{} & 2.8538 ± 0.0295 & 2467.5773 ± 36.9375 & 0.4842 ± 0.0240 & 0.7943 ± 0.0667 & 0.0057 ± 0.0001 \\
+PET &  &  &  &  &  & \underline{2.8177 ± 0.0578} (+1.26\%) & 2460.5213 ± 37.2611 (+0.29\%) & \underline{0.4682 ± 0.0277} (+3.31\%) & 0.7812 ± 0.0754 (+1.66\%) & 0.0056 ± 0.0001 (+1.27\%) \\
\arrayrulecolor{lightgray}\hline\arrayrulecolor{black}
LG & \multirow{2}{*}{$\checkmark$} & \multirow{2}{*}{} & \multirow{2}{*}{$\checkmark$} & \multirow{2}{*}{} & \multirow{2}{*}{} & 2.8587 ± 0.0368 & 2517.6621 ± 50.5441 & 0.4842 ± 0.0240 & 0.7387 ± 0.0984 & 0.0057 ± 0.0001 \\
+PET &  &  &  &  &  & 2.8265 ± 0.0277 (+1.13\%) & 2491.0303 ± 21.3126 (+1.06\%) & 0.4825 ± 0.0193 (+0.36\%) & \textbf{0.7081 ± 0.0449} (+4.15\%) & 0.0056 ± 0.0001 (+1.15\%) \\
\arrayrulecolor{lightgray}\hline\arrayrulecolor{black}
PS & \multirow{2}{*}{} & \multirow{2}{*}{$\checkmark$} & \multirow{2}{*}{} & \multirow{2}{*}{} & \multirow{2}{*}{} & 2.8557 ± 0.0583 & 2467.5773 ± 36.9375 & 0.4842 ± 0.0240 & 0.7387 ± 0.0984 & 0.0057 ± 0.0001 \\
+PET &  &  &  &  &  & \textbf{2.8133 ± 0.0276} (+1.48\%) & \underline{2447.4401 ± 21.7826} (+0.82\%) & 0.4795 ± 0.0189 (+0.98\%) & 0.7784 ± 0.0901 (-5.37\%) & \textbf{0.0055 ± 0.0001} (+3.70\%) \\
\arrayrulecolor{lightgray}\hline\arrayrulecolor{black}
SM & \multirow{2}{*}{} & \multirow{2}{*}{$\checkmark$} & \multirow{2}{*}{} & \multirow{2}{*}{$\checkmark$} & \multirow{2}{*}{} & 2.8557 ± 0.0583 & 2467.5773 ± 36.9375 & 0.4842 ± 0.0240 & 0.7387 ± 0.0984 & 0.0057 ± 0.0001 \\
+PET &  &  &  &  &  & 2.8270 ± 0.0513 (+1.00\%) & \textbf{2447.4315 ± 41.0115} (+0.82\%) & 0.4725 ± 0.0276 (+2.43\%) & \underline{0.7164 ± 0.0808} (+3.02\%) & \textbf{0.0055 ± 0.0001} (+2.84\%) \\
\arrayrulecolor{lightgray}\hline\arrayrulecolor{black}
CL & \multirow{2}{*}{} & \multirow{2}{*}{$\checkmark$} & \multirow{2}{*}{$\checkmark$} & \multirow{2}{*}{$\checkmark$} & \multirow{2}{*}{$\checkmark$} & 2.8538 ± 0.0295 & 2510.3224 ± 67.9050 & 0.4842 ± 0.0240 & 0.7790 ± 0.0430 & 0.0057 ± 0.0002 \\
+PET &  &  &  &  &  & 2.8404 ± 0.0366 (+0.47\%) & 2469.6160 ± 94.8212 (+1.62\%) & \textbf{0.4608 ± 0.0251} (+4.83\%) & 0.7710 ± 0.1161 (+1.02\%) & 0.0056 ± 0.0001 (+2.25\%) \\
\midrule
\end{tabular}
}
\\
\vspace{0.5em}

\resizebox{0.95\textwidth}{!}{
\begin{tabular}{l!{\vrule}ccccc!{\vrule}ccccc}
\midrule
Method & CA & AS & DA & NL & TR & KI & LI & MU & PL & SE \\
\midrule
NT & \multirow{2}{*}{$\checkmark$} & \multirow{2}{*}{} & \multirow{2}{*}{$\checkmark$} & \multirow{2}{*}{} & \multirow{2}{*}{} & 0.0700 ± 0.0008 & 2.9361 ± 0.0481 & 25.5915 ± 1.4039 & 592.3889 ± 0.9151 & \textbf{0.6573 ± 0.0146} \\
+PET &  &  &  &  &  & 0.0671 ± 0.0006 (+4.20\%) & 2.8963 ± 0.0163 (+1.35\%) & \underline{22.2368 ± 0.5336} (+13.11\%) & 591.6694 ± 0.7614 (+0.12\%) & 0.6631 ± 0.0196 (-0.87\%) \\
\arrayrulecolor{lightgray}\hline\arrayrulecolor{black}
LG & \multirow{2}{*}{$\checkmark$} & \multirow{2}{*}{} & \multirow{2}{*}{$\checkmark$} & \multirow{2}{*}{} & \multirow{2}{*}{} & 0.0700 ± 0.0008 & 2.9361 ± 0.0481 & 25.5915 ± 1.4039 & 592.3889 ± 0.9151 & 0.6746 ± 0.0154 \\
+PET &  &  &  &  &  & 0.0671 ± 0.0006 (+4.20\%) & \underline{2.8869 ± 0.0226} (+1.67\%) & 22.6987 ± 1.3330 (+11.30\%) & \underline{591.5936 ± 0.6613} (+0.13\%) & 0.6644 ± 0.0160 (+1.51\%) \\
\arrayrulecolor{lightgray}\hline\arrayrulecolor{black}
PS & \multirow{2}{*}{} & \multirow{2}{*}{$\checkmark$} & \multirow{2}{*}{} & \multirow{2}{*}{} & \multirow{2}{*}{} & 0.0700 ± 0.0008 & 2.9361 ± 0.0481 & 25.8524 ± 0.4659 & 592.0474 ± 1.0442 & 0.6914 ± 0.0265 \\
+PET &  &  &  &  &  & 0.0672 ± 0.0008 (+3.95\%) & 2.8996 ± 0.0291 (+1.24\%) & 22.8125 ± 0.7793 (+11.76\%) & 591.7507 ± 0.8222 (+0.05\%) & \underline{0.6596 ± 0.0174} (+4.60\%) \\
\arrayrulecolor{lightgray}\hline\arrayrulecolor{black}
SM & \multirow{2}{*}{} & \multirow{2}{*}{$\checkmark$} & \multirow{2}{*}{} & \multirow{2}{*}{$\checkmark$} & \multirow{2}{*}{} & 0.0700 ± 0.0008 & 2.9361 ± 0.0481 & 25.8524 ± 0.4659 & 592.0474 ± 1.0442 & 0.6914 ± 0.0265 \\
+PET &  &  &  &  &  & \textbf{0.0665 ± 0.0005} (+4.95\%) & \textbf{2.8798 ± 0.0155} (+1.92\%) & 22.9815 ± 0.8778 (+11.11\%) & \textbf{591.3838 ± 0.7633} (+0.11\%) & 0.6662 ± 0.0095 (+3.65\%) \\
\arrayrulecolor{lightgray}\hline\arrayrulecolor{black}
CL & \multirow{2}{*}{} & \multirow{2}{*}{$\checkmark$} & \multirow{2}{*}{$\checkmark$} & \multirow{2}{*}{$\checkmark$} & \multirow{2}{*}{$\checkmark$} & 0.0700 ± 0.0008 & 2.9361 ± 0.0481 & 25.6402 ± 1.1414 & 592.3889 ± 0.9151 & 0.6914 ± 0.0265 \\
+PET &  &  &  &  &  & \underline{0.0667 ± 0.0008} (+4.73\%) & 2.8965 ± 0.0277 (+1.35\%) & \textbf{20.9698 ± 0.9666} (+18.22\%) & 591.7895 ± 0.6451 (+0.10\%) & 0.6655 ± 0.0153 (+3.74\%) \\
\midrule
\end{tabular}
}
\\
\vspace{0.5em}

\resizebox{0.95\textwidth}{!}{
\begin{tabular}{l!{\vrule}ccccc!{\vrule}ccccc}
\midrule
Method & CA & AS & DA & NL & TR & SO & SP & STO & TE & WIS \\
\midrule
NT & \multirow{2}{*}{$\checkmark$} & \multirow{2}{*}{} & \multirow{2}{*}{$\checkmark$} & \multirow{2}{*}{} & \multirow{2}{*}{} & 17.8475 ± 0.4801 & 0.1060 ± 0.0012 & 0.6951 ± 0.0327 & 4.3082 ± 0.2317 & 38.7461 ± 0.9942 \\
+PET &  &  &  &  &  & 17.8932 ± 0.1668 (-0.26\%) & \textbf{0.1002 ± 0.0013} (+5.51\%) & 0.6869 ± 0.0306 (+1.18\%) & 4.2205 ± 0.1868 (+2.04\%) & \textbf{37.4950 ± 0.1718} (+3.23\%) \\
\arrayrulecolor{lightgray}\hline\arrayrulecolor{black}
LG & \multirow{2}{*}{$\checkmark$} & \multirow{2}{*}{} & \multirow{2}{*}{$\checkmark$} & \multirow{2}{*}{} & \multirow{2}{*}{} & 17.8475 ± 0.4801 & 0.1060 ± 0.0012 & 0.6951 ± 0.0327 & 4.3082 ± 0.2317 & 38.1507 ± 0.1211 \\
+PET &  &  &  &  &  & \textbf{17.7000 ± 0.5455} (+0.83\%) & \underline{0.1015 ± 0.0019} (+4.24\%) & 0.6880 ± 0.0377 (+1.02\%) & 4.2510 ± 0.2027 (+1.33\%) & 37.5892 ± 0.2247 (+1.47\%) \\
\arrayrulecolor{lightgray}\hline\arrayrulecolor{black}
PS & \multirow{2}{*}{} & \multirow{2}{*}{$\checkmark$} & \multirow{2}{*}{} & \multirow{2}{*}{} & \multirow{2}{*}{} & 17.8475 ± 0.4801 & 0.1060 ± 0.0012 & 0.6951 ± 0.0327 & \underline{4.1132 ± 0.1258} & 38.0113 ± 0.2432 \\
+PET &  &  &  &  &  & \underline{17.8157 ± 0.4119} (+0.18\%) & 0.1039 ± 0.0020 (+1.99\%) & 0.6956 ± 0.0191 (-0.07\%) & \textbf{4.0375 ± 0.1762} (+1.84\%) & 37.8419 ± 0.2728 (+0.45\%) \\
\arrayrulecolor{lightgray}\hline\arrayrulecolor{black}
SM & \multirow{2}{*}{} & \multirow{2}{*}{$\checkmark$} & \multirow{2}{*}{} & \multirow{2}{*}{$\checkmark$} & \multirow{2}{*}{} & 17.8475 ± 0.4801 & 0.1060 ± 0.0012 & 0.6951 ± 0.0327 & \underline{4.1132 ± 0.1258} & 38.0113 ± 0.2432 \\
+PET &  &  &  &  &  & 17.8315 ± 0.2163 (+0.09\%) & 0.1039 ± 0.0015 (+1.97\%) & \underline{0.6763 ± 0.0155} (+2.71\%) & 4.1229 ± 0.1091 (-0.24\%) & 37.9194 ± 0.2819 (+0.24\%) \\
\arrayrulecolor{lightgray}\hline\arrayrulecolor{black}
CL & \multirow{2}{*}{} & \multirow{2}{*}{$\checkmark$} & \multirow{2}{*}{$\checkmark$} & \multirow{2}{*}{$\checkmark$} & \multirow{2}{*}{$\checkmark$} & 17.8475 ± 0.4801 & 0.1060 ± 0.0012 & 0.6951 ± 0.0327 & 4.3082 ± 0.2317 & 38.7461 ± 0.9942 \\
+PET &  &  &  &  &  & 17.9177 ± 0.2947 (-0.39\%) & 0.1029 ± 0.0015 (+2.95\%) & \textbf{0.6734 ± 0.0176} (+3.12\%) & 4.2480 ± 0.2301 (+1.40\%) & \underline{37.5466 ± 0.1034} (+3.10\%) \\
\bottomrule
\end{tabular}
}
\caption{Detailed performance on different graph estimation approaches for FT-Transformer on regression datasets. Each result is averaged over 5 random seeds. Use RMSE (lower is better). CA = Causal, AS = Association, DA = Directed Acyclic, NL = Nonlinear, TR = Tree. Results show baseline performance and Tab-PET performance with improvement percentages in parentheses compared to the corresponding baselines. \textbf{Bold} indicates the best result, and \underline{underline} indicates the second-best result.}
\label{tab:ablation_regression_results}
\end{table}

\noindent\textbf{Baseline Performance Variations:} When comparing the five graph estimation approaches, minor differences in baseline performance across certain datasets can be observed. This variation arises from the inherent differences in the Laplacian eigenvectors produced by each graph estimation method. Since different methods generate distinct graph structures with varying spectral properties, the automatic k selection algorithm (detailed in technical appendix \ref{app:4}, along with Algorithm \ref{alg:auto_k_selection}) selects different numbers of eigenvectors for each approach. Consequently, this leads to different PE dimensions ($d_{\text{PE}}$) across methods, which in turn affects the token dimension ($d_{\text{token}} = 192 - d_{\text{PE}}$) according to our fair comparison framework.

\subsection{Timing Analysis of 5 Graph Estimation Approaches} \label{app:7.2}

\begin{table}[H]
\centering
\tiny
\setlength{\tabcolsep}{3pt}
\renewcommand{\arraystretch}{1.2}
\definecolor{lightgray}{gray}{0.7}
\begin{tabular}{l!{\vrule}ccccccccccccccccccccccccc}
\toprule
Method & AU & GE & SA & BL & CHU & CM & CR & DI & DN & EY & FI & HE & JA & KC & KR & MA & PH & QSA & ST & SY & TI & VE & WI & WIN & YE \\
\midrule
NT & 2.68 & 7.59 & 119.97 & 0.26 & 3.77 & 2.32 & 22.69 & \textbf{0.16} & 1173.45 & 5.62 & 117.22 & 11.59 & 2048.70 & 11.40 & 74.16 & 131.11 & 1.00 & 20.38 & 13.24 & 0.75 & 5.74 & 7.54 & 0.83 & 0.51 & 0.27 \\
\arrayrulecolor{lightgray}\hline\arrayrulecolor{black}
LG & \underline{0.44} & 0.60 & 0.63 & 0.40 & 0.51 & 1.54 & 1.91 & 0.23 & 278.84 & 0.43 & 1.23 & 0.52 & 118.83 & 0.37 & 6.38 & 119.21 & 0.28 & 0.67 & 0.34 & 0.56 & \underline{0.54} & 0.39 & 0.23 & \textbf{0.17} & \underline{0.18} \\
\arrayrulecolor{lightgray}\hline\arrayrulecolor{black}
PS & 0.57 & \textbf{0.32} & 0.74 & \underline{0.19} & \underline{0.45} & \underline{0.32} & \underline{0.48} & 0.28 & \underline{10.08} & \underline{0.27} & \underline{0.38} & \textbf{0.20} & \underline{6.17} & 0.34 & \underline{0.64} & \underline{5.33} & 0.26 & \underline{0.51} & 0.24 & 0.48 & 0.67 & \underline{0.34} & \textbf{0.15} & \underline{0.23} & 0.21 \\
\arrayrulecolor{lightgray}\hline\arrayrulecolor{black}
SM & \textbf{0.33} & \underline{0.34} & \textbf{0.26} & \textbf{0.17} & \textbf{0.31} & 0.35 & 0.51 & \underline{0.19} & 10.12 & 0.28 & 0.76 & \underline{0.21} & 6.69 & \underline{0.21} & 0.90 & 5.49 & \underline{0.24} & 1.48 & \underline{0.20} & \underline{0.40} & 0.65 & 0.46 & 0.22 & 0.26 & 0.22 \\
\arrayrulecolor{lightgray}\hline\arrayrulecolor{black}
CL & 0.59 & 0.34 & \underline{0.41} & 0.37 & 0.50 & \textbf{0.19} & \textbf{0.37} & 0.19 & \textbf{2.68} & \textbf{0.26} & \textbf{0.26} & 0.25 & \textbf{1.14} & \textbf{0.20} & \textbf{0.42} & \textbf{1.25} & \textbf{0.20} & \textbf{0.24} & \textbf{0.16} & \textbf{0.24} & \textbf{0.28} & \textbf{0.22} & \underline{0.17} & 0.39 & \textbf{0.14} \\
\bottomrule
\end{tabular}
\vspace{1em}

\begin{tabular}{l!{\vrule}ccccccccccccccccccccccccc}
\toprule
Method & QS & AB & AI & BO & BOS & CA & CH & CL & CO & CP & CPU & DIA & EN & FR & GR & KI & LI & MU & PL & SE & SO & SP & STO & TE & WIS \\
\midrule
NT & \underline{0.17} & 2.95 & 0.20 & 0.50 & 0.48 & 0.22 & \textbf{0.14} & \underline{0.19} & 0.23 & 4.56 & 2.96 & 4.27 & \underline{0.17} & 0.49 & 0.24 & \textbf{0.16} & \textbf{0.15} & 1.19 & 0.44 & 1.08 & 2.20 & 0.20 & 0.48 & 22.88 & 12.19 \\
\arrayrulecolor{lightgray}\hline\arrayrulecolor{black}
LG & 0.18 & 0.24 & \underline{0.20} & 0.22 & \underline{0.24} & \underline{0.19} & 0.19 & \textbf{0.16} & \textbf{0.15} & \underline{0.28} & \underline{0.27} & 0.48 & 0.19 & 0.40 & 0.35 & 0.17 & 0.16 & \underline{0.21} & \underline{0.23} & 0.78 & 0.49 & 0.19 & \textbf{0.20} & 5.72 & 0.47 \\
\arrayrulecolor{lightgray}\hline\arrayrulecolor{black}
PS & \textbf{0.15} & \textbf{0.18} & 0.29 & \textbf{0.17} & \textbf{0.23} & \textbf{0.17} & 0.33 & 0.35 & 0.41 & 1.40 & 0.39 & \underline{0.25} & 0.19 & 0.48 & \underline{0.18} & 0.17 & \underline{0.15} & 0.22 & \textbf{0.19} & \underline{0.28} & 0.33 & 0.21 & 0.52 & 1.33 & 0.80 \\
\arrayrulecolor{lightgray}\hline\arrayrulecolor{black}
SM & 0.18 & \underline{0.20} & 0.24 & \underline{0.17} & 0.25 & 0.21 & 0.24 & 0.48 & 0.29 & 0.42 & 0.40 & 0.31 & \textbf{0.15} & \textbf{0.21} & 0.20 & \underline{0.16} & 0.34 & 0.37 & 0.42 & 0.60 & \textbf{0.25} & \underline{0.16} & 0.24 & \underline{1.29} & \underline{0.27} \\
\arrayrulecolor{lightgray}\hline\arrayrulecolor{black}
CL & 0.18 & 0.22 & \textbf{0.15} & 0.24 & 0.25 & 1.17 & \underline{0.15} & 0.25 & \underline{0.17} & \textbf{0.23} & \textbf{0.24} & \textbf{0.18} & 0.30 & \underline{0.22} & \textbf{0.16} & 0.17 & 0.26 & \textbf{0.17} & 1.10 & \textbf{0.21} & \underline{0.28} & \textbf{0.16} & \underline{0.21} & \textbf{0.47} & \textbf{0.26} \\
\bottomrule
\end{tabular}
\caption{Graph estimation and PE computation time (minutes) for different graph estimation methods across all datasets. Times represent the computational overhead of graph estimation and PE generation before transformer training. \textbf{Bold} indicates fastest method, \underline{underlined} indicates second-fastest method. NT = NOTEARS, LG = LiNGAM, PS = Pearson, SM = Spearman, CL = Chow-Liu.}
\label{tab:timing_combined_results}
\end{table}

\newpage
\subsection{Classification and Regression on Multiple Baselines} \label{app:7.3}

\begin{table}[H]
\centering
\tiny
\definecolor{lightgray}{gray}{0.7}
\begin{tabular}{l!{\vrule}ccccc}
\toprule
Model & AU & GE & SA & BL & CHU \\
\midrule
XGBoost & 0.8371 ± 0.0167\ddag & 0.6136 ± 0.0033 & 0.7457 ± 0.0163 & 0.7253 ± 0.0160 & 0.8928 ± 0.0065 \\
CatBoost & 0.8304 ± 0.0173 & 0.6352 ± 0.0044 & 0.7759 ± 0.0417 & 0.7053 ± 0.0089 & 0.8943 ± 0.0026 \\
\arrayrulecolor{lightgray}\hline\arrayrulecolor{black}
FT-Trans & 0.8283 ± 0.0093 & 0.6412 ± 0.0081 & 0.7791 ± 0.0266 & 0.7481 ± 0.0078 & 0.9076 ± 0.0043\ddag \\
FT-Trans+PET & \textbf{0.8379 ± 0.0162\dag*} (\textbf{+1.16\%}) & \textbf{0.6502 ± 0.0080} (\textbf{+1.41\%}) & \textbf{0.7924 ± 0.0252} (\textbf{+1.71\%}) & \textbf{0.7582 ± 0.0055\dag*} (\textbf{+1.35\%}) & \textbf{0.9172 ± 0.0046\dag*} (\textbf{+1.05\%}) \\
\arrayrulecolor{lightgray}\hline\arrayrulecolor{black}
SAINT & \textbf{0.8264 ± 0.0075} & \textbf{0.6629 ± 0.0081\dag} & 0.8140 ± 0.0675\ddag & 0.7329 ± 0.0314 & 0.8545 ± 0.0132 \\
SAINT+PET & 0.8263 ± 0.0074 (-0.01\%) & 0.6601 ± 0.0075\ddag* (-0.42\%) & \textbf{0.8619 ± 0.0617\dag*} (\textbf{+5.88\%}) & \textbf{0.7558 ± 0.0098\ddag*} (\textbf{+3.13\%}) & \textbf{0.8783 ± 0.0053} (\textbf{+2.79\%}) \\
\arrayrulecolor{lightgray}\hline\arrayrulecolor{black}
TabTrans & 0.8222 ± 0.0125 & — & — & — & 0.8305 ± 0.0083 \\
TabTrans+PET & \textbf{0.8359 ± 0.0082} (\textbf{+1.66\%}) & — & — & — & \textbf{0.8399 ± 0.0134} (\textbf{+1.13\%}) \\
\midrule
\end{tabular}
\\
\vspace{0.5em}

\begin{tabular}{l!{\vrule}ccccc}
\midrule
Model & CM & CR & DI & DN & EY \\
\midrule
XGBoost & 0.5707 ± 0.0124\dag & 0.7224 ± 0.0134\dag & 0.7205 ± 0.0154 & 0.9664 ± 0.0019\dag & 0.6964 ± 0.0041 \\
CatBoost & 0.5620 ± 0.0152 & 0.6814 ± 0.0162 & 0.7095 ± 0.0078 & 0.9643 ± 0.0019\ddag & 0.7260 ± 0.0024 \\
\arrayrulecolor{lightgray}\hline\arrayrulecolor{black}
FT-Trans & 0.5449 ± 0.0222 & 0.6443 ± 0.0244 & 0.7691 ± 0.0148\ddag & 0.9528 ± 0.0071 & 0.7480 ± 0.0078\ddag \\
FT-Trans+PET & \textbf{0.5705 ± 0.0060\ddag*} (\textbf{+4.70\%}) & \textbf{0.6821 ± 0.0201} (\textbf{+5.88\%}) & \textbf{0.7712 ± 0.0179\dag*} (\textbf{+0.28\%}) & \textbf{0.9617 ± 0.0042} (\textbf{+0.93\%}) & \textbf{0.7499 ± 0.0060\dag*} (\textbf{+0.24\%}) \\
\arrayrulecolor{lightgray}\hline\arrayrulecolor{black}
SAINT & 0.5572 ± 0.0088 & \textbf{0.7093 ± 0.0097\ddag} & 0.7302 ± 0.0232 & 0.9606 ± 0.0059 & 0.6870 ± 0.0059 \\
SAINT+PET & \textbf{0.5634 ± 0.0061} (\textbf{+1.12\%}) & 0.7086 ± 0.0137 (-0.10\%) & \textbf{0.7396 ± 0.0347} (\textbf{+1.29\%}) & \textbf{0.9636 ± 0.0036} (\textbf{+0.32\%}) & \textbf{0.7301 ± 0.0050} (\textbf{+6.27\%}) \\
\arrayrulecolor{lightgray}\hline\arrayrulecolor{black}
TabTrans & 0.4854 ± 0.0090 & 0.6612 ± 0.0304 & — & 0.9561 ± 0.0045 & 0.5948 ± 0.0037 \\
TabTrans+PET & \textbf{0.4884 ± 0.0156} (\textbf{+0.60\%}) & \textbf{0.6833 ± 0.0137} (\textbf{+3.35\%}) & — & \textbf{0.9603 ± 0.0032} (\textbf{+0.44\%}) & \textbf{0.5977 ± 0.0049} (\textbf{+0.48\%}) \\
\midrule
\end{tabular}
\\
\vspace{0.5em}

\begin{tabular}{l!{\vrule}ccccc}
\midrule
Model & FI & HE & JA & KC & KR \\
\midrule
XGBoost & 0.5243 ± 0.0037\dag & 0.7107 ± 0.0047 & 0.8113 ± 0.0016\ddag & 0.7038 ± 0.0111 & 0.9965 ± 0.0007\dag \\
CatBoost & 0.5018 ± 0.0087\ddag & 0.7223 ± 0.0014 & 0.8189 ± 0.0066\dag & 0.7192 ± 0.0066\ddag & 0.9921 ± 0.0011 \\
\arrayrulecolor{lightgray}\hline\arrayrulecolor{black}
FT-Trans & 0.4599 ± 0.0075 & 0.7235 ± 0.0059 & 0.8049 ± 0.0059 & 0.7138 ± 0.0095 & \textbf{0.9960 ± 0.0007\ddag} \\
FT-Trans+PET & \textbf{0.4700 ± 0.0121} (\textbf{+2.20\%}) & \textbf{0.7299 ± 0.0018\dag*} (\textbf{+0.89\%}) & \textbf{0.8062 ± 0.0095} (\textbf{+0.16\%}) & \textbf{0.7308 ± 0.0062\dag*} (\textbf{+2.39\%}) & 0.9954 ± 0.0000 (-0.06\%) \\
\arrayrulecolor{lightgray}\hline\arrayrulecolor{black}
SAINT & 0.4710 ± 0.0113 & 0.7250 ± 0.0009 & 0.7972 ± 0.0077 & 0.7148 ± 0.0080 & 0.9918 ± 0.0015 \\
SAINT+PET & \textbf{0.4744 ± 0.0047} (\textbf{+0.71\%}) & \textbf{0.7279 ± 0.0037\ddag*} (\textbf{+0.41\%}) & \textbf{0.8069 ± 0.0054} (\textbf{+1.21\%}) & \textbf{0.7182 ± 0.0095} (\textbf{+0.49\%}) & \textbf{0.9948 ± 0.0012} (\textbf{+0.30\%}) \\
\arrayrulecolor{lightgray}\hline\arrayrulecolor{black}
TabTrans & — & — & 0.7855 ± 0.0034 & — & 0.9933 ± 0.0012 \\
TabTrans+PET & — & — & \textbf{0.7912 ± 0.0064} (\textbf{+0.73\%}) & — & \textbf{0.9944 ± 0.0012} (\textbf{+0.12\%}) \\
\midrule
\end{tabular}
\\
\vspace{0.5em}

\begin{tabular}{l!{\vrule}ccccc}
\midrule
Model & MA & PH & QSA & ST & SY \\
\midrule
XGBoost & 0.8610 ± 0.0048\dag & 0.8451 ± 0.0014 & 0.8580 ± 0.0156\dag & 0.8063 ± 0.0012 & 0.9343 ± 0.0029 \\
CatBoost & 0.8559 ± 0.0073\ddag & 0.8603 ± 0.0028 & 0.8423 ± 0.0184 & 0.8099 ± 0.0105\ddag & 0.9532 ± 0.0015\dag \\
\arrayrulecolor{lightgray}\hline\arrayrulecolor{black}
FT-Trans & 0.7184 ± 0.0271 & 0.8607 ± 0.0071\ddag & 0.8313 ± 0.0128 & 0.7949 ± 0.0104 & 0.9469 ± 0.0042 \\
FT-Trans+PET & \textbf{0.7398 ± 0.0172} (\textbf{+2.98\%}) & \textbf{0.8663 ± 0.0036\dag*} (\textbf{+0.65\%}) & \textbf{0.8441 ± 0.0145} (\textbf{+1.54\%}) & \textbf{0.8112 ± 0.0104\dag*} (\textbf{+2.05\%}) & \textbf{0.9479 ± 0.0044\ddag*} (\textbf{+0.10\%}) \\
\arrayrulecolor{lightgray}\hline\arrayrulecolor{black}
SAINT & \textbf{0.8090 ± 0.0094} & 0.8521 ± 0.0016 & 0.8414 ± 0.0141 & 0.7978 ± 0.0027 & 0.9403 ± 0.0023 \\
SAINT+PET & 0.8048 ± 0.0157 (-0.52\%) & \textbf{0.8535 ± 0.0051} (\textbf{+0.16\%}) & \textbf{0.8540 ± 0.0205\ddag*} (\textbf{+1.49\%}) & \textbf{0.8051 ± 0.0159} (\textbf{+0.92\%}) & \textbf{0.9432 ± 0.0040} (\textbf{+0.31\%}) \\
\arrayrulecolor{lightgray}\hline\arrayrulecolor{black}
TabTrans & — & — & — & — & — \\
TabTrans+PET & — & — & — & — & — \\
\midrule
\end{tabular}
\\
\vspace{0.5em}

\begin{tabular}{l!{\vrule}ccccc}
\midrule
Model & TI & VE & WI & WIN & YE \\
\midrule
XGBoost & 0.9776 ± 0.0067 & 0.7589 ± 0.0072 & 0.9138 ± 0.0073 & 0.3955 ± 0.0017\dag & 0.5822 ± 0.0074\dag \\
CatBoost & 0.9813 ± 0.0105\dag & 0.7471 ± 0.0076 & 0.9301 ± 0.0120 & 0.2732 ± 0.0740 & 0.5672 ± 0.0052\ddag \\
\arrayrulecolor{lightgray}\hline\arrayrulecolor{black}
FT-Trans & 0.9719 ± 0.0061 & 0.7856 ± 0.0180\ddag & 0.9570 ± 0.0079 & 0.3629 ± 0.0121 & 0.5449 ± 0.0062 \\
FT-Trans+PET & \textbf{0.9772 ± 0.0047} (\textbf{+0.54\%}) & \textbf{0.8080 ± 0.0134\dag*} (\textbf{+2.84\%}) & \textbf{0.9704 ± 0.0080} (\textbf{+1.39\%}) & \textbf{0.3755 ± 0.0158\ddag*} (\textbf{+3.47\%}) & \textbf{0.5618 ± 0.0214} (\textbf{+3.10\%}) \\
\arrayrulecolor{lightgray}\hline\arrayrulecolor{black}
SAINT & 0.9689 ± 0.0017 & 0.7619 ± 0.0451 & 0.9766 ± 0.0025\ddag & 0.3503 ± 0.0245 & 0.5450 ± 0.0175 \\
SAINT+PET & \textbf{0.9728 ± 0.0030} (\textbf{+0.40\%}) & \textbf{0.7847 ± 0.0346} (\textbf{+2.99\%}) & \textbf{0.9821 ± 0.0017\dag*} (\textbf{+0.56\%}) & \textbf{0.3509 ± 0.0194} (\textbf{+0.15\%}) & \textbf{0.5509 ± 0.0197} (\textbf{+1.09\%}) \\
\arrayrulecolor{lightgray}\hline\arrayrulecolor{black}
TabTrans & 0.9727 ± 0.0082 & — & — & — & — \\
TabTrans+PET & \textbf{0.9797 ± 0.0066\ddag*} (\textbf{+0.72\%}) & — & — & — & — \\
\bottomrule
\end{tabular}
\caption{Performance comparison on classification datasets. Each result is averaged over 5 random seeds. Use balanced accuracy (higher is better). +PET = corresponding Tab-PET variant using Spearman for graph estimation. \dag and \ddag indicate best and second-best results. * indicates our PET methods are in the top 2. Results show baseline performance and Tab-PET performance with improvement percentages in parentheses.}
\label{tab:classification_results}
\end{table}

\newpage
\begin{table}[H]
\centering
\tiny
\begin{tabular}{l!{\vrule}ccccc}
\toprule
Model & QS & AB & AI & BO & BOS \\
\midrule
XGBoost & 1.0467 ± 0.0359 & 2.3327 ± 0.0058 & 2.2468 ± 0.0321 & 3.4243 ± 0.0428 & 3.5448 ± 0.0925 \\
CatBoost & 0.9999 ± 0.0089 & 2.2427 ± 0.0202 & 1.4447 ± 0.0187\ddag & 0.5867 ± 0.0416 & 3.1593 ± 0.0374 \\
\arrayrulecolor{lightgray}\hline\arrayrulecolor{black}
FT-Trans & 0.9706 ± 0.0163 & 2.2002 ± 0.0326 & 1.4659 ± 0.0911 & 0.3268 ± 0.0891\ddag & 3.2319 ± 0.1545 \\
FT-Trans+PET & \textbf{0.9606 ± 0.0199} (\textbf{+1.03\%}) & \textbf{2.1681 ± 0.0124\ddag*} (\textbf{+1.46\%}) & \textbf{1.3096 ± 0.0510\dag*} (\textbf{+10.67\%}) & \textbf{0.2773 ± 0.0480\dag*} (\textbf{+15.13\%}) & \textbf{2.9934 ± 0.2309\dag*} (\textbf{+7.38\%}) \\
\arrayrulecolor{lightgray}\hline\arrayrulecolor{black}
SAINT & 0.9389 ± 0.0082\ddag & 2.1867 ± 0.0049 & \textbf{1.9570 ± 0.0142} & 1.0605 ± 0.1709 & 3.1944 ± 0.2299 \\
SAINT+PET & \textbf{0.9319 ± 0.0119\dag*} (\textbf{+0.74\%}) & \textbf{2.1604 ± 0.0104\dag*} (\textbf{+1.21\%}) & 2.0052 ± 0.0671 (-2.46\%) & \textbf{1.0393 ± 0.1498} (\textbf{+2.00\%}) & \textbf{3.0902 ± 0.2620\ddag*} (\textbf{+3.26\%}) \\
\arrayrulecolor{lightgray}\hline\arrayrulecolor{black}
TabTrans & — & — & — & — & 5.2031 ± 0.0725 \\
TabTrans+PET & — & — & — & — & \textbf{5.1731 ± 0.0907} (\textbf{+0.58\%}) \\
\midrule
\end{tabular}
\\
\vspace{0.5em}

\begin{tabular}{l!{\vrule}ccccc}
\midrule
Model & CA & CH & CL & CO & CP \\
\midrule
XGBoost & 0.1692 ± 0.0010 & 0.5932 ± 0.0095 & 0.2998 ± 0.0254\ddag & 5.0495 ± 0.1027 & 6.3669 ± 0.0114 \\
CatBoost & 0.1327 ± 0.0001\dag & 0.5401 ± 0.0064\ddag & 0.2723 ± 0.0365\dag & 4.3052 ± 0.1735\dag & 2.2881 ± 0.0211\dag \\
\arrayrulecolor{lightgray}\hline\arrayrulecolor{black}
FT-Trans & \textbf{0.1406 ± 0.0015\ddag} & \textbf{0.5399 ± 0.0007\dag} & 0.4844 ± 0.0719 & 5.0793 ± 0.1121 & 2.3576 ± 0.0224 \\
FT-Trans+PET & 0.1408 ± 0.0013 (-0.18\%) & 0.5402 ± 0.0007 (-0.05\%) & \textbf{0.3154 ± 0.0132} (\textbf{+34.88\%}) & \textbf{5.0447 ± 0.2000} (\textbf{+0.68\%}) & \textbf{2.3392 ± 0.0273} (\textbf{+0.78\%}) \\
\arrayrulecolor{lightgray}\hline\arrayrulecolor{black}
SAINT & \textbf{0.1454 ± 0.0023} & 0.5750 ± 0.0135 & 0.4909 ± 0.1181 & 5.1688 ± 0.1554 & 2.3389 ± 0.0171 \\
SAINT+PET & 0.1463 ± 0.0017 (-0.60\%) & \textbf{0.5474 ± 0.0029} (\textbf{+4.80\%}) & \textbf{0.3758 ± 0.0450} (\textbf{+23.45\%}) & \textbf{4.9838 ± 0.0845\ddag*} (\textbf{+3.58\%}) & \textbf{2.3324 ± 0.0221\ddag*} (\textbf{+0.28\%}) \\
\arrayrulecolor{lightgray}\hline\arrayrulecolor{black}
TabTrans & — & — & 1.3689 ± 0.0339 & — & — \\
TabTrans+PET & — & — & \textbf{1.3275 ± 0.0377} (\textbf{+3.02\%}) & — & — \\
\midrule
\end{tabular}
\\
\vspace{0.5em}

\begin{tabular}{l!{\vrule}ccccc}
\midrule
Model & CPU & DIA & EN & FR & GR \\
\midrule
XGBoost & 3.1125 ± 0.0121 & 1156.5345 ± 3.0987 & 0.4756 ± 0.0135 & 0.9801 ± 0.0043 & 0.0155 ± 0.0000 \\
CatBoost & 2.7129 ± 0.0198\dag & 529.3698 ± 2.4176\dag & 0.4330 ± 0.0369\dag & 0.9872 ± 0.0833 & 0.0071 ± 0.0001 \\
\arrayrulecolor{lightgray}\hline\arrayrulecolor{black}
FT-Trans & 2.8557 ± 0.0583 & 2467.5773 ± 36.9375 & 0.4842 ± 0.0240 & 0.7387 ± 0.0984\ddag & 0.0057 ± 0.0001 \\
FT-Trans+PET & \textbf{2.8270 ± 0.0513} (\textbf{+1.00\%}) & \textbf{2447.4315 ± 41.0115} (\textbf{+0.82\%}) & \textbf{0.4725 ± 0.0276\ddag*} (\textbf{+2.43\%}) & \textbf{0.7164 ± 0.0808\dag*} (\textbf{+3.02\%}) & \textbf{0.0055 ± 0.0001\dag*} (\textbf{+2.84\%}) \\
\arrayrulecolor{lightgray}\hline\arrayrulecolor{black}
SAINT & \textbf{2.7914 ± 0.0274\ddag} & 537.5104 ± 3.1526 & \textbf{0.5472 ± 0.0245} & \textbf{0.9468 ± 0.0724} & 0.0058 ± 0.0000 \\
SAINT+PET & 2.8071 ± 0.0376 (-0.56\%) & \textbf{535.4909 ± 2.7511\ddag*} (\textbf{+0.38\%}) & 0.5503 ± 0.0150 (-0.56\%) & 0.9500 ± 0.0612 (-0.34\%) & \textbf{0.0056 ± 0.0000\ddag*} (\textbf{+3.68\%}) \\
\arrayrulecolor{lightgray}\hline\arrayrulecolor{black}
TabTrans & — & 1038.8152 ± 42.4760 & — & — & — \\
TabTrans+PET & — & \textbf{1024.9435 ± 30.2794} (\textbf{+1.34\%}) & — & — & — \\
\midrule
\end{tabular}
\\
\vspace{0.5em}

\begin{tabular}{l!{\vrule}ccccc}
\midrule
Model & KI & LI & MU & PL & SE \\
\midrule
XGBoost & 0.1529 ± 0.0010 & 2.8577 ± 0.0135\dag & 30.1484 ± 0.9435 & 238.0414 ± 1.6642 & 0.6460 ± 0.0042\dag \\
CatBoost & 0.0927 ± 0.0005 & 3.0108 ± 0.0964 & 10.8683 ± 0.5617\ddag & 227.5076 ± 6.8662 & 0.6765 ± 0.0188 \\
\arrayrulecolor{lightgray}\hline\arrayrulecolor{black}
FT-Trans & 0.0700 ± 0.0008 & 2.9361 ± 0.0481 & 25.8524 ± 0.4659 & 592.0474 ± 1.0442 & 0.6914 ± 0.0265 \\
FT-Trans+PET & \textbf{0.0665 ± 0.0005\ddag*} (\textbf{+4.95\%}) & \textbf{2.8798 ± 0.0155\ddag*} (\textbf{+1.92\%}) & \textbf{22.9815 ± 0.8778} (\textbf{+11.11\%}) & \textbf{591.3838 ± 0.7633} (\textbf{+0.11\%}) & \textbf{0.6662 ± 0.0095\ddag*} (\textbf{+3.65\%}) \\
\arrayrulecolor{lightgray}\hline\arrayrulecolor{black}
SAINT & 0.0674 ± 0.0006 & \textbf{2.9403 ± 0.0131} & 12.4042 ± 2.7267 & 224.6200 ± 0.7025\ddag & 0.7174 ± 0.0197 \\
SAINT+PET & \textbf{0.0637 ± 0.0003\dag*} (\textbf{+5.49\%}) & 2.9404 ± 0.0147 (-0.00\%) & \textbf{9.7170 ± 0.3454\dag*} (\textbf{+21.66\%}) & \textbf{224.1048 ± 0.8715\dag*} (\textbf{+0.23\%}) & \textbf{0.7150 ± 0.0100} (\textbf{+0.34\%}) \\
\arrayrulecolor{lightgray}\hline\arrayrulecolor{black}
TabTrans & — & — & 153.5263 ± 0.1779 & 227.8648 ± 0.9975 & 0.8013 ± 0.0321 \\
TabTrans+PET & — & — & \textbf{152.9747 ± 0.2845} (\textbf{+0.36\%}) & \textbf{226.8719 ± 0.8694} (\textbf{+0.44\%}) & \textbf{0.7125 ± 0.0219} (\textbf{+11.09\%}) \\
\midrule
\end{tabular}
\\
\vspace{0.5em}

\begin{tabular}{l!{\vrule}ccccc}
\midrule
Model & SO & SP & STO & TE & WIS \\
\midrule
XGBoost & 24.6158 ± 1.4767 & 0.1686 ± 0.0003 & 1.8113 ± 0.0053 & 6.5296 ± 0.0147 & 39.7473 ± 1.8945 \\
CatBoost & 24.2181 ± 1.4402 & 0.1079 ± 0.0009 & 0.6891 ± 0.0043\ddag & 1.5440 ± 0.0396\dag & 36.6933 ± 0.6210 \\
\arrayrulecolor{lightgray}\hline\arrayrulecolor{black}
FT-Trans & 17.8475 ± 0.4801 & 0.1060 ± 0.0012 & 0.6951 ± 0.0327 & \textbf{4.1132 ± 0.1258} & 38.0113 ± 0.2432 \\
FT-Trans+PET & \textbf{17.8315 ± 0.2163} (\textbf{+0.09\%}) & \textbf{0.1039 ± 0.0015} (\textbf{+1.97\%}) & \textbf{0.6763 ± 0.0155\dag*} (\textbf{+2.71\%}) & 4.1229 ± 0.1091 (-0.24\%) & \textbf{37.9194 ± 0.2819} (\textbf{+0.24\%}) \\
\arrayrulecolor{lightgray}\hline\arrayrulecolor{black}
SAINT & 17.5344 ± 1.9513\ddag & \textbf{0.1002 ± 0.0015\dag} & 0.9942 ± 0.0786 & 2.4851 ± 0.4729 & 35.3796 ± 0.1146\ddag \\
SAINT+PET & \textbf{17.5074 ± 2.1392\dag*} (\textbf{+0.15\%}) & 0.1002 ± 0.0024\dag (-0.04\%) & \textbf{0.9887 ± 0.0345} (\textbf{+0.55\%}) & \textbf{1.8181 ± 0.2750\ddag*} (\textbf{+26.84\%}) & \textbf{35.2012 ± 0.0637\dag*} (\textbf{+0.50\%}) \\
\arrayrulecolor{lightgray}\hline\arrayrulecolor{black}
TabTrans & 28.8328 ± 1.6286 & — & — & — & — \\
TabTrans+PET & \textbf{28.0807 ± 1.4750} (\textbf{+2.61\%}) & — & — & — & — \\
\bottomrule
\end{tabular}
\caption{Performance comparison on regression datasets. Each result is averaged over 5 random seeds. Use RMSE (lower is better). +PET = corresponding Tab-PET variant using Spearman for graph estimation. \dag and \ddag indicate best and second-best results. * indicates our PET methods are in the top 2. Results show baseline performance and Tab-PET performance with improvement percentages in parentheses.}
\label{tab:regression_results}
\end{table}

\newpage
\subsection{Learnable PE vs Tab-PET} \label{app:7.4}

\begin{table}[H]
\centering
\tiny
\setlength{\tabcolsep}{6pt}
\renewcommand{\arraystretch}{1.2}
\definecolor{lightgray}{gray}{0.7}
\begin{tabular}{l!{\vrule}ccccc}
\toprule
Method & AU & GE & SA & BL & CHU \\
\midrule
Baseline & 0.8283 ± 0.0093 & 0.6412 ± 0.0081 & 0.7791 ± 0.0266 & 0.7481 ± 0.0078 & 0.9076 ± 0.0043 \\
\arrayrulecolor{lightgray}\hline\arrayrulecolor{black}
PET & \textbf{0.8379 ± 0.0162} (+1.16\%) & \textbf{0.6502 ± 0.0080} (+1.41\%) & \textbf{0.7924 ± 0.0252} (+1.71\%) & \textbf{0.7582 ± 0.0055} (+1.35\%) & \textbf{0.9172 ± 0.0046} (+1.05\%) \\
Learnable & 0.8241 ± 0.0088 (-0.51\%) & 0.6415 ± 0.0057 (+0.05\%) & 0.7846 ± 0.0254 (+0.71\%) & 0.7461 ± 0.0098 (-0.27\%) & 0.9034 ± 0.0083 (-0.46\%) \\
\midrule
\end{tabular}
\\
\vspace{0.5em}

\begin{tabular}{l!{\vrule}ccccc}
\midrule
Method & CM & CR & DI & DN & EY \\
\midrule
Baseline & 0.5449 ± 0.0222 & 0.6443 ± 0.0244 & 0.7691 ± 0.0148 & 0.9528 ± 0.0071 & 0.7480 ± 0.0078 \\
\arrayrulecolor{lightgray}\hline\arrayrulecolor{black}
PET & \textbf{0.5705 ± 0.0060} (+4.70\%) & \textbf{0.6821 ± 0.0201} (+5.88\%) & 0.7712 ± 0.0179 (+0.28\%) & \textbf{0.9617 ± 0.0042} (+0.93\%) & \textbf{0.7499 ± 0.0060} (+0.24\%) \\
Learnable & 0.5703 ± 0.0126 (+4.66\%) & 0.6676 ± 0.0305 (+3.62\%) & \textbf{0.7760 ± 0.0089} (+0.90\%) & 0.9520 ± 0.0045 (-0.08\%) & 0.7282 ± 0.0066 (-2.65\%) \\
\midrule
\end{tabular}
\\
\vspace{0.5em}

\begin{tabular}{l!{\vrule}ccccc}
\midrule
Method & FI & HE & JA & KC & KR \\
\midrule
Baseline & 0.4599 ± 0.0075 & 0.7235 ± 0.0059 & 0.8049 ± 0.0059 & 0.7138 ± 0.0095 & \textbf{0.9960 ± 0.0007} \\
\arrayrulecolor{lightgray}\hline\arrayrulecolor{black}
PET & \textbf{0.4700 ± 0.0121} (+2.20\%) & \textbf{0.7299 ± 0.0018} (+0.89\%) & \textbf{0.8062 ± 0.0095} (+0.16\%) & \textbf{0.7308 ± 0.0062} (+2.39\%) & 0.9954 ± 0.0000 (-0.06\%) \\
Learnable & 0.4663 ± 0.0165 (+1.39\%) & 0.7224 ± 0.0054 (-0.15\%) & 0.7982 ± 0.0090 (-0.83\%) & 0.7108 ± 0.0055 (-0.42\%) & \textbf{0.9960 ± 0.0007} (+0.00\%) \\
\midrule
\end{tabular}
\\
\vspace{0.5em}

\begin{tabular}{l!{\vrule}ccccc}
\midrule
Method & MA & PH & QSA & ST & SY \\
\midrule
Baseline & 0.7184 ± 0.0271 & 0.8607 ± 0.0071 & 0.8313 ± 0.0128 & 0.7949 ± 0.0104 & 0.9469 ± 0.0042 \\
\arrayrulecolor{lightgray}\hline\arrayrulecolor{black}
PET & \textbf{0.7398 ± 0.0172} (+2.98\%) & \textbf{0.8663 ± 0.0036} (+0.65\%) & \textbf{0.8441 ± 0.0145} (+1.54\%) & \textbf{0.8112 ± 0.0104} (+2.05\%) & \textbf{0.9479 ± 0.0044} (+0.10\%) \\
Learnable & 0.7018 ± 0.0365 (-2.31\%) & 0.8546 ± 0.0025 (-0.71\%) & 0.8315 ± 0.0139 (+0.02\%) & 0.7976 ± 0.0044 (+0.34\%) & 0.9462 ± 0.0056 (-0.07\%) \\
\midrule
\end{tabular}
\\
\vspace{0.5em}

\begin{tabular}{l!{\vrule}ccccc}
\midrule
Method & TI & VE & WI & WIN & YE \\
\midrule
Baseline & 0.9719 ± 0.0061 & 0.7856 ± 0.0180 & 0.9570 ± 0.0079 & 0.3629 ± 0.0121 & 0.5449 ± 0.0062 \\
\arrayrulecolor{lightgray}\hline\arrayrulecolor{black}
PET & 0.9772 ± 0.0047 (+0.54\%) & \textbf{0.8080 ± 0.0134} (+2.84\%) & \textbf{0.9704 ± 0.0080} (+1.39\%) & \textbf{0.3755 ± 0.0158} (+3.47\%) & \textbf{0.5618 ± 0.0214} (+3.10\%) \\
Learnable & \textbf{0.9787 ± 0.0041} (+0.70\%) & 0.7943 ± 0.0206 (+1.11\%) & 0.9523 ± 0.0122 (-0.49\%) & 0.3515 ± 0.0235 (-3.14\%) & 0.5424 ± 0.0135 (-0.46\%) \\
\bottomrule
\end{tabular}
\caption{Comparison between Tab-PET and Learnable PE methods on classification datasets. Each result is averaged over 5 random seeds. Use balanced accuracy (higher is better). Tab-PET uses Spearman for graph estimation. \textbf{Bold} indicates best result. Results show performance with improvement percentages in parentheses.}
\label{tab:learnable_comparison_classification_results}
\end{table}

\begin{table}[H]
\centering
\tiny
\setlength{\tabcolsep}{6pt}
\renewcommand{\arraystretch}{1.2}
\definecolor{lightgray}{gray}{0.7}
\begin{tabular}{l!{\vrule}ccccc}
\toprule
Method & QS & AB & AI & BO & BOS \\
\midrule
Baseline & 0.9706 ± 0.0163 & 2.2002 ± 0.0326 & 1.4659 ± 0.0911 & 0.3268 ± 0.0891 & 3.2319 ± 0.1545 \\
\arrayrulecolor{lightgray}\hline\arrayrulecolor{black}
PET & 0.9606 ± 0.0199 (+1.03\%) & \textbf{2.1681 ± 0.0124} (+1.46\%) & \textbf{1.3096 ± 0.0510} (+10.67\%) & \textbf{0.2773 ± 0.0480} (+15.13\%) & \textbf{2.9934 ± 0.2309} (+7.38\%) \\
Learnable & \textbf{0.9426 ± 0.0157} (+2.88\%) & 2.1851 ± 0.0209 (+0.69\%) & 1.3653 ± 0.0676 (+6.86\%) & 0.3248 ± 0.0719 (+0.61\%) & 3.2026 ± 0.2269 (+0.91\%) \\
\midrule
\end{tabular}
\\
\vspace{0.5em}

\begin{tabular}{l!{\vrule}ccccc}
\midrule
Method & CA & CH & CL & CO & CP \\
\midrule
Baseline & \textbf{0.1406 ± 0.0015} & \textbf{0.5399 ± 0.0007} & 0.4844 ± 0.0719 & 5.0793 ± 0.1121 & 2.3576 ± 0.0224 \\
\arrayrulecolor{lightgray}\hline\arrayrulecolor{black}
PET & 0.1408 ± 0.0013 (-0.18\%) & 0.5402 ± 0.0007 (-0.05\%) & \textbf{0.3154 ± 0.0132} (+34.88\%) & \textbf{5.0447 ± 0.2000} (+0.68\%) & \textbf{2.3392 ± 0.0273} (+0.78\%) \\
Learnable & 0.1423 ± 0.0009 (-1.21\%) & 0.5418 ± 0.0027 (-0.35\%) & 0.4128 ± 0.0428 (+14.78\%) & 5.3437 ± 0.1114 (-5.21\%) & 2.3564 ± 0.0144 (+0.05\%) \\
\midrule
\end{tabular}
\\
\vspace{0.5em}

\begin{tabular}{l!{\vrule}ccccc}
\midrule
Method & CPU & DIA & EN & FR & GR \\
\midrule
Baseline & 2.8557 ± 0.0583 & 2467.5773 ± 36.9375 & 0.4842 ± 0.0240 & 0.7387 ± 0.0984 & 0.0057 ± 0.0001 \\
\arrayrulecolor{lightgray}\hline\arrayrulecolor{black}
PET & \textbf{2.8270 ± 0.0513} (+1.00\%) & \textbf{2447.4315 ± 41.0115} (+0.82\%) & 0.4725 ± 0.0276 (+2.43\%) & \textbf{0.7164 ± 0.0808} (+3.02\%) & \textbf{0.0055 ± 0.0001} (+2.84\%) \\
Learnable & 2.8650 ± 0.0178 (-0.33\%) & 2513.3235 ± 44.6829 (-1.85\%) & \textbf{0.4418 ± 0.0059} (+8.76\%) & 0.8021 ± 0.1032 (-8.58\%) & 0.0058 ± 0.0002 (-1.75\%) \\
\midrule
\end{tabular}
\\
\vspace{0.5em}

\begin{tabular}{l!{\vrule}ccccc}
\midrule
Method & KI & LI & MU & PL & SE \\
\midrule
Baseline & 0.0700 ± 0.0008 & 2.9361 ± 0.0481 & 25.8524 ± 0.4659 & 592.0474 ± 1.0442 & 0.6914 ± 0.0265 \\
\arrayrulecolor{lightgray}\hline\arrayrulecolor{black}
PET & \textbf{0.0665 ± 0.0005} (+4.95\%) & \textbf{2.8798 ± 0.0155} (+1.92\%) & \textbf{22.9815 ± 0.8778} (+11.11\%) & \textbf{591.3838 ± 0.7633} (+0.11\%) & \textbf{0.6662 ± 0.0095} (+3.65\%) \\
Learnable & 0.0691 ± 0.0004 (+1.29\%) & 2.9121 ± 0.0497 (+0.82\%) & 25.7153 ± 1.3453 (+0.53\%) & 592.5523 ± 1.0172 (-0.09\%) & 0.6806 ± 0.0258 (+1.56\%) \\
\midrule
\end{tabular}
\\
\vspace{0.5em}

\begin{tabular}{l!{\vrule}ccccc}
\midrule
Method & SO & SP & STO & TE & WIS \\
\midrule
Baseline & 17.8475 ± 0.4801 & 0.1060 ± 0.0012 & 0.6951 ± 0.0327 & \textbf{4.1132 ± 0.1258} & 38.0113 ± 0.2432 \\
\arrayrulecolor{lightgray}\hline\arrayrulecolor{black}
PET & \textbf{17.8315 ± 0.2163} (+0.09\%) & \textbf{0.1039 ± 0.0015} (+1.97\%) & \textbf{0.6763 ± 0.0155} (+2.71\%) & 4.1229 ± 0.1091 (-0.24\%) & \textbf{37.9194 ± 0.2819} (+0.24\%) \\
Learnable & 18.0131 ± 0.3303 (-0.93\%) & 0.1073 ± 0.0022 (-1.23\%) & 0.6802 ± 0.0224 (+2.14\%) & 4.2435 ± 0.1353 (-3.17\%) & 38.6926 ± 0.4269 (-1.79\%) \\
\bottomrule
\end{tabular}
\caption{Comparison between Tab-PET and Learnable PE methods on classification datasets. Each result is averaged over 5 random seeds. Use RMSE (lower is better). Tab-PET uses Spearman for graph estimation. \textbf{Bold} indicates best result. Results show performance with improvement percentages in parentheses.}
\label{tab:learnable_comparison_regression_results}
\end{table}

\newpage
\section{Analysis and Ablation of Tab-PET}\label{app:8}

\subsection{Graph Estimation Approaches Analysis} \label{app:8.1}

This section provides detailed analysis of the different graph estimation approaches used in Tab-PET, examining both their performance characteristics and the structural properties of the graphs they generate.

\subsubsection{Performance Distribution Analysis}

Figure~\ref{fig:performance_distribution} presents a view of performance improvements across all graph estimation methods. We can see that Spearman correlation demonstrates the most robust performance across all datasets. Notably, while other methods sometimes can negatively impact performance in some cases, Spearman consistently maintains positive or near-zero improvements. The box plot shows that Spearman's median improvement is the highest among all methods, with the interquartile range positioned entirely above the baseline performance line.

In addition, across all methods, outliers mainly appear in the high-improvement region (above the upper whisker), indicating that graph-derived PEs can occasionally provide substantial performance gains.

\begin{figure}[H]
    \centering
    \includegraphics[width=0.5\linewidth]{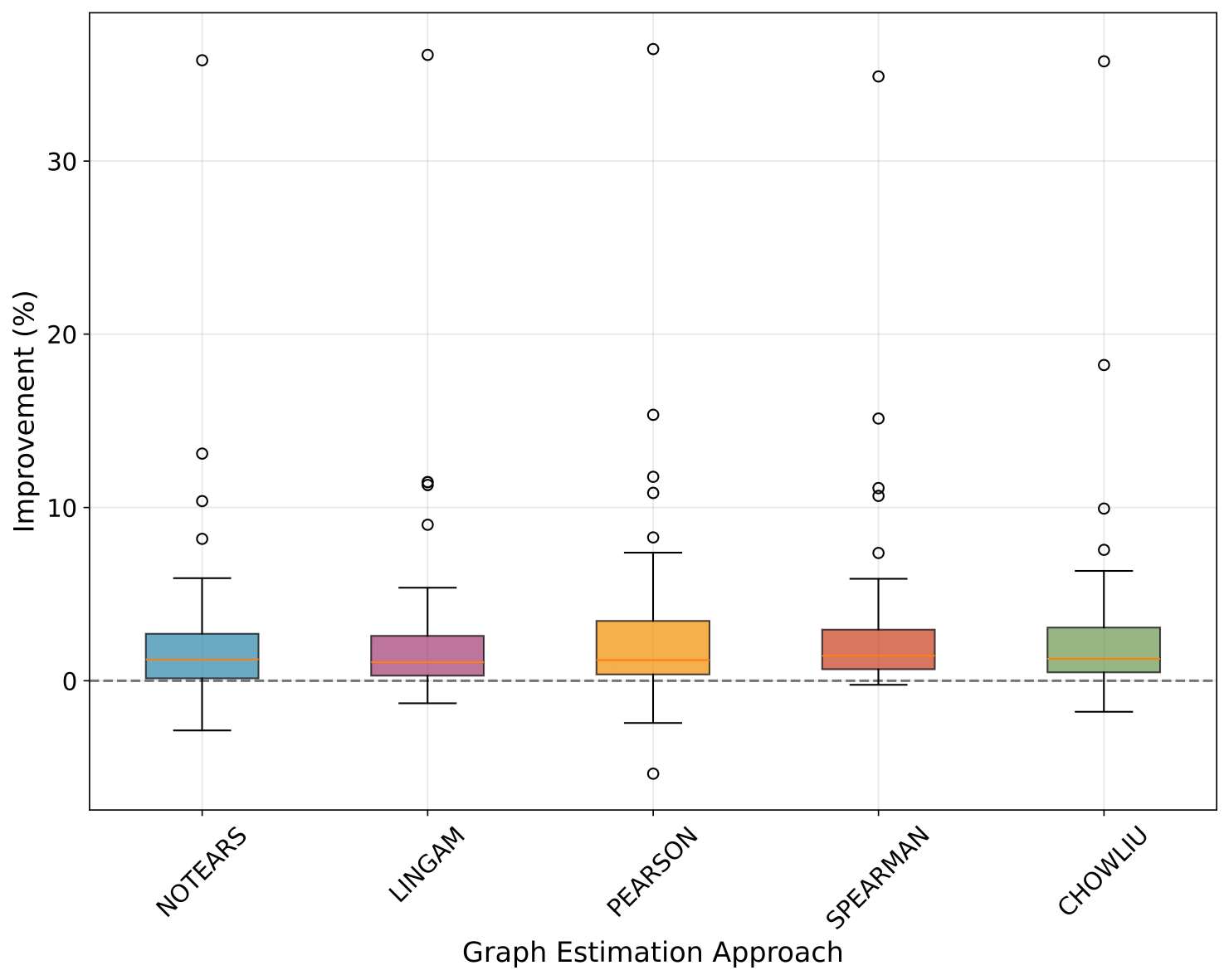}
    \caption{Performance distribution comparison across graph estimation approaches. Box plots show the distribution of improvement percentages for each approach across all datasets and tasks. Each box represents the interquartile range (IQR) with the median indicated by the red line. Whiskers extend to 1.5×IQR. The horizontal dashed line at 0\% indicates the baseline performance.}
    \label{fig:performance_distribution}
\end{figure}

\newpage
\subsubsection{Graph Structure Analysis}

Figure~\ref{fig:graph_metrics_distribution}, \ref{fig:graph_box_plot} and \ref{fig:graph_scatter_plot} reveal the relationship between graph structural properties and performance by analyzing two key metrics: graph entropy and Fiedler value.

\textbf{Graph Entropy:}
For classification tasks (left panels), we observe a pattern between graph entropy and method performance. Causal approaches such as NOTEARS and LiNGAM concentrate in the low graph entropy region, producing highly structured, sparse graphs. This corresponds to the lowest classification improvement (1.36\% and 1.41\% shown in the main paper).
Spearman and Pearson correlations generate graphs with higher entropy values (0.8-1.0), indicating richer, more densely connected structures. These methods also achieve the highest classification performance improvements (1.72\% and 1.61\% respectively).
Regression results mirror the classification findings: higher-entropy graphs from Pearson and Spearman also correspond to the highest performance gains.

\textbf{Fiedler Value:}
The Fiedler value (also known as the algebraic connectivity) is defined as the second-smallest eigenvalue of the graph Laplacian. It reflects how well connected a graph is: a higher value indicates stronger connectivity and robustness of the network structure \cite{wiki_fiedler}.

Overall, the association-based methods have a higher Fiedler value than causality-driven ones. This would result in PEs that can more effectively capture both local and global structural relationships. Additionally, graphs with high Fiedler values are more resistant to disconnection when edges are removed, indicating more stable structural representations that are less sensitive to noise. From a spectral perspective, higher algebraic connectivity leads to higher quality Laplacian matrices, resulting in more informative eigenvectors for PE generation that capture meaningful structural variations across features.

The association-based methods' broader distribution toward higher Fiedler values corresponds to their superior empirical results. This suggests that the structural connectivity captured by these graphs directly translates to more effective PEs. This analysis provides a graph-theoretic explanation for the superiority of association-based methods over causality-driven approaches: while causality-based methods focus on identifying directed relationships that may result in sparse, poorly-connected graphs, association-based methods capture broader statistical dependencies that naturally lead to more connected, higher Fiedler value structures. These well-connected graphs provide richer positional information that better guides the transformer's attention mechanisms in tabular learning.
\begin{figure}[H]
    \centering
    \includegraphics[width=\linewidth]{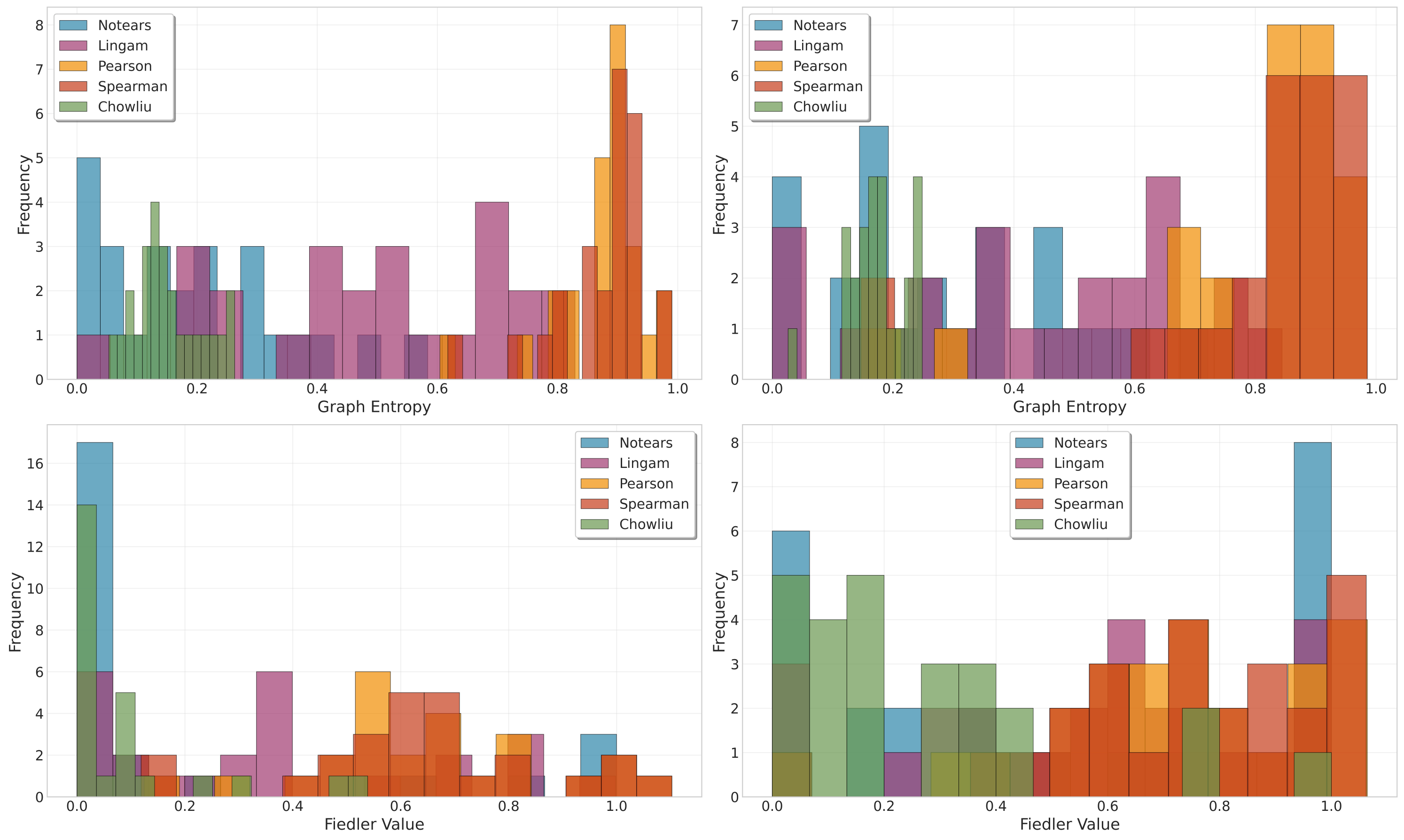}
    \caption{Distribution of graph metrics across different construction methods for classification (left panels) and regression (right panels) tasks. The figure shows frequency histograms for graph entropy (top row) and Fiedler value (bottom row), comparing five graph estimation methods. The varying widths of histogram bars reflect the different value ranges spanned by each method - methods with broader value distributions result in wider bins, while methods with more concentrated value ranges produce narrower bins.}
    \label{fig:graph_metrics_distribution}
\end{figure}

\newpage

\begin{figure}[H]
    \centering
    \includegraphics[width=\linewidth]{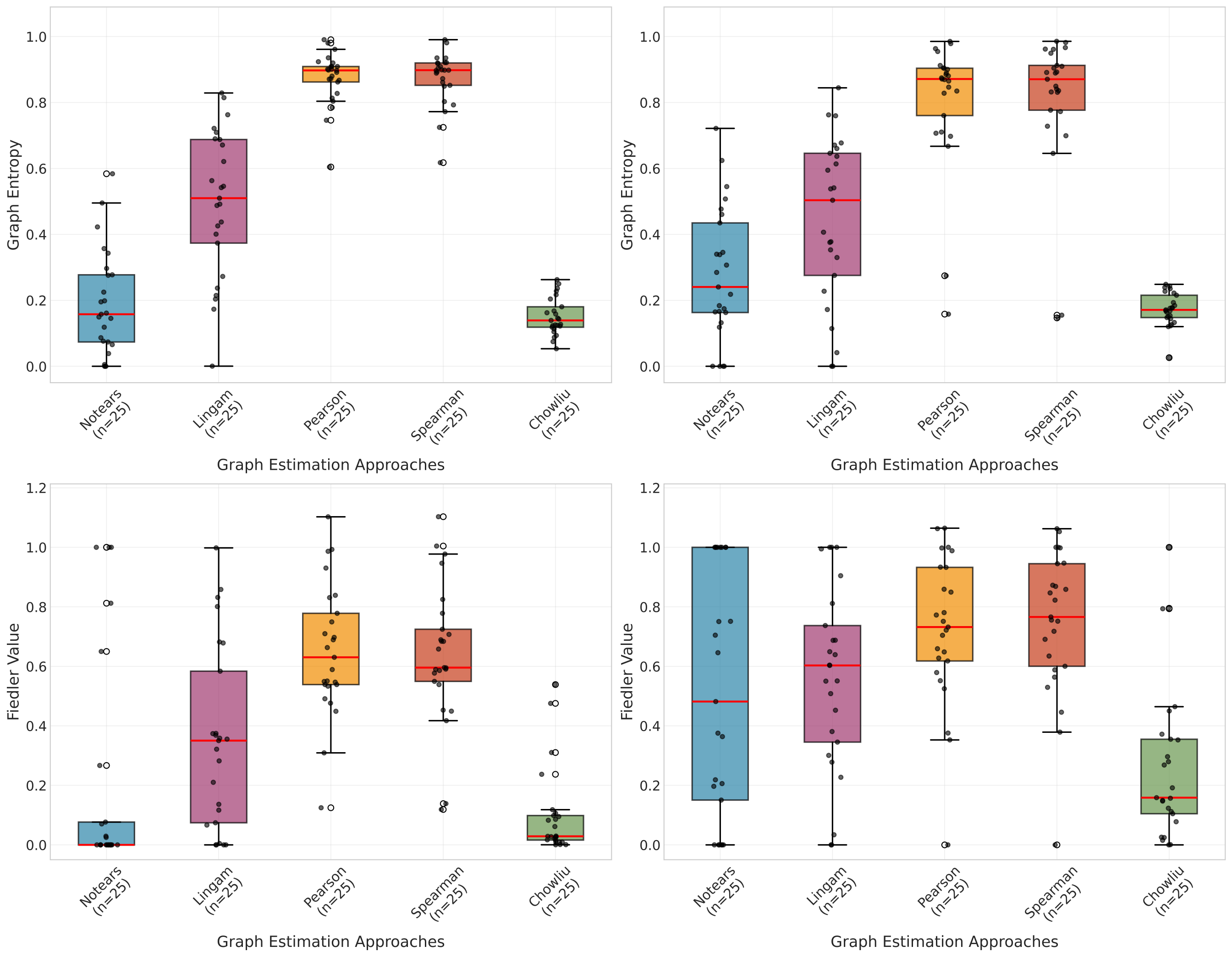}
    \caption{Box plot comparison of graph estimation approaches for classification (left panels) and regression (right panels) tasks. Each box shows the median (red line), quartiles (box boundaries), and whiskers extending to 1.5 times the interquartile range. Individual data points are shown as black dots.}
    \label{fig:graph_box_plot}
\end{figure}

\begin{figure}[H]
    \centering
    \includegraphics[width=\linewidth]{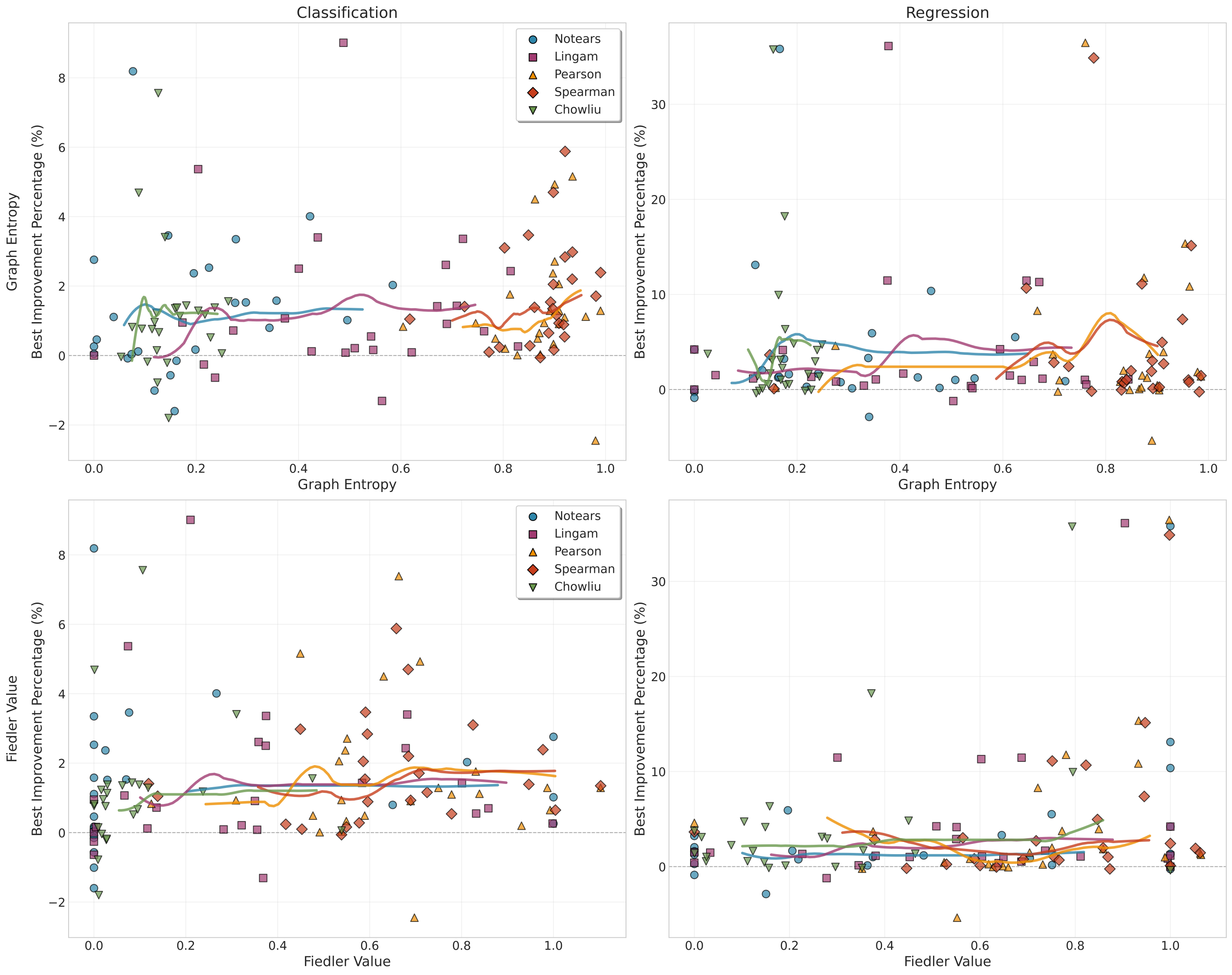}
    \caption{Scatter plot with cumulative threshold trends for classification (left panels) and regression (right panels) tasks. Each subplot contains two types of visualization: (1) \textbf{Scatter points} representing individual datasets, where each point shows the specific graph metric value (x-axis) and Tab-PET performance improvement (y-axis) for that dataset under a particular graph estimation method. (2) \textbf{Cumulative trend lines} for NOTEARS (blue), LiNGAM (purple), Pearson correlation (orange), Spearman correlation (orange-red), and Chow-Liu (green). For each method and each threshold value on the x-axis, we calculate the average performance improvement of all datasets whose graph entropy/Fiedler value is less than or equal to that threshold, creating a cumulative threshold analysis where y-axis represents the mean improvement of datasets below the current x-axis threshold.}
    \label{fig:graph_scatter_plot}
\end{figure}

\newpage
\subsection{Tab-PET vs. Learnable PEs}

Figure~\ref{fig:tabpet_vs_learnable} shows our fixed graph-derived Tab-PET approach consistently outperforming learnable PEs across both classification and regression tasks. The results clearly show that Tab-PET achieves more stable and higher performance improvements, with the majority of datasets exhibiting positive gains compared to the baseline (no PE) condition. In contrast, learnable PEs frequently produce negative improvements that indicate worse performance than the baseline. This pattern strongly supports that incorporating structured priors through graph-derived fixed PEs is more effective than learning positional representations from scratch, especially in the low-data regime typical of tabular learning. The superior performance of Tab-PET can be attributed to its ability to leverage meaningful structural relationships discovered through graph estimation, providing informative inductive biases that guide the transformer's attention mechanism, whereas learnable PEs must discover these relationships during training with limited data, often leading to suboptimal or unstable representations.

\begin{figure}[H]
    \centering
    \includegraphics[width=\linewidth]{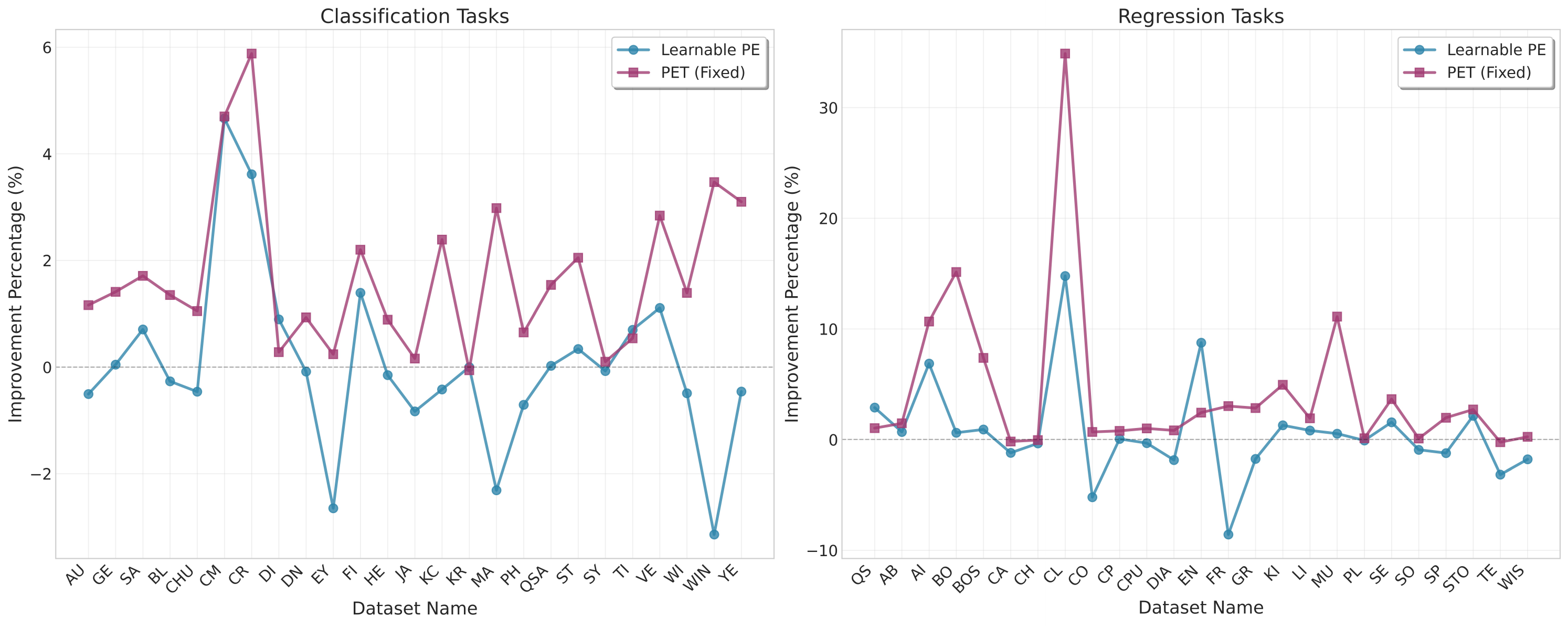}
    \caption{Performance improvement (\%) over the baseline (no PE) across datasets for Learnable PE and our PET on classification (left) and regression (right) tasks.}
    \label{fig:tabpet_vs_learnable}
\end{figure}

\subsection{P-value Comparison between Transformer-based and Tree-based}

Table~\ref{tab:transformer_vs_tree_extended} reveals that while tree-based methods generally maintain competitive performance against traditional transformer-based approaches, Tab-PET consistently enhances the competitiveness of transformer models across all architectures.

\begin{table}[H]
\centering
\setlength{\tabcolsep}{1.5pt}
\renewcommand{\arraystretch}{1.2}
\begin{tabular}{l!{\vrule}cc!{\vrule}cc!{\vrule}cc!{\vrule}cc}
\toprule
\multirow{3}{*}{Method} & \multicolumn{4}{c!{\vrule}}{vs XGBoost} & \multicolumn{4}{c}{vs CatBoost} \\
\cmidrule{2-9}
& \multicolumn{2}{c!{\vrule}}{Win Rate\%} & \multicolumn{2}{c!{\vrule}}{P-value} & \multicolumn{2}{c!{\vrule}}{Win Rate\%} & \multicolumn{2}{c}{P-value} \\
\cmidrule{2-9}
& C & R & C & R & C & R & C & R \\
\midrule
FT-Transformer (No-PE) & 48 & 72 & 0.979 & 0.026 & 48 & 40 & 0.833 & 0.120 \\
FT-Transformer + Tab-PET & 56 & 80 & 0.230 & 0.007 & 72 & 52 & 0.032 & 0.653 \\
\midrule
SAINT (No-PE) & 40 & 80 & 0.312 & 9.12e-04 & 36 & 40 & 0.692 & 0.210 \\
SAINT + Tab-PET & 44 & 84 & 0.653 & 2.17e-04 & 52 & 48 & 0.367 & 0.711 \\
\midrule
TabTransformer (No-PE) & 0 & 29 & 0.004 & 0.688 & 11 & 0 & 0.008 & 0.016 \\
TabTransformer + Tab-PET & 11 & 29 & 0.012 & 0.812 & 33 & 14 & 0.164 & 0.047 \\
\bottomrule
\end{tabular}
\caption{Comparison of transformer methods against tree baselines. Win Rate shows percentage of datasets where transformer method outperforms tree baseline. P-values from Wilcoxon signed-rank tests. C = Classification, R = Regression.}
\label{tab:transformer_vs_tree_extended}
\end{table}
\newpage 
\subsection{Average Performance of All Methods}

Here we give the average performance of all methods, including baselines XGBoost, CatBoost, FT-Transformer (without PEs), SAINT (without PEs), and TabTransformer (without PEs), and our Tab-PET variants (FT-Transformer+Tab-PET, SAINT+Tab-PET, TabTransformer+Tab-PET).

During preliminary analysis, we observed that certain datasets exhibited extreme performance characteristics that could distort overall comparisons. For instance, the diamonds\_v8 dataset in regression tasks showed RMSE values exceeding 2000 for FT-Transformer (both without PEs and +Tab-PET variant), while typical RMSE values across other datasets remained very low (roughly less than $10^2$). Such extreme outliers can substantially bias mean performance metrics.

To address this issue, for each task, we apply the interquartile range (IQR) criterion to identify statistical outliers. For regression, the outliers would be datasets with mean RMSE $>$ Q3 + 1.5×IQR; for classification, the outliers would be datasets with mean balanced accuracy $<$ Q1 - 1.5×IQR. Unlike classification tasks where balanced accuracy ranges between 0 and 1, regression tasks exhibit much larger variations in RMSE values. Therefore, for regression tasks, we supplement the IQR analysis by excluding extreme datasets where the mean RMSE exceeds 10 times the median performance across all datasets. Datasets meeting either criterion are flagged as outliers and excluded from the following analysis to ensure fair comparison across methods. Excluded outlier datasets are wine-quality-white\_v1 for classification, and diamonds\_v8, munich-rent-index-1999\_v1, plasma\_retinol\_v1, socmob\_v1, wisconsin\_v1 for regression.

\begin{table}[H]
\centering
\begin{tabular}{l!{\vrule}cc!{\vrule}cc}
\toprule
\multirow{2}{*}{Method} & \multicolumn{2}{c|}{Classification} & \multicolumn{2}{c}{Regression} \\
\cmidrule{2-5}
& Mean Acc & N & Mean RMSE & N \\
\midrule
XGBoost & 0.782 & 24$^{(1)}$ & 2.091 & 20$^{(5)}$ \\
CatBoost & 0.783 & 24$^{(1)}$ & 1.312 & 20$^{(5)}$ \\
\midrule
FT-Transformer & 0.778 & 24$^{(1)}$ & 1.475 & 20$^{(5)}$ \\
FT-Transformer + Tab-PET & 0.790 & 24$^{(1)}$ & 1.431 & 20$^{(5)}$ \\
\midrule
SAINT & 0.780 & 24$^{(1)}$ & 1.483 & 20$^{(5)}$ \\
SAINT + Tab-PET & 0.790 & 24$^{(1)}$ & 1.428 & 20$^{(5)}$ \\
\midrule
TabTransformer & 0.789 & 9 & 2.458 & 3$^{(4)}$ \\
TabTransformer + Tab-PET & 0.797 & 9 & 2.404 & 3$^{(4)}$ \\
\bottomrule
\end{tabular}
\caption{Mean performance across datasets after systematic outlier exclusion. N shows the number of datasets used, with superscript parentheses N$^{(x)}$ indicating the number of outlier datasets excluded for each method. Mean Acc = Mean Balanced Accuracy.}
\label{tab:systematic_clean_performance}
\end{table}

As shown in Table \ref{tab:systematic_clean_performance}, transformer methods demonstrate competitive performance with tree-based baselines in classification tasks. This competitive parity in classification may explain why some recent tabular deep learning research focuses exclusively on classification tasks (e.g., TransTab~\cite{wang2022transtab}, SCARF~\cite{bahri2021scarf}), as the performance gap between deep learning approaches and tree-based methods is less pronounced in this domain. However, in regression tasks, CatBoost (1.312) maintains a clear advantage over transformer-based approaches, with the best transformer method (SAINT + Tab-PET: 1.428) and second-best (FT-Transformer + Tab-PET: 1.431). Despite this challenge in regression, the Tab-PET enhancement provides consistent improvements across all transformer architectures, which could validate Tab-PET's effectiveness as a general improvement for transformer-based tabular learning.
\newpage 
\subsection{Tab-PET Performance vs. Dataset Size}

Figure~\ref{fig:pe_size_analysis} illustrates the relationship between Tab-PET's performance improvement and dataset size across transformer models. We analyze datasets with more than 100 samples to ensure statistical reliability.

\begin{figure}[H]
    \centering
    \includegraphics[width=0.5\linewidth]{
    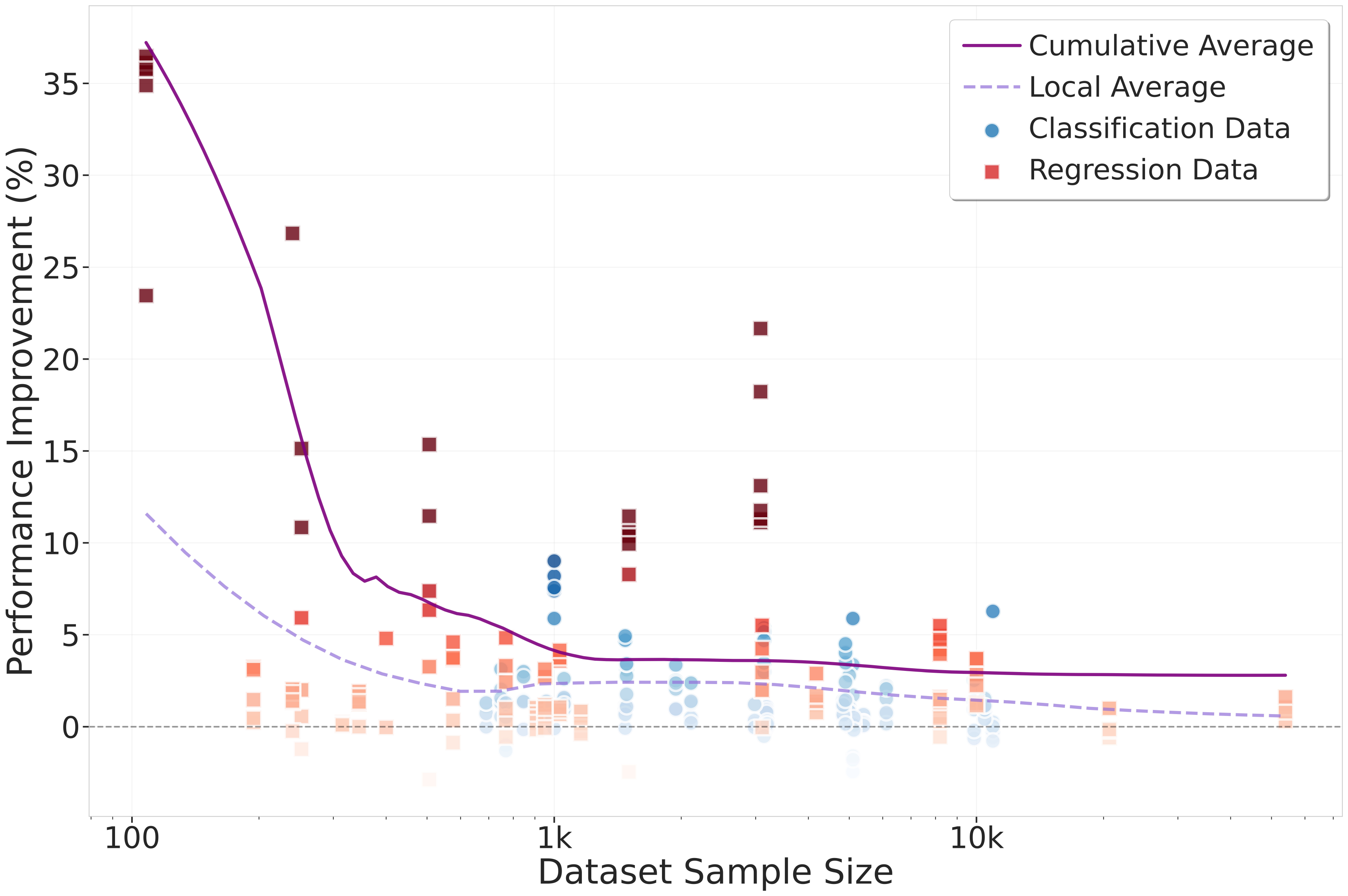}
    \caption{Tab-PET improvement vs. dataset size on transformer models (datasets $>$100 samples).}
    \label{fig:pe_size_analysis}
\end{figure}

The analysis reveals two key observations. First, Tab-PET provides larger performance improvements on smaller datasets, where the inductive bias introduced by PEs is particularly valuable for learning effective representations with limited training examples. This aligns with the intuition that explicit structural information becomes more critical when data is scarce. Second, despite the diminishing relative gains, the average improvement remains consistently positive as dataset size increases, demonstrating that Tab-PET's benefits extend across the full spectrum of data scales.

\subsection{Statistical Significance Analysis} \label{app:8.6}

To rigorously assess the statistical significance of Tab-PET's performance improvements, we conduct Wilcoxon signed-rank tests comparing Tab-PET variants against their corresponding baseline transformer models without PEs.

We report three key metrics for each architecture: (1) Mean Improvement\%, which quantifies the average performance gain achieved by Tab-PET; (2) Positive Rate\%, indicating the percentage of datasets where Tab-PET outperforms the baseline; and (3) P-values from Wilcoxon signed-rank tests, where values below 0.05 indicate statistically significant improvements at the 95\% confidence level.
The results in Table~\ref{tab:pvalue} demonstrate that Tab-PET achieves statistically significant improvements across all transformer architectures and both task types, with all p-values well below the 0.05 threshold.

\begin{table}[H]
\centering
\renewcommand{\arraystretch}{1}
\begin{tabular}{l!{\vrule}cc!{\vrule}cc!{\vrule}cc}
\toprule
\multirow{2}{*}{Method} & \multicolumn{2}{c}{Mean Improv.\%} & \multicolumn{2}{c}{Positive Rate\%} & \multicolumn{2}{c}{P-value} \\
\cmidrule{2-7}
& C & R & C & R & C & R \\
\midrule
FT-Transformer & 1.72 & 4.34 & 96 & 88 & 1.19e-07 & 3.28e-06 \\
SAINT & 1.24 & 3.75 & 84 & 72 & 5.39e-05 & 6.90e-03 \\
TabTransformer & 1.03 & 2.78 & 100 & 100 & 3.90e-03 & 1.56e-02 \\
\midrule
Overall & 1.41 & 3.89 & 91.5 & 82.5 & 2.03e-10 & 3.50e-08 \\
\bottomrule
\end{tabular}
\caption{Statistical significance of Tab-PET improvements over baseline transformers. Mean Improv.\% shows average performance gains. Positive Rate\% indicates percentage of datasets with positive improvements. P-values from Wilcoxon signed-rank tests assess statistical significance. C = Classification (25, 25, 9 datasets for FT-Transformer/SAINT/TabTransformer), R = Regression (25, 25, 7 datasets for FT-Transformer/SAINT/TabTransformer).}
\label{tab:pvalue}
\end{table}

\newpage
\section{Proofs of Theoretical Results}\label{app:9}

\setcounter{theorem}{0}

\begin{theorem}[Effective Rank under Random Inputs]
    Let \( x \in \mathbb{R}^d \) be an input vector to a single-layer, single-head FT-Transformer with components \( x_i \sim \text{i.i.d.} \) and \( x_i \in (0,1) \). Let \( d_T \) denote the token dimension (inclusive of concatenated position encodings). Let \( q\in \mathbb{R}^{d_T} \) denote the learnable CLS token embedding, and \( p_i \in \mathbb{R}^{d_p} \) be the positional encodings for each input dimension. Assume the scaled positional encodings \( p'_i = \alpha p_i \) are used, where \( \alpha > 0 \).  Given the query, key and value matrices $Q,K,V$, where $K$ can be decomposed as $[K_x;K_p]$, where $K_x \in \mathbb{R}^{d \times d_T}, K_p \in \mathbb{R}^{d_p \times d_T}$, and similarly for $V$. suppose the following conditions hold: $\max_i \langle Q^Tq,K_p^Tp_i \rangle - \max_{j \neq i} \langle  Q^Tq,K_p^Tp_j \rangle = \tau,$ and the norm of the query-key matrices $Q,K$ and CLS embedding $q$ are all bounded by $c_Q,c_K$ and $c_q$ respectively. Lastly, assume that the tokenizer weights $w_i$ have the same norm and the value matrix $V$ is norm preserving and satisfies $V_p = 0$.
Define
\begin{equation}
C_\alpha = \exp\left( \frac{\alpha \tau - 2 c_Kc_Q c_q }{\sqrt{d_T}} \right).
\end{equation}
Then the effective rank \( r_{\mathrm{eff}} \) of the CLS token output after self-attention satisfies
\begin{equation}
r_{\mathrm{eff}} \leq \left( C_\alpha + d \right) \cdot \exp\left( -\frac{C_\alpha}{C_\alpha + d} \cdot \log C_\alpha \right).
\end{equation}
In the regime where \( C_\alpha \gg d \), this simplifies to $r_{\mathrm{eff}} \approx 1 + \frac{d}{C_\alpha}.$
\end{theorem}
\begin{proof}
We analyze the attention mechanism applied to the CLS token in a Transformer. Let the attention score for token \( i \) be:
\begin{equation}
s_i = \langle Q^\top q, K^\top [x_i; p'_i] \rangle,
\end{equation}
where \( x_i \) is the token embedding, \( p'_i \) is the positional encoding, and \( q \) is the query vector for the CLS token.

We decompose the key matrix \( K \) into two components:
\begin{equation}
K^\top [x_i; p'_i] = K_x^\top x_i w_i + K_p^\top p'_i,
\end{equation}
where \( w_i \) is a learned token weight vector. Assume:
 \( \|x_i\| \leq 1 \),
 \( \|w_i\| = 1 \),
 \( \|K_x\| \leq c_K \),
 \( \|K_p\| \leq c_K \),
 \( \|Q\| \leq c_Q \),
 \( \|q\| \leq c_q \). Let \( i^* = \arg\max_i \langle Q^\top q, K_p^\top p'_i \rangle \), and let \( j \neq i^* \) be the second-largest. Note that
$\langle Q^\top q, K_p^\top p'_{i^*} \rangle - \langle Q^\top q, K_p^\top p'_j \rangle = \alpha \tau.$ Now, the full attention score difference is:
$s_{i^*} - s_j = \langle Q^\top q, K_x^\top x_{i^*} w_{i^*} + K_p^\top p'_{i^*} \rangle - \langle Q^\top q, K_x^\top x_j w_j + K_p^\top p'_j \rangle.$

Bounding the key-query interaction terms, we obtain
$\left| \langle Q^\top q, K_x^\top x_i w_i \rangle \right| \leq \|Q^\top q\| \cdot \|K_x^\top x_i w_i\| \leq c_Q c_q \cdot c_K  c_w.$ So the difference in key-query terms is bounded by:
$\langle Q^\top q, K_x^\top x_{i^*} w_{i^*} \rangle - \langle Q^\top q, K_x^\top x_j w_j \rangle \geq - 2 c_K c_Q c_q.$ Thus, the total score difference satisfies:
$s_{i^*} - s_j \geq \alpha \tau - 2 c_K c_Q c_q.$

Let \( d_T \) be the temperature scaling factor in softmax. Then the ratio of attention weights is:
\begin{equation}
\frac{\alpha_{i^*}}{\alpha_j} = \exp\left( \frac{s_{i^*} - s_j}{\sqrt{d_T}} \right) \geq \exp\left( \frac{\alpha \tau - 2 c_K c_Q c_q c_w}{\sqrt{d_T}} \right) =: C_\alpha.
\end{equation}

Assuming all other attention weights \( \alpha_j \) for \( j \neq i^* \) are equal, we solve for normalized weights:
\begin{equation}
\alpha_{i^*} = \frac{C_\alpha}{C_\alpha + d}, \quad \alpha_j = \frac{1}{C_\alpha + d}, \quad \text{for } j \neq i^*.
\end{equation}

Let \( v_i = V^\top [x_i; p_i] \) be the value vector corresponding to token \( i \). By the assumptions in the theorem, the value matrix \( V \) is norm-preserving, and the learned token weights \( w_i \) have equal norm. Since the variation across tokens is introduced by the embeddings \( x_i \), and all other components are bounded and uniform, we assume:
\begin{equation}
\|v_i\| = c_v \quad \text{for all } i,
\end{equation}
where \( c_v \) is constant across tokens.

The output of the CLS token is:
\begin{equation}
\text{CLS}_{\text{out}} = \sum_{i=1}^d \alpha_i x_i v_i,
\end{equation}
where \( x_i \) introduces directional variation, and \( \alpha_i \) are the attention weights.

To quantify the diversity of directions represented in this output, we use the notion of effective rank. When the vectors \( x_i v_i \) are orthogonal and have equal norm, the effective rank is defined as:
\begin{equation}
r_{\mathrm{eff}} := \exp(H(\alpha)),
\end{equation}
where \( H(\alpha) \) is the Shannon entropy of the attention weights:
\begin{equation}
H(\alpha) = -\sum_{i=1}^d \alpha_i \log \alpha_i.
\end{equation}

This definition captures the number of effectively used directions in the weighted sum, assuming equal energy and orthogonality. Orthogonality maximizes the entropy-based rank because it ensures that each direction contributes independently to the output. Using the earlier solution for the attention weights:
\begin{equation}
\alpha_{i^*} = \frac{C_\alpha}{C_\alpha + d}, \quad \alpha_j = \frac{1}{C_\alpha + d} \quad \text{for } j \neq i^*,
\end{equation}
we compute the entropy:
\begin{equation}
H(\alpha) = -\frac{C_\alpha}{C_\alpha + d} \log \frac{C_\alpha}{C_\alpha + d} - \frac{d}{C_\alpha + d} \log \frac{1}{C_\alpha + d}.
\end{equation}

Subsequently, the effective rank becomes:
\begin{equation}
r_{\mathrm{eff}} = \exp(H(\alpha)) = \left( C_\alpha + d \right) \cdot \exp\left( -\frac{C_\alpha}{C_\alpha + d} \log C_\alpha \right).
\end{equation}

In the regime \( C_\alpha \gg d \), we expand this expression to obtain:
\begin{equation}
r_{\mathrm{eff}} \approx 1 + \frac{d}{C_\alpha}.
\end{equation}

This completes the proof. 

\end{proof}

\begin{theorem}[Effective Rank under Structured Inputs]
Consider the same setting as in Theorem 1, except that the input vector \( x \in \mathbb{R}^d \) is structured as follows: \( d \) is even, and
\begin{equation}
x_i = 
\begin{cases}
\theta & \text{for } i \leq d/2, \\
\theta'  & \text{for } i > d/2,
\end{cases}
\end{equation}
with shared latent variables \( \theta, \theta' \in (0,1) \), and coefficients \( \beta_i, \gamma_i \in \mathbb{R} \). Then the effective rank \( r_{\mathrm{eff}} \) of the CLS token output after self-attention satisfies:

\begin{itemize}
    \item \textbf{(a) Random positional encodings:}
\begin{equation}
    r_{\mathrm{eff}} \leq \left( 2C_\alpha + d \right) \cdot \exp\left( -\frac{2C_\alpha \log(2C_\alpha) + d \log d}{2C_\alpha + d} \right),
\end{equation}
    $\text{which simplifies to } r_{\mathrm{eff}} \approx 1 + \frac{d}{2C_\alpha} \text{ when } C_\alpha \gg d.$

    \item \textbf{(b) Shared positional encodings within groups:}
    If \( p_i \) is fixed for all \( i \leq d/2 \), and is a different fixed vector for \( i > d/2 \), then
\begin{equation}
    r_{\mathrm{eff}} \leq (C_\alpha + 1) \cdot \exp\left( -\frac{C_\alpha}{C_\alpha + 1} \cdot \log C_\alpha \right),
\end{equation}
    
 $\text{which simplifies to } r_{\mathrm{eff}} \approx 1 + \frac{1}{C_\alpha} \text{ for large } C_\alpha.$

\end{itemize}
\end{theorem}

\begin{proof}

\textbf{[Proof of Theorem 2, Part (a)]}

We build on the analysis from Theorem 1, where the attention score for token \( i \) is given by:
\begin{equation}
s_i = \langle Q^\top q, K^\top [x_i; p'_i] \rangle,
\end{equation}
and the attention weight ratio is controlled by the score difference:
\begin{equation}
\frac{\alpha_{i^*}}{\alpha_j} \geq \exp\left( \frac{\alpha \tau - 2 c_K c_Q c_q}{\sqrt{d_T}} \right) =: C_\alpha.
\end{equation}

This term \( C_\alpha \) captures the exponential separation between the top attention score and the others.

Now consider the structured input setting where the input vector \( x \in \mathbb{R}^d \) is defined as:
\begin{equation}
x_i = 
\begin{cases}
\theta & \text{for } i \leq d/2, \\
\theta' & \text{for } i > d/2,
\end{cases}
\end{equation}
with shared latent variables \( \theta, \theta' \in (0,1) \). This induces a rank-2 structure in the token embeddings, meaning that the directions \( x_i v_i \) span at most two distinct subspaces.

Since the value matrix \( V \) extracts only from the token embeddings and is norm-preserving, the output of the CLS token is:
\begin{equation}
\text{CLS}_{\text{out}} = \sum_{i=1}^d \alpha_i x_i v_i,
\end{equation}
where \( x_i v_i \) lies in the span of \( \theta \) and \( \theta' \). Thus, the output lies in a 2D subspace, and the entropy-based effective rank is upper-bounded by the number of distinct directions, which is at most two.

To maximize entropy under this constraint, we assign the maximum attention weight to a single token in the first half (say, token \( i^* \leq d/2 \)), and distribute the remaining attention equally among all tokens in the second half. This ensures that both latent directions \( \theta \) and \( \theta' \) contribute to the output.

Let the maximum attention weight be \( \alpha_{i^*} \), and let the remaining \( d/2 \) tokens in the second half each receive attention weight \( \alpha_j \). Then:
\begin{equation}
\alpha_{i^*} = \frac{2C_\alpha}{2C_\alpha + d}, \quad \alpha_j = \frac{2}{2C_\alpha + d} \quad \text{for } j > d/2.
\end{equation}

The entropy of the attention distribution over the two dominant directions is:
\begin{equation}
H(\alpha) = -\frac{2C_\alpha}{2C_\alpha + d} \log \left( \frac{2C_\alpha}{2C_\alpha + d} \right)
- \frac{d}{2C_\alpha + d} \log \left( \frac{d}{2C_\alpha + d} \right).
\end{equation}

Thus, the effective rank becomes:
\begin{equation}
r_{\mathrm{eff}} = \exp(H(\alpha)) = \left( 2C_\alpha + d \right) \cdot \exp\left( 
-\frac{2C_\alpha}{2C_\alpha + d} \log(2C_\alpha) 
- \frac{d}{2C_\alpha + d} \log d 
\right).
\end{equation}

In the regime \( C_\alpha \gg d \), we expand this expression to obtain:
\begin{equation}
r_{\mathrm{eff}} \approx 1 + \frac{d}{2C_\alpha}.
\end{equation}

\noindent\textbf{[Proof of Theorem 2, Part (b)]}
In this setting, the input vector \( x \in \mathbb{R}^d \) is structured as:
\begin{equation}
x_i = 
\begin{cases}
\theta & \text{for } i \leq d/2, \\
\theta' & \text{for } i > d/2,
\end{cases}
\end{equation}
and the positional encodings are shared within each half:
\begin{equation}
p_i = 
\begin{cases}
p & \text{for } i \leq d/2, \\
p' & \text{for } i > d/2.
\end{cases}
\end{equation}

As a result, all tokens in the same half have identical attention scores. Let \( s_L \) and \( s_R \) denote the scores for the left and right halves:
$s_L = \langle Q^\top q, K^\top [\theta; \alpha p] \rangle, \quad s_R = \langle Q^\top q, K^\top [\theta'; \alpha p'] \rangle.$

Then the total attention mass assigned to each half is:
\begin{equation}
\alpha_L = \sum_{i \leq d/2} \alpha_i, \quad \alpha_R = \sum_{i > d/2} \alpha_i.
\end{equation}

By the same argument as in Theorem 1, the ratio of these total weights satisfies:
\begin{equation}
\frac{\alpha_L}{\alpha_R} \geq \exp\left( \frac{\alpha \tau - 2 c_K c_Q c_q}{\sqrt{d_T}} \right) = C_\alpha.
\end{equation}

Assuming normalization over all \( d \) tokens, we solve:
\begin{equation}
\alpha_L = \frac{C_\alpha}{C_\alpha + 1}, \quad \alpha_R = \frac{1}{C_\alpha + 1}.
\end{equation}

Since the output lies in the span of \( \theta \) and \( \theta' \), the effective rank is determined by the entropy over the two aggregated directions. Thus, we compute:
$H(\alpha) = -\alpha_L \log \alpha_L - \alpha_R \log \alpha_R.$

Substituting the expressions:
\begin{equation}
H(\alpha) = -\frac{C_\alpha}{C_\alpha + 1} \log \left( \frac{C_\alpha}{C_\alpha + 1} \right)
- \frac{1}{C_\alpha + 1} \log \left( \frac{1}{C_\alpha + 1} \right).
\end{equation}

Then the effective rank is:
\begin{equation}
r_{\mathrm{eff}} = \exp(H(\alpha)) = (C_\alpha + 1) \cdot \exp\left( -\frac{C_\alpha}{C_\alpha + 1} \cdot \log C_\alpha \right).
\end{equation}

In the regime \( C_\alpha \gg 1 \), this simplifies to:
\begin{equation}
r_{\mathrm{eff}} \approx 1 + \frac{1}{C_\alpha}.
\end{equation}

\end{proof}

\end{document}